\documentclass[dvipsnames,10pt]{article} %
\usepackage[accepted]{tmlr}

\usepackage{amsmath,amsfonts,bm}

\def\eqref#1{equation~\ref{#1}}

\def\1{\bm{1}}

\DeclareMathAlphabet{\mathsfit}{\encodingdefault}{\sfdefault}{m}{sl}
\SetMathAlphabet{\mathsfit}{bold}{\encodingdefault}{\sfdefault}{bx}{n}

\usepackage[utf8]{inputenc} %
\usepackage[T1]{fontenc}    %
\usepackage{hyperref}       %
\usepackage{url}            %
\usepackage{booktabs}       %
\usepackage{amsfonts}       %
\usepackage{nicefrac}       %
\usepackage{microtype}      %
\usepackage{enumitem}
\usepackage{xcolor-material}
\usepackage{xcolor-material, colortbl, soul, wrapfig}
\usepackage{nicefrac, xspace, multirow, makecell}
\usepackage{subfig, pifont, arydshln}
\usepackage{amsmath}
\usepackage{amssymb}
\usepackage{mathtools}
\usepackage{amsthm}
\usepackage{pifont}
\usepackage[capitalize,noabbrev]{cleveref}

\theoremstyle{plain}
\newtheorem{theorem}{Theorem}[section]

\newtheorem{lemma}[theorem]{Lemma}
\newtheorem{corollary}[theorem]{Corollary}
\theoremstyle{definition}

\newtheorem{assumption}[theorem]{Assumption}
\theoremstyle{remark}
\newtheorem{remark}[theorem]{Remark}
\definecolor{Gray}{gray}{0.9}
\usepackage{algorithmic}
\usepackage{algorithm2e}
\RestyleAlgo{ruled}

\newcommand{\xmark}{\ding{55}}%

\makeatletter
\DeclareRobustCommand\onedot{\futurelet\@let@token\@onedot}
\def\@onedot{\ifx\@let@token.\else.\null\fi\xspace}

\def\eg{\emph{e.g}\onedot} 
\def\ie{\emph{i.e}\onedot}

\def\wrt{w.r.t\onedot} 

\makeatother

\makeatletter
 \def\SOUL@hlpreamble{%
 \setul{}{3.2ex}%
 \let\SOUL@stcolor\SOUL@hlcolor
 \SOUL@stpreamble
 }
\makeatother

\definecolor{difcolor}{RGB}{204, 0, 204}

\DeclareRobustCommand{\edit}[1]{#1}

\newcommand{\hlc}[2][yellow]{{%
    \colorlet{foo}{#1}%
    \sethlcolor{foo}\hl{#2}}%
}

\DeclareRobustCommand{\fedavg}
{\textsc{FedAvg}}
\DeclareRobustCommand{\fedprox}
{\textsc{FedProx}}
\DeclareRobustCommand{\scaffold}
{\textsc{SCAFFOLD}}
\DeclareRobustCommand{\feddyn}
{\textsc{FedDyn}}
\DeclareRobustCommand{\adabest}
{\textsc{AdaBest}}
\DeclareRobustCommand{\fedavgm}
{\textsc{FedAvgM}}
\DeclareRobustCommand{\mime}
{\textsc{Mime}}
\DeclareRobustCommand{\mimemom}
{\textsc{MimeMom}}
\DeclareRobustCommand{\mimelite}
{\textsc{MimeLiteMom}}

\DeclareRobustCommand{\fedcm}
{\textsc{FedCM}}
\DeclareRobustCommand{\fedadc}
{\textsc{FedADC}}

\DeclareRobustCommand{\cifar}[1]
{\textsc{Cifar-#1}}
\DeclareRobustCommand{\shakespeare}
{\textsc{Shakespeare}}
\DeclareRobustCommand{\stackoverflow}
{\textsc{StackOverflow}}
\DeclareRobustCommand{\landmarks}
{\textsc{GLDv2}}
\DeclareRobustCommand{\inaturalist}
{\textsc{INaturalist}}

\DeclareRobustCommand{\resnet}
{\textsc{ResNet-20}}
\DeclareRobustCommand{\lenet}
{\textsc{CNN}}
\DeclareRobustCommand{\mobilenet}
{\textsc{MobileNetV2}}
\DeclareRobustCommand{\vit}
{\textsc{ViT-B\textbackslash16}}

\DeclareRobustCommand{\nan}
{\textcolor{BrickRed}{\xmark}}
\DeclareRobustCommand{\no}
{\textcolor{BrickRed}{\xmark}}
\DeclareRobustCommand{\yes}
{\textcolor{MaterialGreen800}{\ding{51}}}

\DeclareRobustCommand{\fedhbm}
{\textsc{FedHBM}}
\DeclareRobustCommand{\ghb}
{\textsc{GHBM}}
\DeclareRobustCommand{\localghb}
{\textsc{LocalGHBM}}

\DeclareRobustCommand{\S}[1]
{\mathcal{S}^{#1}}

\DeclareRobustCommand{\svrvar}[1]
{{\mathcal{E}}_{#1}}

\DeclareRobustCommand{\fsdrift}[1]
{\Xi_{#1}}

\DeclareRobustCommand{\drift}[1]
{\mathcal{U}_{#1}}

\DeclareRobustCommand{\normsq}[1]
{\left\|#1\right\|^2}

\DeclareRobustCommand{\expect}[1]
{\mathbb{E}\left[#1\right]}

\DeclareRobustCommand{\expectcs}[2]
{\mathbb{E}_{\S{#2} \sim\mathcal{U}(\S{})}\left[#1\right]}

\DeclareRobustCommand{\expectd}[2]
{\mathbb{E}_{{d_{#2} \sim \mathcal{D}_{#2}}}\left[#1\right]}

\crefname{equation}{Eq.}{Eq.}
\Crefname{equation}{Equation}{Equations}
\creflabelformat{equation}{(#2#1#3)}

\crefname{figure}{Fig.}{Figs.}
\Crefname{figure}{Figure}{Figures}

\crefname{table}{Tab.}{Tabs.}
\Crefname{table}{Table}{Tables}

\crefname{chapter}{Chap.}{Chaps.}
\Crefname{chapter}{Chapter}{Chapters}

\crefname{section}{Sec.}{Secs.}
\Crefname{section}{Section}{Sections}

\crefname{subsection}{Sec.}{Secs.}
\Crefname{subsection}{Section}{Sections}

\crefname{subsubsection}{Sec.}{Secs.}
\Crefname{subsubsection}{Section}{Sections}

\crefname{prob}{Prob.}{Probs.}
\Crefname{prob}{Problem}{Problems}

\crefname{property}{Prop.}{Props.}
\Crefname{property}{Property}{Properties}

\crefname{constr}{Constraint}{Constraints} %
\Crefname{constr}{Constraint}{Constraints}

\crefname{algocf}{Algorithm}{Algorithms}
\Crefname{algocf}{Algorithm}{Algorithms}

\crefname{algorithm}{Algorithm}{Algorithms}
\Crefname{algorithm}{Algorithm}{Algorithms}

\crefname{hyp}{Hyp.}{Hyps.}
\Crefname{hyp}{Hypothesis}{Hypothesis}

\crefname{assumption}{Assumption}{Assumptions}
\Crefname{assumption}{Assumption}{Assumptions}

\crefname{lem}{Lem.}{Lems.}
\Crefname{lem}{Lemma}{Lemma}

\crefname{prop}{Prop.}{Props.}
\Crefname{prop}{Proposition}{Propositions}

\crefname{theorem}{Thm.}{Thms.}
\Crefname{theorem}{Thm.}{Thms.}

\usepackage{aligned-overset}
\usepackage{parskip}
\usepackage[flushleft]{threeparttable}

\let\classAND\AND
\let\AND\relax
\usepackage{algorithmic}

\let\AND\classAND
\AtBeginEnvironment{algorithmic}{\let\AND\algoAND}

\title{Communication-Efficient Heterogeneous \\ Federated Learning with Generalized \\ Heavy-Ball Momentum}

\author{%
  \name Riccardo Zaccone\thanks{Corresponding author} \\
  \addr Politecnico di Torino \\
  \email riccardo.zaccone@polito.it
  \AND
  \name Sai Praneeth Karimireddy \\
  \addr USC Viterbi School of Engineering \\
  \email karimire@usc.edu
  \AND
  \name Carlo Masone \\
  \addr Politecnico di Torino \\
  \email carlo.masone@polito.it
  \AND
  \name Marco Ciccone \\
  \addr Vector Institute \\
  \email marco.ciccone@vectorinstitute.ai
}

\def\month{06}  %
\def\year{2025} %
\def\openreview{\url{https://openreview.net/forum?id=LNoFjcLywb}} %

\begin{document}

\maketitle

\begin{abstract}
Federated Learning (FL) has emerged as the state-of-the-art approach for learning from decentralized data in privacy-constrained scenarios.
However, system and statistical challenges hinder its real-world applicability, requiring efficient learning from edge devices and robustness to data heterogeneity.
Despite significant research efforts, existing approaches often degrade severely due to the joint effect of heterogeneity and partial client participation.
In particular, while momentum appears as a promising approach for overcoming statistical heterogeneity, in current approaches its update is biased towards the most recently sampled clients.
As we show in this work, this is the reason why it fails to outperform {\fedavg}, preventing its effective use in real-world large-scale scenarios.
In this work, we propose a novel \textit{Generalized Heavy-Ball Momentum} (\ghb) and theoretically prove it enables convergence under unbounded data heterogeneity in \textit{cyclic partial participation}, thereby advancing the understanding of momentum's effectiveness in FL.
We then introduce adaptive and communication-efficient variants of {\ghb} that match the communication complexity of {\fedavg} in settings where clients can be \textit{stateful}.
Extensive experiments on vision and language tasks confirm our theoretical findings, demonstrating that {\ghb} substantially improves state-of-the-art performance under random uniform client sampling, particularly in large-scale settings with high data heterogeneity and low client participation\footnote{Code is available at \href{https://rickzack.github.io/GHBM}{\texttt{https://github.com/RickZack/GHBM}}}.
\end{abstract}
\section{Introduction}
Federated Learning (FL) \citep{mcmahan2017FedAvg} is a paradigm to learn from decentralized data in which a central server orchestrates an iterative two-step training process that involves 1) local training, potentially on a large number of clients, each with its own private data, and 2) the aggregation of these updated local models on the server into a single, shared global model. This process is repeated over several communication rounds.
While the inherent privacy-preserving nature of FL makes it well-suited for decentralized applications with restricted data sharing, it also introduces significant challenges. Since local data reflects unique characteristics of individual clients, limiting the optimization to a client's personal data can lead to issues caused by \textit{statistical heterogeneity}.
This becomes particularly problematic when multiple optimization steps are performed before model synchronization, causing clients to \textit{drift} from the ideal global updates \citep{karimireddy2020scaffold}. Indeed, heterogeneity has been shown to hinder the convergence of \fedavg{} \citep{hsu2019measuring}, increasing the number of communication rounds needed to achieve a target model quality \citep{reddi2020FedOpt} and negatively impacting final performance.

Several studies have proposed solutions to mitigate the effects of heterogeneity. For instance, \scaffold~\citep{karimireddy2020scaffold} relies on additional control variables to correct the local client’s updates, while {\feddyn}~\citep{acar2021FedDyn} uses ADMM to align the global and local client solutions.
Albeit theoretically grounded, experimentally these methods are not sufficiently robust to handle extreme heterogeneity, low client participation, or large-scale problems, exhibiting slow convergence and instabilities~\citep{varno2022AdaBest}. 

Momentum-based FL methods show promise in addressing these challenges. By accumulating past update directions, momentum can help clients overcome the inconsistencies of local objectives introduced by heterogeneous data.
Several works explored incorporating momentum in FL, either at the server \citep{hsu2019measuring} or at client-level to correct local updates \cite{Ozfatura2021FedADC,xu2021fedcm}. Notably, \mime~\citep{karimireddy2021Mime} has been proposed as a framework to make clients mimic the updates of a centralized model trained on i.i.d. data by leveraging extra server statistics at the client side. 
While the theoretical advantages of momentum in FL have been demonstrated under \textit{full participation} \cite{cheng2024momentum}, it has been shown, both theoretically and experimentally, that its effectiveness is limited when client participation varies across training rounds. Indeed, the only momentum-based FL method that operates under \textit{partial participation} and does not rely on assumptions on bounded gradient heterogeneity, \textsc{SCAFFOLD-M} \citep{cheng2024momentum}, still relies on variance reduction - similarly to {\scaffold} - to contrast heterogeneity. As a result, it inherits both the limitations of variance reduction in deep learning \citep{defazioineffvr} and the drawbacks of {\scaffold} in FL, as highlighted by \cite{reddi2020FedOpt}.
In practice, as our work shows, existing momentum-based FL methods exhibit significant limitations in settings with low participation, high heterogeneity, and real-world large-scale problems.
Moreover, current approaches often incur increased communication costs due to the additional information exchanged to correct local updates \citep{karimireddy2020scaffold, karimireddy2021Mime, xu2021fedcm, Ozfatura2021FedADC}. 
This can be a significant drawback in communication-constrained environments, further hindering the practical adoption of FL in real-world applications and highlighting the critical need for more robust, effective, and communication-efficient FL algorithms.
In this work, we provide a theoretical justification for the ineffectiveness of classical momentum in FL demonstrating that due to the interplay of data heterogeneity and partial participation, the momentum term is updated with a biased estimate of the global gradient, reducing its effectiveness in correcting client drift.
To address these challenges, we propose a novel \textit{Generalized Heavy-Ball} (\ghb) formulation, which computes momentum as a decayed average of the past $\tau$ momentum terms. This design reduces bias toward the most recently selected clients, enabling convergence under arbitrary heterogeneity, not only in full participation but also in \textit{cyclic partial participation}. We then propose \fedhbm{},  an adaptive and communication-efficient instantiation of {\ghb}, and experimentally demonstrate its significantly improved performance over state-of-the-art methods.
\paragraph{Contributions.} We summarize our main results below.
\begin{itemize}[noitemsep,nolistsep,leftmargin=*]
    \item We present a novel formulation of momentum called \textit{Generalized Heavy-Ball} (\ghb) momentum, which extends the classical heavy-ball \citep{polyak1964heavyball}, and propose variants that are robust to heterogeneity and communication-efficient by design.
    \item We establish the theoretical convergence rate of {\ghb} for non-convex functions, extending the previous result of \cite{cheng2024momentum} of classical momentum, showing that {\ghb} converges under arbitrary heterogeneity even (and most notably) in \textit{cyclic partial participation}. 
    \item We empirically show that existing FL algorithms suffer severe limitations in extreme non-iid scenarios and real-world settings. In contrast, {\ghb} is extremely robust and achieves higher model quality with significantly faster convergence speeds than other client-drift correction methods.
\end{itemize}
\section{Related works}
\paragraph{The Problem of Statistical Heterogeneity.}
The detrimental effects of non-iid data in FL were first observed by \citep{zhao2018federated}, who proposed mitigating performance loss by broadcasting a small portion of public data to reduce the divergence between clients' distributions. Alternatively, \citep{li2019fedmd} uses server-side public data for knowledge distillation. Both approaches rely on the strong assumption of readily available and suitable data. 
Recognizing weight divergence as a source of performance loss, \fedprox~\citep{li2020FedProx} adds a regularization term to penalize divergence from the global model. Nevertheless, this was proved ineffective in addressing data heterogeneity~\cite{caldarola2022improving}. Other works \citep{kopparapu2020fedfmc, zaccone2022FedSeq, zeng2022fedGSP, Caldarola_2021_CVPR} explored grouping clients based on their data distribution to mitigate the challenges of aggregating divergent models.

\paragraph{Stochastic Variance Reduction in FL.}
Stochastic variance reduction techniques have been applied in FL \citep{chen2020FedSVRG, li2019FedDANE} with SCAFFOLD~\cite{karimireddy2020scaffold} providing for the first time convergence guarantees for arbitrarily heterogeneous data. The authors also shed light on the \textit{client-drift} of local optimization, which results in slow and unstable convergence. SCAFFOLD uses control variates to estimate the direction of the server model and clients' models %
and to correct the local update. This approach requires double the communication to exchange the control variates, and it is not robust enough to handle large-scale scenarios akin to cross-device FL \citep{reddi2020FedOpt, karimireddy2021Mime}.
Similarly, \textsc{SCAFFOLD-M} \citep{cheng2024momentum} integrates classical momentum into {\scaffold} to attain a slightly better convergence rate and maintain robustness to unbounded heterogeneity in partial participation. However, it still relies on variance reduction to tackle heterogeneity, inheriting and the same limitations of {\scaffold}, as the ineffectiveness of variance reduction in deep learning \citep{defazioineffvr}.

\paragraph{ADMM and Adaptivity.}
Other methods are based on the Alternating Direction Method of Multipliers \citep{chen2022OADMM, gong2022fedadmm, wang2022fedadmm}. In particular, \feddyn \citep{acar2021FedDyn} dynamically modifies the loss function such that the model parameters converge to stationary points of the global empirical loss. Although technically it enjoys the same convergence properties of SCAFFOLD without suffering from its increased communication cost, in practical cases it has displayed problems in dealing with pathological non-iid settings \citep{varno2022AdaBest}. Other works explored the use of adaptivity to speed up the convergence of FedAvg and reduce the communication overhead \citep{xie2019AdaAlter, reddi2020FedOpt}.

\paragraph{Use of Momentum as Local Correction.}
As a first attempt, \citet{hsu2019measuring} adopted momentum at server-side to reduce the impact of heterogeneity.
\edit{%
With a similar idea, \cite{kim2024communication} use the Nesterov Accelerated Gradient (NAG) to broadcast a lookahead global model and adds a proximal local penalty similar to {\fedprox} (additional details in \cref{app:discussion}).}
However, \edit{server-side momentum} has been proven of limited effectiveness under high heterogeneity, because the drift happens at the client level. This motivated later approaches that apply server momentum at each local step \citep{Ozfatura2021FedADC, xu2021fedcm}, and the more general approach by \citet{karimireddy2021Mime} to adapt any centralized optimizer to cross-device FL. It employs a combination of control variates and server optimizer state (\eg momentum) at each client step, which lead to increased communication bandwidth and frequency. A recent similar approach \citep{das2022faster} employs \edit{quantized} updates, still requiring significantly more computation client-side.
Rather differently from previous works, we propose a novel formulation of momentum specifically designed to take incorporate the descent information of clients selected at past $\tau$ rounds, which generalizes the classical heavy-ball \citep{polyak1964heavyball}. Most notably, we prove that our {\ghb} algorithm converges under arbitrary heterogeneity in cyclic partial participation - the first momentum method achieving this result without relying on other mechanisms like variance reduction.

\paragraph{Lowering Communication Requirements in FL.} 
Researchers have studied methods to reduce the memory needed for exchanging gradients in the distributed setting, for example by quantization \citep{alistarh2017QSGD} or by compression \citep{mishchenko2019distributed,koloskova2020Decentralized}.
In the context of FL, such ideas have been developed to meet the communication and scalability constraints \citep{reisizadeh2020FedPAQ}, and to take into account heterogeneity \citep{sattler2020STC}. Our work focuses on a novel formulation of momentum that takes into account the joint effects of heterogeneity and partial participation, and that has a heavy-ball structure allowing efficient use of the information already being sent in vanilla {\fedavg}, so additional techniques to compress that information remain orthogonal to our approach.
\section{Method}

\subsection{Setup}
In FL a server and a set $\S{}$ of clients collaboratively solve a learning problem, with $|\S{}|=K \in \mathbb{N^+}$. At each round $t \in [T]$, a fraction of $C \in (0,1]$ clients from $\S{}$ is selected to participate to the learning process: we denote this portion as $\S{t}\subseteq\S{}$. Each client $i \in \S{t}$ receives the server model $\theta_i^{t, 0} \equiv \theta^{t-1}$, and performs $J$ local optimization steps, using stochastic gradients $\tilde{g}_i^{t,j}$ evaluated on local parameters $\theta_i^{t,j-1}$ and a batch $d_{i,j}$, sampled from its local dataset $\mathcal{D}_i$. 
During local training, $\theta_i^{t,j}$ is the model of client $i$ at round $t$ after the $j$-th optimization step, while $\theta_i^{t} \equiv \theta^{t, J_i}$ is the model sent back to the server.
The server then aggregates the client updates $\tilde{g}_i^t:=(\theta^{t-1} - \theta_i^t)$, building \textit{pseudo-gradients} $\tilde{g}^t$ that are used to update the model \citep{reddi2020FedOpt}.

In this work we formalize the learning objective as a finite-sum optimization problem, where each function is the local clients' loss function with only access to that client's stochastic samples:
\begin{equation}
    \label{eq:objective_cross_silo}
    \arg\min_{\theta \in \mathbb{R}^d}
    \left[f(\theta):=\frac{1}{|\S{}|}\sum_{i\in \mathcal{S}} \left(f_i(\theta) :=\mathbb{E}_{d_i\sim{\mathcal{D}_i}}[f_i(\theta; d_i)] \right)\right]
\end{equation}

The analysis we provide in \cref{sec:theory:conv_rate} is based on the above formalization of the learning problem, which is commonly used to model \textit{cross-silo} FL settings, hence our theoretical results apply to that kind of scenarios.
In this context, we prove that \ghb{} converges under unbounded heterogeneity relying solely on momentum, expanding the understanding of its effectiveness compared to other methods that rely on \textit{variance reduction} or ADMM to achieve this result \citep{karimireddy2020scaffold, cheng2024momentum, acar2021FedDyn}.
On the other hand, it has been proved that it is not possible to guarantee convergence under arbitrary heterogeneity in the \textit{``stochastic''} or \textit{``streaming''} context which is commonly used for modeling \textit{cross-device} FL (see the lower bound in Theorem 3.4 of \cite{patel2022celsgd}), so considering it in our formal analysis would be of limited usefulness. 
Hence, we focus the theoretical analysis on the former case. Nevertheless, we also provide large-scale experimental validation on settings that adhere to the characteristics of  \textit{cross-device} FL to demonstrate that \ghb{} is suitable for such real-world scenarios (see \cref{sec:method_localghb}).

\subsection{Addressing Client Drift with Momentum}
One of the core propositions of federated optimization is to take advantage of local clients' work, by running multiple optimization steps on local parameters before synchronization. This has been proven effective for speeding up convergence when local datasets are i.i.d. with respect to a global distribution \citep{stich2018local, Lin2020, mcmahan2017FedAvg}, and is particularly important for improving communication efficiency, which is the bottleneck when learning in decentralized settings.
However, the statistical heterogeneity of clients' local datasets causes local models to \textit{drift} from the ideal trajectory of server parameters.
One way of addressing such drift is to use momentum during local optimization, based on the idea that a moving average of past server pseudo-gradients can correct local optimization towards the solution of the global problem.
At each round, FL methods based on momentum typically use the gradients of the selected clients, whether computed at local \citep{xu2021fedcm, Ozfatura2021FedADC} or global \citep{karimireddy2021Mime} parameters, to update the momentum term server-side. 

\paragraph{Partial Participation and Biased Momentum.} We claim that existing momentum-based methods overlook a critical aspect of federated learning: \textit{partial client participation}. 
Indeed, when only a portion of clients participate in the training rounds, the server pseudo-gradient used to update the momentum estimate can be biased towards the previously selected clients, hampering its corrective benefit to local optimization. 
This effect is particularly pronounced in settings with high data heterogeneity and low client participation (common in cross-device FL), where, as our experiments demonstrate, conventional momentum fails to correct the drift and improve over vanilla FedAvg.

\paragraph{Main Contribution.}
To address the challenges posed by partial participation, we propose a novel momentum-based approach that explicitly accounts for client sampling. Our key idea is to update the momentum term using a pseudo-gradient that approximates the true global gradient over all clients, including those not participating in the current round. By integrating the descent directions from past rounds into local updates, our method effectively mitigates the bias introduced by partial participation, resulting in a more accurate and robust momentum estimate. Notably, our momentum formulation retains a heavy-ball structure similar to classical momentum, enabling it to be used in FL without requiring to send additional data from server to clients, thus maintaining the same communication complexity as FedAvg.

\subsection{Generalized Heavy-Ball Momentum (GHBM)}
\label{sec:method_ghb}
In this section, we introduce our novel formulation for momentum, which we call \textit{Generalized Heavy-Ball Momentum} (\ghb).
First, we recall that classical momentum consists of a moving average of past gradients, and it is commonly expressed as in \cref{eq:mom_exp}, which can be equivalently expressed in a version commonly referred to as \textit{heavy-ball momentum} in \cref{eq:mom_hb} (see \cref{lemma:expr_hb}):

\begin{minipage}{0.55\textwidth}
{\small\textbf{\textsc{Heavy-Ball Momentum (HBM)}}}
\begin{align}
\label{eq:mom_exp}
    \tilde{m}^t &\leftarrow \beta \tilde{m}^{t-1} + \tilde{g}^t(\theta^{t-1}; \mathcal{D}^t) \\
    \theta^{t} &\leftarrow \theta^{t-1} - \eta \tilde{m}^t \nonumber
\end{align}
\end{minipage}
\begin{minipage}{0.4\textwidth}
\textbf{}
\begin{align}
    \label{eq:mom_hb}
    \tilde{m}^t &\leftarrow (\theta^{t-1}-\theta^{t-2}) \\
    \theta^{t} &\leftarrow \theta^{t-1} - \eta \tilde{g}^t(\theta^{t-1}; \mathcal{D}^t) + \beta\tilde{m}^t \nonumber
\end{align}    
\end{minipage}

Let us notice that, when applied to FL optimization, the gradient referred to above as $\tilde{g}^t$ is built from updates of clients $i \in \S{t}$ (and so on dataset $\mathcal{D}^t:=\cup_{i \in \S{t}}\mathcal{D}_i$), which are usually a small portion of all the clients participating in the training. Consequently, at each round the momentum is updated using a direction biased towards the distribution of clients selected in that round.
Indeed, the prerequisites for this update to reflect the objectives of the other clients are (i) iidness of local datasets or (ii) high client participation. Both conditions are rarely met in practice, and lead to ineffectiveness of existing momentum-based FL methods in realistic scenarios.
Our objective is to update the momentum term at each round with a reliable estimate of the gradient \wrt the global data distribution of all clients. 
In practice, the desired update rule for momentum would use the average gradient of all clients selected in the last $\tau$ rounds at current parameters $\theta^{t-1}$, as in \cref{eq:proto_mom_update}.

\begin{minipage}{0.45\textwidth}
{\small\textbf{\textsc{Desired Momentum Update}}}
\begin{align}
\label{eq:proto_mom_update}
    \tilde{m}^t &\leftarrow \beta \tilde{m}^{t-1} + \frac{1}{\tau}\sum_{k=t-\tau+1}^{t} \tilde{g}^{k}(\theta^{t-1}; \mathcal{D}^k)
\end{align} 
\end{minipage}
\qquad\qquad
\begin{minipage}{0.45\textwidth}
{\small\textbf{\textsc{Practical Momentum Update}}}
\begin{align}
\label{eq:proto_mom_update_real}
    \tilde{m}^t &\leftarrow \beta \tilde{m}^{t-1} + \frac{1}{\tau}\sum_{k=t-\tau+1}^{t} \tilde{g}^{k}(\theta^{k-1}; \mathcal{D}^k)
\end{align} 
\end{minipage}

While \cref{eq:proto_mom_update} cannot be implemented in partial participation because clients selected in rounds $k \in [t-\tau+1, t)$ do not have access to model parameters $\theta^{t-1}$, it is possible to reuse old gradients calculated at parameters $\theta^{k-1}$ as their approximation, as shown in \cref{eq:proto_mom_update_real}.
This introduces a \textit{lag} due to using outdated gradients. However,  as we show \cref{fig:grad_dev_tau}, the benefits of reducing heterogeneity greatly compensate for this lag, as increasing $\tau$ leads to a reduction in the deviation from the gradient calculated over all the clients.

\begin{figure}
    \centering
    \includegraphics[width=0.48\linewidth]{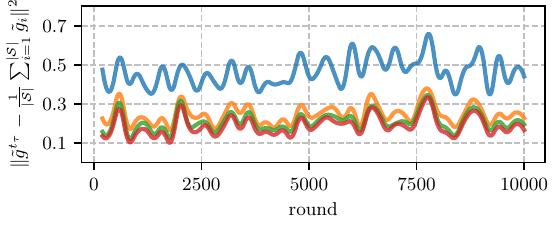}
    \quad
    \includegraphics[width=0.48\linewidth]{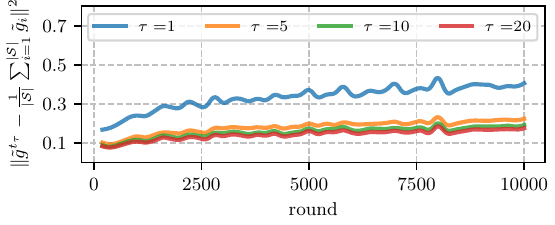} 
    \vspace{-0.3cm}
    \caption{\textbf{Reusing old gradients is beneficial, despite the introduced lag.} The plot shows the empirical measure of the deviation between (i) the average of the last $\tau$ server pseudo-gradient (at different parameters) and (ii) the server-pseudo gradient calculated over all the clients (at the same parameters), varying $\tau$, on \cifar{100} with {\resnet}, in non-iid ($\alpha=0$, left) and iid (\edit{$\alpha=10.000$}, right) settings. }
    \label{fig:grad_dev_tau}
    \vspace{-0.5cm}
\end{figure}

With this idea in mind, our proposed formulation consists of calculating the momentum term as the decayed average of past $\tau$ momentum terms, instead of explicitly using the server pseudo-gradients at the last $\tau$ rounds, as shown in \cref{eq:ghb_exp}.
This formulation is close to the update rule sketched in \cref{eq:proto_mom_update_real} and has the additional advantage of enjoying a heavy-ball form similar to \cref{eq:mom_hb} (see \cref{lemma:expr_ghb}), which will be useful for deriving communication-efficient FL algorithms.
In practice, the difference \wrt \cref{eq:mom_hb}  consists in considering a delta $\tau>1$:

\begin{minipage}{0.45\textwidth}
{\small\textbf{\textsc{Generalized Heavy-Ball Momentum (GHBM)}}}
\begin{align}
\label{eq:ghb_exp}
    \tilde{m}_\tau^t &\leftarrow \frac{1}{\tau}\sum_{k=1}^{\tau} \beta \tilde{m}_\tau^{t-k} + \tilde{g}^t(\theta^{t-1}; \mathcal{D}^t) \\
    \theta^t &\leftarrow \theta^{t-1} - \eta \tilde{m}_\tau^{t} \nonumber
\end{align} 
\end{minipage}
\begin{minipage}{0.538\textwidth}
\vspace{5pt}
\vphantom{{\small\textbf{\textsc{Generalized Heavy-Ball Momentum (GHBM)}}}}
\begin{align}
\label{eq:ghb}
    \tilde{m}_\tau^t &\leftarrow \frac{1}{\tau}\left(\theta^{t-1}-\theta^{t-\tau-1}\right) \\ 
    \theta^{t} &\leftarrow \theta^{t-1} - \eta \tilde{g}^t(\theta^{t-1}; \mathcal{D}^t) + \beta\tilde{m}_\tau^t \nonumber
\end{align} 
\end{minipage}

As it is trivial to notice, {\ghb} with $\tau=1$ recovers the classical momentum, hence it can be considered as a generalized formulation.
The {\ghb} term is then embedded into local updates using the heavy-ball form shown in \cref{eq:ghb}, leading to the following update rule:
\begin{align}
\label{eq:ghb_fl}
    \textsc{\pmb{Client step:}} \qquad \qquad &\theta_i^{t,j} \leftarrow \theta_i^{t, j-1} - \eta_l \tilde{g}_i^{t, j}(\theta_i^{t, j-1}; d_i^{t,j}) + \underbrace{\frac{\beta}{\tau J}\left(\theta^{t-1} - \theta^{t-\tau-1} \right)}_{\tau-\ghb}
\end{align}
\vspace{-5mm}
\paragraph{Discussion on $\pmb{\tau}$.}
\label{par:discussion_tau}
The $\tau$ hyperparameter in {\ghb} plays a crucial role, since it controls the number of server pseudo-gradients to average when estimating the update to the momentum term. 
Intuitively, when considering only the effect on heterogeneity reduction, the optimal value would be the one that provides the average over all clients.
Under proper assumptions on client sampling (see \cref{sec:assumptions}), this optimal value is $\tau=\nicefrac{1}{C}$, which is the inverse of the client participation rate.
As we demonstrate, this property is the key factor that allows {\ghb} to converge under arbitrary heterogeneity, achieving the same convergence rate in \textit{cyclic partial participation} as methods based on classical momentum attain in \textit{full participation} (see Sec. \ref{sec:theory:conv_rate}).
However, because {\ghb} reuses old gradients, it introduces a \textit{lag} that grows with $\tau$. Therefore, the optimal choice of $\tau$ comes with an inevitable trade-off between the heterogeneity reduction effect and other sources of error, which we discuss in \cref{par:overall_error_ghbm}.

\subsection{Communication Complexity of {\ghb} and Efficient Variants}
\label{sec:method_localghb}
\begin{wrapfigure}{Lh}{0.52\textwidth}
\vspace{5mm}
\begin{minipage}{0.52\textwidth}
\input{algorithms/ghbm_main}\end{minipage}
\vspace{-17mm}
\end{wrapfigure}
As it is possible to notice from \cref{algo:ghbm_main}, {\ghb} requires the server to additionally send the past model $\theta^{t-\tau-1}$, which is used to calculate the momentum term in \cref{eq:ghb_fl}. Alternatively, the server could send the momentum term $\tilde{m}^t_\tau$: in both cases, this introduces a communication overhead of $1.5\times$ {\wrt} {\fedavg}, as momentum is usually applied to all model parameters.
However, this overhead can be avoided by leveraging the observation that the choice of $\tau=\nicefrac{1}{C}$ is expected to be optimal. Indeed, it is sufficient to notice that, if clients participate cyclically, \ie, the period between each subsequent sampling is equal for all clients, and the frequency at which each client is selected for training is exactly $\nicefrac{1}{C}$. Notice that this is still true on average under uniform client sampling, \ie, calling $\tau_i$ the sampling period for client $i$, $\expect{\tau_i}=\tau=\nicefrac{1}{C}$. 

\vspace{2mm}
Leveraging those observations and exploiting the fact that {\ghb} has an equivalent heavy-ball form, the additional requirement on communication can be traded for a requirement on persistent storage at the clients, allowing them to keep the model received by the server across rounds, as shown in \cref{algo:ghbm_main}.  In this algorithm, which we call \textbf{\localghb}, $\tau_i$ is adaptive and determined stochastically by client participation.
The space complexity is constant in the size of model parameters for the clients and the communication complexity is the same as {\fedavg}.
We empirically found that performance can be further improved by considering $\theta_{i,j}^t$ instead of $\theta^{t-1}$ and $\theta_{i}^{t-\tau_i}$ instead of $\theta^{t-\tau_i-1}$ when calculating $\tilde{m}^t_{\tau_i}$. This final communication-efficient
update rule is named \textbf{\fedhbm}.
\newpage
\paragraph{Applicability of {\ghb}-based Algorithms in FL Scenarios.}
Although based on the same principle, our algorithms are suitable for different scenarios.
Similarly to algorithms proposed for cross-device FL \citep{karimireddy2021Mime}, {\ghb} uses \textit{stateless} clients, with the main $\tau$ hyperparameter controlled by the server. 
This ensures that clients always apply a momentum term consistent with the {\ghb} update rule, differently from algorithms that require clients participating in multiple rounds to adhere to their formulation, such as {\scaffold} and {\feddyn}. 
This is particularly important when the number of clients is large and a small portion of them participates in each round, and it is why, in our large-scale setting, these methods fail to converge.
These design choices make our algorithm in practice suitable for cross-device FL, where it offers significant advantages, as experimentally validated in \cref{sec:experiments:large_scale}.
On the other hand, {\fedhbm} and {\localghb} take advantage of the fact that clients participate multiple times in the training process to remove the need to send the momentum term from the server, recovering the same communication complexity of {\fedavg}. As a result, clients in these methods are \textit{stateful} - requiring to maintain variables across rounds \citep{kairouz2021advances} - and are therefore best suited for scenarios akin to \textit{cross-silo} FL.
\section{Theoretical Discussion}
\begin{table}[t]
\centering
\caption{\small 
\textbf{Comparison of convergence rates of FL algorithms.} {\ghb} improves the state-of-art by attaining, in \textit{cyclic partial participation}, the same rate of classical momentum in \textit{full participation}.
Remind that $L$ is the smoothness constant of objective functions, $\Delta=f(\theta^0)-\min_\theta f(\theta)$ is the initialization gap, $\sigma^2$ is the clients' gradient variance, $|\S{}|$ is the number of clients, $C$ is the participation ratio, $J$ is the number of local steps per round, and $T$ is the number of communication rounds.
$\zeta:=\sup_\theta\|\nabla f(\theta)\|$ and $G$ are uniform bounds of gradient norm and dissimilarity.
}
    \label{tab:conv_rates}

    \resizebox{\linewidth}{!}{
    \begin{threeparttable}
    \begin{tabular}{lllc}
    \toprule
    {\bf Algorithm}  & {\centering\bf Convergence Rate  $\frac{1}{T}\sum_{t=1}^{T}\expect{\|\nabla f(\theta^t)\|^2}\lesssim$} & \thead{\bf Additional \\ \bf Assumptions} & \thead{\bf Partial \\ \bf participation?}\vspace{1mm} \\ 
    \midrule \vspace{-3mm}\\
    \makecell[l]{
    {\sc FedAvg}\\ 
    \quad \citep{yang2021achieving}} & $\left(\frac{L\Delta\sigma^2}{|\S{}|JT}\right)^{1/2} + \frac{L\Delta}{T}$ & Bounded hetero.\tnote{1} & \no \\

    \quad \citep{yang2021achieving} & $\left(\frac{L\Delta J\sigma^2}{|\S{}|CT}\right)^{1/2} + \frac{L\Delta}{T}$ & Bounded hetero.\tnote{1} & \yes \\

    \makecell[l]{
    {\sc FedCM} \\ \quad\citep{xu2021fedcm}} & ${\left(\frac{L\Delta ({\sigma^2}+|\S{}|CJ\zeta^2)}{|\S{}|CJT}\right)^{1/2}+\left(\frac{L\Delta ({\sigma}/{\sqrt{J}}+ \sqrt{|\S{}|C}(\zeta+G)}{\sqrt{|\S{}|C}T}\right)^{2/3}}$ & {\makecell[l]{Bounded grad. \\ Bounded hetero.}} & \yes \\

    \quad \citep{cheng2024momentum} 
    & $\left(\frac{L\Delta \sigma^2}{|\S{}|JT}\right)^{1/2}+\frac{L\Delta}{T}$ &  $-$ & \no \vspace{0mm}  \\

    \makecell[l]{
    \textsc{SCAFFOLD-M} \\ \quad\citep{cheng2024momentum}} & $\left(\frac{L\Delta \sigma^2}{|\S{}|CJT}\right)^{1/2}+\frac{L\Delta}{T} \left(1 + \frac{|\S{}|^{2/3}}{|\S{}|C} \right)$ & $-$ & \yes \\

    \rowcolor{Cerulean!30!white}
    {\bf {\ghb} (\cref{thm:GHBM})}
    & $\left(\frac{L\Delta \sigma^2}{|\S{}|JT}\right)^{1/2}+\frac{L\Delta}{T}$ &  Cyclic participation & \yes \vspace{0mm}  \\
    
    \bottomrule
    \end{tabular}
    \begin{tablenotes}
        \footnotesize
        \item[1] The local learning rate vanishes to zero when gradient dissimilarity is unbounded, \ie, $G\to \infty$.
    \end{tablenotes}
    \end{threeparttable}
    }
\vspace{-5mm}
\end{table}
In this section, we establish the theoretical foundations of our algorithms. Our analysis reveals that: (i) the momentum update rule implemented by {\ghb} in \cref{eq:proto_mom_update_real} approximates an update with global gradient, with $\tau$ controlling the trade-off between heterogeneity reduction and the \textit{lag} due to using old gradients; (ii) thanks to this algorithmic design choice, {\ghb} converges under arbitrary heterogeneity even in (cyclic) partial participation. The proofs are deferred to \cref{appendix:theory}.

\subsection{Assumptions}
\label{sec:assumptions}
To prove our results we rely on notions of stochastic gradient with bounded variance (\ref{assum:unbias_bvar}) and the smoothness of the clients' objective functions (\ref{assum:smoothness}), which are common in deep learning.
Additionally, to facilitate comparisons with other algorithms that require it, we introduce the Bounded Gradient Dissimilarity (BGD) (\cref{assum:bounded_gd}). This assumption, commonly used in FL literature, provides an upper bound on the dissimilarity of clients' objectives.
While our main result in \cref{thm:GHBM} does not require this assumption, we use it to demonstrate the heterogeneity reduction effect of \ghb{}, and to show that, under the proper choice of $\tau$, BGD is not necessary.
Finally, we introduce the additional assumption that clients participate following a cyclic pattern (\cref{assum:cyclic_part}). Notably, this assumption is only required for obtaining our convergence rate and serves as a technical detail needed to deterministically quantify the contributions of the clients to the \ghb{} momentum term \edit{(see \cref{fig:cyclic_illustration} in the Appendix for an illustration of cyclic participation).}

\begin{minipage}[t]{0.45\textwidth}
\vspace{-13pt}
\begin{assumption}[Unbiasedness and bounded variance of stochastic gradient]
    \label{assum:unbias_bvar}
    \begin{align*}
        \expectd{\tilde{g}_i(\theta; d_i)}{i} = g_i(\theta; \mathcal{D}_i)& \\
        \expectd{\normsq{\tilde{g_i}(\theta; d_i) - g_i(\theta; \mathcal{D}_i)}}{i} \leq \sigma^2&
    \end{align*}
\end{assumption}
\end{minipage}\quad\quad
\begin{minipage}[t]{0.5\textwidth}
\vspace{-13pt}
\begin{assumption}[Smoothness of client's objectives]
Let it be a constant $L>0$, then for any $i,\, \theta_1,\, \theta_2$ the following holds:
    \label{assum:smoothness}
    \begin{align*}
    &\normsq{g_i(\theta_1) - g_i(\theta_2)} \leq L^2\normsq{\theta_1 - \theta_2} 
    \end{align*}
\end{assumption}
\end{minipage}

\begin{minipage}[t]{0.45\textwidth}
\vspace{0pt}
\begin{assumption}[Bounded Gradient Dissimilarity]
There exist a constant $G\geq0$ such that, $\forall i,\, \theta$: %
    \label{assum:bounded_gd}
    \begin{equation*}
        \frac{1}{|\S{}|} \sum_{i=1}^{|\S{}|}\normsq{g_i(\theta) - g(\theta)} \leq G^2
    \end{equation*}
\end{assumption}
\end{minipage}\quad\quad
\begin{minipage}[t]{0.5\textwidth}
\vspace{0pt}
\begin{assumption}[Cyclic Participation]
    \label{assum:cyclic_part}
    Let $\S{t}$ be the set of clients sampled at any round $t$. A sampling strategy is \textit{``cyclic``} with period $p=\nicefrac{1}{C}$ if:
    \begin{align*}
     \S{t} &= \S{t-p} &\forall\; t&>p \quad\land\quad
     \S{k} \cap \S{t} = \varnothing &\forall\; k \in (t-p, t)
    \end{align*}
\end{assumption}
\end{minipage}

\vspace{10pt}
\begin{remark}
\textbf{Our main result (\cref{thm:GHBM}) does not require the BGD assumption}: indeed we show that, under a proper choice of $\tau$, the effect of heterogeneity is completely removed from the convergence rate. 
\end{remark}
\vspace{5pt}
\begin{remark}
While \cref{thm:GHBM} relies on \cref{assum:cyclic_part}, \textbf{cyclic participation is not enforced in the experiments}, where we select clients randomly and uniformly, ensuring fair comparison with algorithms that do not need this assumption in their analysis. For a more comprehensive discussion on the role of the cyclic participation assumption in our work, we refer the reader to \cref{sec:theory:discussion_cyclic}.
\end{remark}

\subsection{Overcoming Bounded Gradient Dissimilarity in Partial Participation}
\label{sec:theory_ghb}
In this section, we explain the core elements used in our theory to guarantee convergence under arbitrary heterogeneity for {\ghb}.

\paragraph{Bounding the Participation-induced Heterogeneity.}
Let us recall the main idea behind {\ghb}: because of partial participation, at each round classical momentum is updated using a direction biased towards the distribution of clients selected in that round.
As a result, recalling that {\ghb} recovers classical momentum when $\tau=1$, we begin by bounding the effect of heterogeneity induced by partial client participation on the momentum estimate as a function of $\tau$.
To this end, let us provisionally adopt \cref{assum:bounded_gd} and assume we perform federated optimization with a single full gradient step in partial participation and consider the momentum update in \cref{eq:proto_mom_update}. In this setup, the following lemma holds:
\vspace{4pt}
\begin{lemma}[Deviation of $\tau$-averaged gradient from true gradient]
\label{lemma:tau_ghb}
Define $\S{t}_\tau:=\cup_{k=0}^{\tau-1}\S{t-k}$ as the set of clients selected in the last $\tau$ rounds, and $g^{t_\tau} := \nicefrac{1}{|\S{t}_\tau|}\sum_{i=1}^{|\S{t}_\tau|} g_i^t(\theta^{t-1})$ as the average server pseudo-gradient. 
The approximation of a gradient over the last $\tau$ rounds $g^{t_\tau}$ \wrt the true gradient is quantified by the following:
\begin{equation*}
    \expect{\normsq{g^{t_{\tau}} - \nabla f(\theta^{t-1})}} \leq 8 \expect{ \left(\frac{|\S{}| - |\S{t}_{\tau}|}{|\S{}|}\right)^2} \left( G^2 + \normsq{\nabla f(\theta^{t-1})} \right)
\end{equation*}
\end{lemma}
\vspace{-8pt}
\cref{lemma:tau_ghb} shows that, as $\tau$ increases, the effect of heterogeneity reduces quadratically as the difference between the $|\S{t}|$ and $|\S{t}_\tau|$ approaches to zero.
The deviation is exactly zero when $\S{t}_\tau=\S{}$, \ie the set of clients selected in the last $\tau$ rounds includes all the clients.
While under uniform sampling it is unlikely to realize this condition because of the non-zero probability of sampling the same clients over consecutive rounds, under cyclic participation it is possible to make the above error exactly equal to zero
\footnote{An alternative approach could keep track of gradients of each client and then compute $g^{t_\tau}$ such that it includes the latest gradients of all clients. In that case, cyclic participation is not necessary, but calculating the needed $\tau$ is an instance of the \textit{Batched Coupons Collector} problem~\citep{coupon0, coupon1, coupon2}, for which a closed form solution is unknown. That approach would be unrealistic to implement so, motivated by the strong empirical success of {\ghb}, in our analysis we prefer adopting an additional assumption, and providing guarantees under cyclic client participation}.
\vspace{4pt}
\begin{corollary}
\label{corollary:tau_ghb_cyclic}
Consider \cref{lemma:tau_ghb} and further assume that, at each round of FL training, clients are sampled according to a rule satisfying \cref{assum:cyclic_part}. Then, for any $\tau \in \left(0, \frac{1}{C}\right]$:%
\begin{equation*}
    \expect{\normsq{g^{t_{\tau}} - \nabla f(\theta^{t-1})}} \leq 8  \left(1-\tau C\right)^2  \left( G^2 + \normsq{\nabla f(\theta^{t-1})} \right)
\end{equation*}
\end{corollary}

\begin{remark}
    Under \cref{assum:cyclic_part} and $\tau=\nicefrac{1}{C}$, \textbf{the BGD assumption (\ref{assum:bounded_gd}) is not necessary}, as the two terms in the left-hand side (LHS) of the above inequality are the same by definition.
\end{remark}
\vspace{-4pt}
\paragraph{Bounding the Overall Error in Momentum Update.}
\label{par:overall_error_ghbm}
In the previous paragraph, we established the role of $\tau$ in {\ghb} for counteracting heterogeneity and derived its optimal value \wrt partial client participation. However, our analysis assumed that all clients selected in the last $\tau$ rounds compute a full gradient on the same server parameters. As discussed in \cref{sec:method_ghb}, a more realistic update rule for momentum would reuse past gradients as in \cref{eq:proto_mom_update_real}, computed at local parameters. This is because clients selected in rounds $k \in [t-\tau+1, t)$ do not have access to model parameters $\theta^{t-1}$.
As a result, increasing $\tau$ introduces additional sources of error to the momentum term, which we quantify in the following lemma.
\vspace{4pt}
\begin{lemma}[Bounded Error of Momentum Update]
\label{lemma:delayed_mgrad}
Consider the update rule in \cref{eq:proto_mom_update_real}, and call {\small $\tilde{g}^{t_{\tau}} = \frac{1}{\tau} \sum_{k=t-\tau+1}^{t} \frac{1}{|\S{k}|J} \sum_{i=1}^{|\S{k}|} \sum_{j=1}^{J} \tilde{g}_i^{k,j}({\theta_i^{k,j-1}})$} the server stochastic average pseudo-gradient over the last $\tau$ global steps \edit{and the average server pseudo-gradient at current parameters as {\small $g^{t_\tau} := \nicefrac{1}{|\S{t}_\tau|}\sum_{i=1}^{|\S{t}_\tau|} g_i^t(\theta^{t-1})$}}. Let also define the client drift {\small $\drift{t}:=\frac{1}{|\S{}|J} \sum_{j=1}^{J} \sum_{i=1}^{|\S{}|} \mathbb{E}\|\theta_{i}^{t,j} - \theta^{t-1}\|^2$} and the error of server update {\small $\svrvar{t}:= \mathbb{E}\|\nabla f(\theta^{t-1}) - \tilde{m}^{t+1}_\tau\|^2$}.
Under \cref{assum:unbias_bvar,assum:smoothness,assum:cyclic_part}, it holds that:
\begin{equation*}
    \expect{\normsq{\tilde{g}^{t_\tau} - g^{t_\tau} }} \leq 3\bigg(
    \underbrace{\frac{\sigma^2}{|\S{t}_\tau|J}\vphantom{\sum^{t}_{k=t-\tau+1}\drift{k}}}_{\text{(a) Noise}} + 
    \underbrace{\frac{L^2}{\tau} \sum^{t}_{k=t-\tau+1}\drift{k}}_{\text{(b) Client drift}} + 
    \underbrace{2L^2\eta^2 \sum_{k=t-\tau+1}^{t-1}\left(\expect{\normsq{\nabla f(\theta^{k-1})}} + \svrvar{k} \right) }_{\text{(c) Gradient lag}} \bigg)
\end{equation*}
\end{lemma}

\cref{lemma:delayed_mgrad} shows that the error affecting the {\ghb} momentum update rule can be decomposed into three main components: 
the first term \textbf{(a)} is caused by clients taking stochastic gradients on mini-batches of data. The dependency indicates that increasing $\tau$ has a positive effect until the gradients of all clients participate to the estimate (\ie $\S{t}_\tau=\S{}$).
The second term \textbf{(b)} represents the average client drift over the last $\tau$ rounds, arising from clients performing multiple local steps. The lemma shows this term has a benign dependency, as increasing $\tau$ does not increase the overall error due to this component.
The last term \textbf{(c)} is the \textit{gradient lag}, which reflects the error introduced by using pseudo-gradients from clients based on old parameters. While this may be the main source of error (since it linearly increases with $\tau$), it depends on $\svrvar{k}$, which is the deviation of server update from the true gradient. If momentum succeeds in correcting local optimization (\ie $\svrvar{k}$ is small), this term will also be small and not hinder the optimization.
We verify experimentally that this is indeed the case: the heterogeneity reduction achieved by increasing $\tau$ outweights the overall error bounded in \cref{lemma:delayed_mgrad}, as showed in \cref{fig:grad_dev_tau}.

\subsection{Convergence Guarantees}
\label{sec:theory:conv_rate}
We can now state the convergence result for {\ghb} for \textbf{\textit{non-convex}} functions in (cyclic) partial participation. Comparison with recent related algorithms is provided in \cref{tab:conv_rates}.
\vspace{5pt}
\begin{theorem}
\label{thm:GHBM}
Under \cref{assum:unbias_bvar,assum:smoothness,assum:cyclic_part}, if we take $\tilde{m}^{0}_\tau=0$, and $\beta$, $\eta$ and $\eta_l$ as in \cref{eq:conv_thm_constraints}, then {\ghb} with $\tau=\nicefrac{1}{C}$ converges as:
   \begin{equation*}
       \frac{1}{T} \sum_{t=1}^{T}\expect{\normsq{\nabla f(\theta^{t-1})}} \lesssim \frac{L\Delta}{T} + \sqrt{\frac{L\Delta \sigma^2}{|\S{}|JT}}
   \end{equation*}
   where $\Delta:=f(\theta^0)-\min_\theta f(\theta)$, \edit{$\eta_l \leq \mathcal{O}\left(\nicefrac{1}{\sqrt{\tau}}\right)$ (see \cref{eq:conv_thm_constraints})} and $\lesssim$ absorbs numeric constants.
\end{theorem}

\paragraph{Discussion.}
The rate of {\ghb} shows two major improvements: (i) it does not rely on the BGD assumption (\ref{assum:bounded_gd}) and (ii) the dominant term on the right-hand side (RHS) scales with the size of all client population $|\S{}|$, instead of the clients selected in a single round $|\S{}|C$, thanks to incorporating old gradients.
While under the assumptions of \cref{thm:GHBM} any $\tau=\frac{k}{C}, \; \forall k \in \mathbb{N^+}$ will lead to similar conclusions, considering larger interval increases the error due to using old gradients (see \cref{par:overall_error_ghbm}), so we would like to choose $\tau$ as the minimum allowing convergence under unbounded heterogeneity. \edit{Indeed, a larger $\tau$ imposes a stricter bound on the client learning rate $\eta_l \leq \mathcal{O}\left(\nicefrac{1}{\sqrt{\tau}}\right)$ in \cref{eq:conv_thm_constraints}. Since \cref{thm:GHBM} also imposes $\tau=\nicefrac{1}{C}$, the bound on $\eta_l$ is explicitly related to the participation ratio $C$.}
\newpage
\begin{wrapfigure}{Rh!}{0.35\textwidth}
    \vspace{-2mm}
    \begin{center}
        \includegraphics[width=\linewidth]{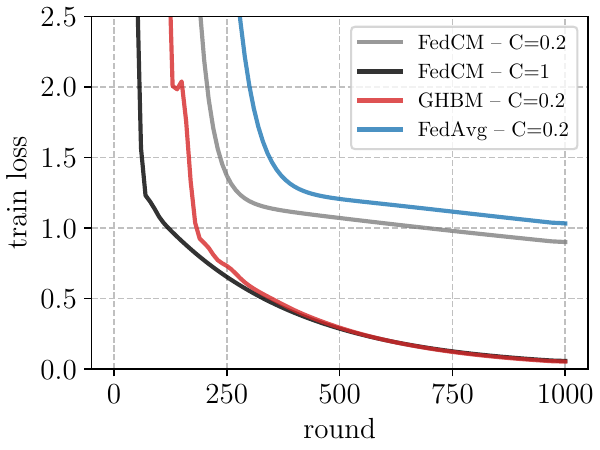}
    \end{center}
    \caption{\textbf{Comparison between {\fedcm} and {\ghb}} in \textit{cyclic participation} on a linear regression problem, non-iid setting, with $J=2$ local steps and $K=10$ clients. {\ghb} with $\tau=\nicefrac{1}{C}$ in \textit{cyclic participation} ($C=0.2$) performs similarly as {\fedcm} in \textit{full participation} ($C=1$).}
    \label{fig:theory_exp}
\vspace{-10mm}
\end{wrapfigure}

\paragraph{Comparison with FedCM.}
The best-known rate for {\fedcm} in partial participation \citep{xu2021fedcm} relies both on bounded gradients and bounded gradient dissimilarity and it is asymptotically weaker than ours.
For the case of \textit{full participation}, \cite{cheng2024momentum} proved that {\fedcm} converges without requiring bounded client dissimilarity.
Our results extend theirs in that we prove that {\ghb} can achieve the same convergence rate even in cyclic partial participation.
This follows from the fact that in this setting {\ghb} update rule approximates the one of classical momentum in full participation.
Indeed, to validate this theoretical finding, in Figure \ref{fig:theory_exp} we simulate a cyclic participation setting and show the train loss of {\ghb} across rounds, comparing with {\fedcm}, both when selecting a subset of clients and when selecting them all. As it is shown, the curve of {\ghb} with $\tau$ as prescribed by \cref{thm:GHBM} approaches the one of {\fedcm} in full participation. 

\paragraph{Comparison with \textsc{SCAFFOLD-M}.}
Recently \cite{cheng2024momentum} proved that momentum accelerates {\scaffold}, preserving strong guarantees against heterogeneity in partial participation.
However, the resulting \textsc{SCAFFOLD-M} method is still based on variance reduction, \ie, it converges under arbitrary heterogeneity thanks to variance reduction, not because it uses momentum.
Our rate additionally requires \cref{assum:cyclic_part}, but is faster and, most importantly, shows that momentum, when modified according to our formulation, can by itself provide similar guarantees even when not all clients participate.
\paragraph{Advantage of Local Steps and Connections to Incremental Gradient Methods.}
\edit{
\cref{thm:GHBM} does not show an explicit benefit from the local steps, similar to the best-known theory for momentum-based FL methods \citep{cheng2024momentum}. However, \ghb{} offers a clear advantage {\wrt} centralized methods for finite-sum optimization applied in FL (where clients represent functions), referred to as \textit{incremental gradient methods}.
One algorithm of this family, the Incremental Aggregated Gradient (IAG), removes the effect of functions heterogeneity by approximating a full gradient with an aggregate of past gradients, assuming cyclic participation \citep{gur2022IGD}. 
However, this holds only in standard distributed mini-batch optimization, where $J=1$.
\ghb{} shares a similar intuition, but applying this logic to the momentum update rather than the gradient estimate is crucial when local steps are involved.
Simply extending IAG with local steps would not mitigate client drift-induced heterogeneity as \ghb{} does. In fact, its convergence rate would be bounded by that of {\fedavg} in full participation, whose lower bound is known to be affected by heterogeneity (see Thm. II of \cite{karimireddy2020scaffold}).
}

\paragraph{On the Use of Cyclic Participation Assumption.}
\label{sec:theory:discussion_cyclic}
\edit{
The use of cyclic participation in the proof of \cref{thm:GHBM} allows precise control over the clients' contributions to the average of the last $\tau$ pseudo-gradients. This ensures that the $\tau$-averaged pseudo-gradient used to update the momentum is unaffected by heterogeneity, which is the important point behind the proof of \cref{thm:GHBM}. Under random uniform, due to the non-zero probability of sampling the same client within $\tau$ rounds,  this condition is hardly verified.
Although one could technically enforce this condition without cyclic sampling — by explicitly tracking each client’s pseudo-gradient and computing a uniform average across the most recent one from each client — this would be impractical. Such a design would not be compliant with protocols like Secure Aggregation, widely adopted in real-world FL systems, thus posing a significant practical limitation.
}

\edit{
Please note that in our analysis convergence under unbounded heterogeneity is not a simple byproduct of the assumption, but comes explicitly from the algorithmic structure of {\ghb} (\ie setting $\tau=\frac{k}{C}, \, \forall k \in \mathbb{N^+}$ is \textbf{necessary}).
The best-known analysis of {\fedavg} under cyclic participation is provided by \cite{cho2023cyclic}, which proves that in certain situations (\eg clients run GD instead of SGD) there can be an asymptotic advantage in the case we prospect with \cref{assum:cyclic_part}. However, it is important to notice that all the results presented in \cite{cho2023cyclic} rely on forms of bounded heterogeneity, and with this respect, the results presented in this work are novel and advance state of the art.
}

\section{Experimental Results}
\label{sec:experiments}
We present evidence both in controlled and real-world scenarios, showing that: (i) the {\ghb} formulation is pivotal to enable momentum to provide an effective correction even in extreme heterogeneity, (ii) our adaptive {\localghb} effectively exploits client participation to enhance communication efficiency and (iii) {\ghb} is suitable for cross-device scenarios, with stark improvement on large datasets and architectures.

\subsection{Setup}
\paragraph{Scenarios, Datasets and Models.}
For the controlled scenarios, we employ \cifar{10/100} as computer vision tasks, with {\resnet} and the same {\lenet} similar to a LeNet-5 commonly used in FL works \citep{hsu2020FederatedVisual}, and {\shakespeare} dataset as NLP task following \citep{reddi2020FedOpt, karimireddy2021Mime}. 
For \cifar{10/100} we follow the common practice of \cite{hsu2020FederatedVisual}, sampling local datasets according to a Dirichlet distribution with concentration parameter $\alpha$, denoting as \textsc{non-iid} and \textsc{iid} respectively the splits corresponding to $\alpha = 0$ and $\alpha=10.000$ (additional details in \cref{app:exp:non_iid}).
For {\shakespeare} we use instead the predefined splits \citep{caldas2019leaf}. The datasets are partitioned among $K=100$ clients, selecting a portion $C=10\%$ of them at each round.
The training round budget $T$ is set to be big enough for all algorithms to reach convergence in the worst-case scenario ($\alpha=0$), constrained by a time budget for the simulations. Being our proposed algorithm always faster, this ensures fair comparison with competitors.

For simulating real-world scenarios, we adopt the large-scale {\landmarks} and {\inaturalist} datasets as CV tasks, with both a {\vit} \citep{dosovitskiy2021an} and a {\mobilenet} \citep{sandler2018mobilenetv2} pretrained on ImageNet, and {\stackoverflow} dataset as NLP task, following \cite{reddi2020FedOpt, karimireddy2021Mime}. These settings are particularly challenging, because the learning tasks are complex, the number of client is high (\ie on the order of $10^4$-$10^5$) and the client participation (for convenience directly reported in \cref{tab:experiments:large_scale}) is scarce (see details in \cref{tab:datasets}).
As is, those settings are akin to \textit{cross-device} FL.
\paragraph{Metrics and Experimental protocol.}
As metrics, we consider \textit{final model quality}, as the average top-1 accuracy over the last 100 rounds of training (\cref{tab:main_controlled,tab:experiments:large_scale}), and \textit{communication/computational efficiency}: this is evaluated by measuring the total amount of exchanged bytes (\ie considering both the downlink and uplink communication) and the wall-clock time spent by an algorithm to reach the performance of {\fedavg} (\cref{tab:comm_comp_cost}).
\edit{We also provide full convergence curves for a subset of the experiments in \cref{fig:full_conv_curves_main}.}
Results are always reported as the average over $5$ independent runs, performed on the best-performing hyperparameters extensively searched separately for all competitor algorithms.
\textbf{All the experiments are conducted under random uniform client sampling}, as it is standard practice.
Further details on datasets, splits, models and hyperparameters are in \cref{app:exp}.

\subsection{The Effectiveness of {\ghb} Compared to Classical Momentum}
\label{sec:exp_ghb}
\begin{figure}[t]
    \centering
    \includegraphics[ width=0.39\textwidth]{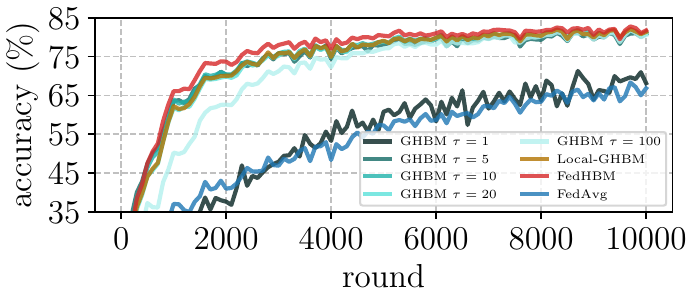}
    \includegraphics[ width=0.39\textwidth]{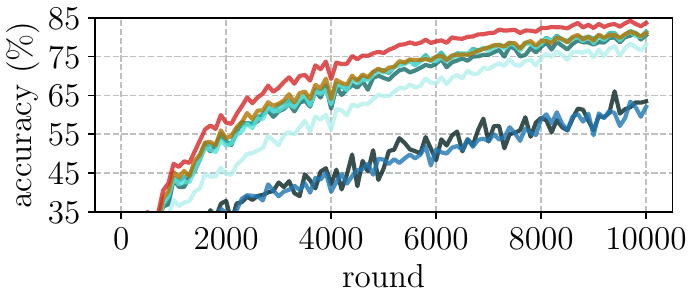}
    \includegraphics[ width=0.19\textwidth]{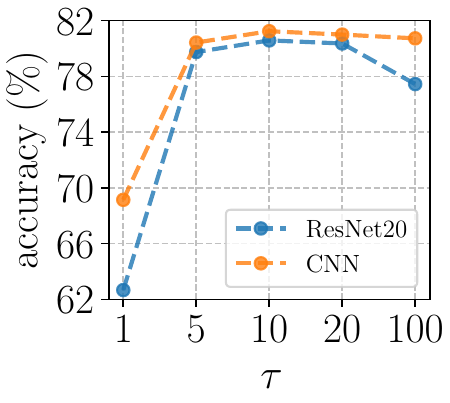}
    \vspace{-0.3cm}
    \caption{\textbf{{\ghb} effectively counteracts the effects of heterogeneity:} our momentum formulation ($\tau>1$) is crucial for superior performance, with an optimal value $\tau=\nicefrac{1}{C}=10$, as predicted in theory. Results on \cifar{10} with {\lenet} (left) and {\resnet} (right), under worst-case heterogeneity. }
    \label{fig:ablation_tau}
\end{figure}
We provide evidence of the effectiveness of {\ghb} under worst-case heterogeneity (\ie $\alpha=0$) by comparing the impact of our generalized heavy-ball momentum formulation to the classical momentum approach, which corresponds to selecting $\tau>1$ in the update rule in \cref{eq:ghb_fl}. As shown in \cref{fig:ablation_tau}, prior momentum-based methods \citep{xu2021fedcm, Ozfatura2021FedADC} fail to improve upon {\fedavg}. In contrast, as $\tau$ increases, {\ghb} exhibits a significant enhancement in both convergence speed and final model quality. The optimal value of $\tau$ is experimentally determined to be $\tau \approx \nicefrac{1}{C}=10$, with larger sub-optimal values only slightly affecting performance (rightmost plot).

This experiment demonstrates that, while complete heterogeneity reduction is theoretically proven only under cyclic participation (\ie \cref{thm:GHBM} holds under \cref{assum:cyclic_part}), {\ghb} empirically achieves strong heterogeneity reduction even with random uniform client sampling. In particular, the theoretical prescription on the optimal value $\tau=\nicefrac{1}{C}$ also holds in this setting.
Moreover, our communication-efficient variants always match or surpass the best-tuned {\ghb}, confirming that their adaptive estimate of each client's momentum positively contributes in a scenario of stochastic client participation (see \cref{sec:theory_ghb}).

\subsection{Comparison with the State-of-art}
\label{exp:comparison_sota}
\paragraph{Results in Controlled Scenario.}
We compare {\ghb} with the most common FL methods, and in particular with other momentum-based FL algorithms, including the recently proposed \textsc{SCAFFOLD-M} \citep{cheng2024momentum}, which which uses both the control variates of {\scaffold} and the momentum of {\fedcm} (and consequently incurs in a communication overhead of $2.5\times$ \wrt{} {\fedavg}).
Our results in \cref{tab:main_controlled} underscore that methods based on classical momentum fail at improving {\fedavg} in scenarios with high heterogeneity and partial participation, confirming that in those cases they should not be expected to provide a significant advantage over heterogeneity.
The general ineffectiveness of classical momentum also holds for \textsc{SCAFFOLD-M}: as it is possible to notice, its performance is not significantly better than {\scaffold}'s, and this well aligns with the theory, where the guarantees against heterogeneity come from the use of control variates, while momentum only brings acceleration. In that our results align with previous findings in literature suggesting that variance reduction, besides theoretically strong, is often not effective empirically in deep learning \citep{defazioineffvr}.
Conversely, our algorithms outperform {\fedavg} with an impressive margin of \textcolor{MaterialGreen800}{$+20.6\%$} and \textcolor{MaterialGreen800}{$+14.4\%$} on {\resnet} and {\lenet} under worst-case heterogeneity, and consistently over less severe conditions (higher values of $\alpha$ in \cref{fig:abl_alpha}).
In particular, as shown in \cref{fig:full_conv_curves_main}, {\ghb} improves over competitor methods also in \textsc{iid} scenarios: this relates to our convergence rate improving not only \wrt heterogeneity, but also displaying a better dependency on the stochastic noise.

\begin{table}[ht]
\vspace{2mm}
    \begin{minipage}{0.57\textwidth}
        \vspace{-6.5cm}
            \captionof{table}{
    \textbf{Comparison with state-of-the-art in controlled setting} (acc@$10k$-$20k$ rounds for {\resnet}/{\lenet}). {\small\textsc{NON-IID}} ($\alpha=0$) and {\small\textsc{IID}} (\edit{$\alpha=10.000$}). Best result in \textbf{bold}, second best \underline{underlined}. {\nan} indicates non-convergence.\vspace{-6pt}}
    \label{tab:main_controlled}
    \centering
    \resizebox{1.03\linewidth}{!}{
    \begin{tabular}{lcccccc}
      \toprule 
      \multirow{2}{*}{\textsc{Method}} &
      \multicolumn{2}{c}{\cifar{100} {\scriptsize(\resnet)}} &
      \multicolumn{2}{c}{\cifar{100} {\scriptsize(\lenet)}} &
      \multicolumn{2}{c}{\textsc{Shakespeare}} \\
    \cmidrule(lr){2-3} \cmidrule(lr){4-5} \cmidrule(lr){6-7} 
       &
      \multicolumn{1}{c}{\small \textsc{NON-IID}} &
      \multicolumn{1}{c}{\small \textsc{IID}} &
      \multicolumn{1}{c}{\small \textsc{NON-IID}} &
      \multicolumn{1}{c}{\small \textsc{IID}} &
      \multicolumn{1}{c}{\small \textsc{NON-IID}} &
      \multicolumn{1}{c}{\small \textsc{IID}} \\
        \midrule
         \fedavg & \edit{$24.7 {\scriptstyle\,\pm 1.2}$} & $58.6 {\scriptstyle\,\pm 0.4}$ & \edit{$38.3 {\scriptstyle\,\pm 0.3}$} & $49.7 {\scriptstyle\,\pm 0.2}$ & $47.3 {\scriptstyle\,\pm 0.1}$ & $47.1 {\scriptstyle\,\pm 0.2}$\\
         \fedprox & \edit{$24.8 {\scriptstyle\,\pm 1.1}$} & $58.5 {\scriptstyle\,\pm 0.3}$ & \edit{$40.6 {\scriptstyle\,\pm 0.2}$} & $49.9 {\scriptstyle\,\pm 0.2}$ & $47.3 {\scriptstyle\,\pm 0.1}$ & $47.1 {\scriptstyle\,\pm 0.2}$\\
         \scaffold & $30.7 {\scriptstyle\,\pm 1.3}$ & $58.0 {\scriptstyle\,\pm 0.6}$ & $45.5 {\scriptstyle\,\pm 0.1}$ & $49.4 {\scriptstyle\,\pm 0.4}$ & $50.2 {\scriptstyle\,\pm 0.1}$ & $50.1 {\scriptstyle\,\pm 0.1}$\\
         \feddyn &  $6.0 {\scriptstyle\,\pm 0.5}$ & $60.8 {\scriptstyle\,\pm 0.7}$ & {\nan} & $\underline{51.9} {\scriptstyle\,\pm 0.2}$ & ${50.7} {\scriptstyle\,\pm 0.2}$ & ${50.8} {\scriptstyle\,\pm 0.2}$\\
         \adabest & $8.4 {\scriptstyle\,\pm 2.0}$ & $55.6 {\scriptstyle\,\pm 0.3}$ & $35.6 {\scriptstyle\,\pm 0.3}$ & $49.7 {\scriptstyle\,\pm 0.2}$ & $47.3 {\scriptstyle\,\pm 0.1}$ & $47.1 {\scriptstyle\,\pm 0.2}$\\
         \mime & \edit{$26.8 {\scriptstyle\,\pm 2.1}$} & $59.0 {\scriptstyle\,\pm 0.3}$ & \edit{$45.3 {\scriptstyle\,\pm 0.4}$} & $50.9 {\scriptstyle\,\pm 0.4}$ & $48.3 {\scriptstyle\,\pm 0.2}$ & $48.5 {\scriptstyle\,\pm 0.1}$\\
        \cmidrule{1-1}
         \fedavgm & \edit{$24.8 {\scriptstyle\,\pm 0.7}$} & $58.7 {\scriptstyle\,\pm 0.9}$ & \edit{$42.1 {\scriptstyle\,\pm 0.3}$} & $50.7 {\scriptstyle\,\pm 0.2}$ & $50.0 {\scriptstyle\,\pm 0.0}$ & $50.4 {\scriptstyle\,\pm 0.1}$\\
        \textsc{FedACG} & $25.7 {\scriptstyle\,\pm 0.5}$ & $58.7 {\scriptstyle\,\pm 0.3}$ & $43.5 {\scriptstyle\,\pm 0.4}$ & $51.3 {\scriptstyle\,\pm 0.3}$ & $50.9 {\scriptstyle\,\pm 0.1}$ & $51.0 {\scriptstyle\,\pm 0.1}$\\
         \textsc{SCAFFOLD-M} & $30.9 {\scriptstyle\,\pm 0.7}$ & $60.1 {\scriptstyle\,\pm 0.5}$ & $45.7 {\scriptstyle\,\pm 0.2}$ & $50.1 {\scriptstyle\,\pm 0.3}$ & $50.8 {\scriptstyle\,\pm 0.0}$ & $51.0 {\scriptstyle\,\pm 0.1}$\\
         {\fedcm} ${\scriptstyle (\ghb\,\tau=1)}$ & $22.2 {\scriptstyle\,\pm 1.0}$ & $53.1 {\scriptstyle\,\pm 0.2}$ & $36.0 {\scriptstyle\,\pm 0.3}$ & $50.2 {\scriptstyle\,\pm 0.5}$ & $49.2 {\scriptstyle\,\pm 0.1}$ & $50.4 {\scriptstyle\,\pm 0.1}$\\
         {\fedadc} ${\scriptstyle (\ghb\,\tau=1)}$ & $22.4 {\scriptstyle\,\pm 0.1}$ & $53.2 {\scriptstyle\,\pm 0.2}$ & $37.9 {\scriptstyle\,\pm 0.3}$ & $50.2 {\scriptstyle\,\pm 0.4}$ & $49.2 {\scriptstyle\,\pm 0.1}$ & $50.4 {\scriptstyle\,\pm 0.1}$\\
         \mimemom & \edit{$24.3 {\scriptstyle\,\pm 0.9}$} & $60.5 {\scriptstyle\,\pm 0.6}$ & ${48.2} {\scriptstyle\,\pm 0.7}$ & $50.6 {\scriptstyle\,\pm 0.1}$ & $48.5 {\scriptstyle\,\pm 0.2}$ & $48.9 {\scriptstyle\,\pm 0.2}$\\
         \mimelite & \edit{$21.2 {\scriptstyle\,\pm 1.6}$} & $59.2 {\scriptstyle\,\pm 0.5}$ & $46.0 {\scriptstyle\,\pm 0.3}$ & $50.7 {\scriptstyle\,\pm 0.1}$ & $49.1 {\scriptstyle\,\pm 0.4}$ & $49.4 {\scriptstyle\,\pm 0.3}$\\
         \cmidrule{1-1}
          \rowcolor{Cerulean!30!white}\textbf{{\localghb} (ours)} & $\underline{38.2} {\scriptstyle\,\pm 1.0}$ & $\underline{62.0} {\scriptstyle\,\pm 0.5}$ & $\underline{50.3} {\scriptstyle\,\pm 0.5}$ & $51.9 {\scriptstyle\,\pm 0.4}$ & $\underline{51.2} {\scriptstyle\,\pm 0.1}$ & $\underline{51.1} {\scriptstyle\,\pm 0.3}$\\
         \rowcolor{Cerulean!30!white}\textbf{{\fedhbm} (ours)} & $\mathbf{42.5} {\scriptstyle\,\pm 0.8}$ & $\mathbf{62.5} {\scriptstyle\,\pm 0.5}$ & $\mathbf{50.4} {\scriptstyle\,\pm 0.5}$ & $\mathbf{52.0} {\scriptstyle\,\pm 0.4}$ & $\mathbf{51.3} {\scriptstyle\,\pm 0.1}$ & $\mathbf{51.4} {\scriptstyle\,\pm 0.2}$ \\
        \bottomrule
    \end{tabular}
    \setlength{\tabcolsep}{1.4pt}
    }

    \end{minipage} \hfill
    \begin{minipage}[b]{0.4\textwidth}
    \centering
        \begin{minipage}[b]{\textwidth}
         \includegraphics[width=0.9\textwidth]{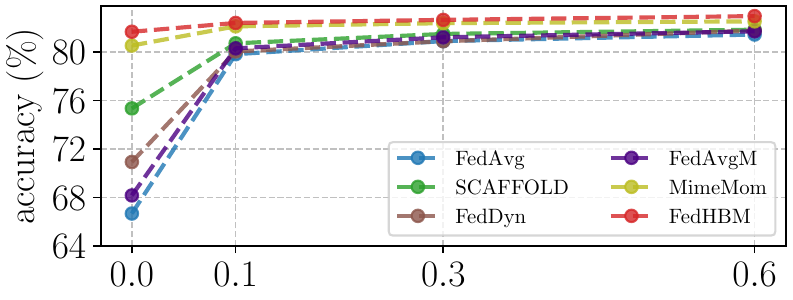}
        \end{minipage}\hfill
        \begin{minipage}[b]{\textwidth}\includegraphics[width=0.9\textwidth]{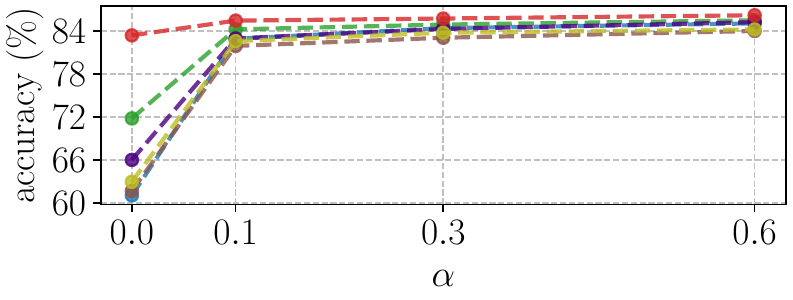}
        \end{minipage}
        \captionof{figure}{\textbf{Final model quality at different values of} $\pmb{\alpha}$ (lower $\alpha \rightarrow $ higher heterogeneity) on \cifar{10}, with {\lenet} (top) and {\resnet} (bottom).\label{fig:abl_alpha}}
    \end{minipage}
\end{table}
\begin{figure}[t]
    \centering
    \includegraphics[width=0.32\linewidth]{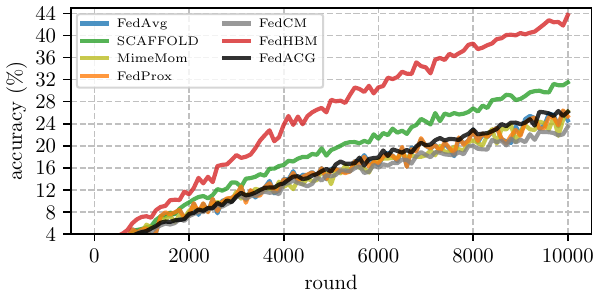}
    \includegraphics[width=0.32\linewidth]{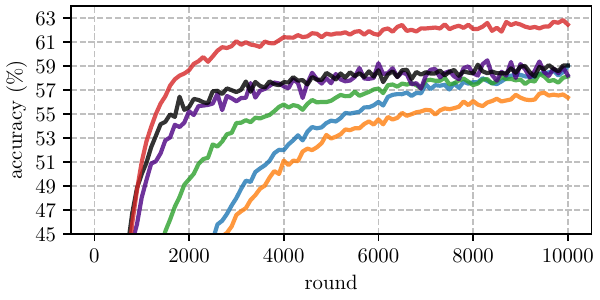}
    \includegraphics[width=0.32\linewidth]{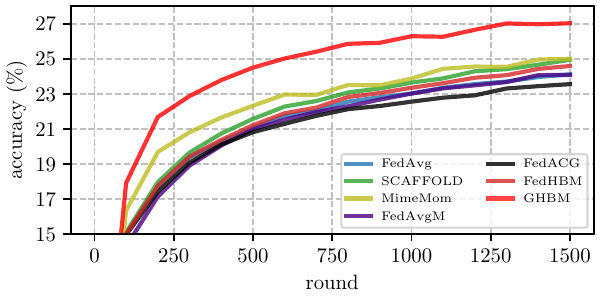}
    \vspace{-2mm}
    \caption{\textbf{{\ghb} largely outperforms state-of-the-art methods:} the plots show the test accuracy (\%) over rounds, with {\resnet} on \cifar{100}, both in \textsc{non-iid} (left) and \textsc{iid} (middle) settings, and on {\stackoverflow} (right). {\ghb} always displays much faster convergence and higher accuracy, even when distributions are \textsc{iid}, confirming robustness \wrt heterogeneity and better dependency on stochastic noise.}
    \label{fig:full_conv_curves_main}
\vspace{-3mm}
\end{figure}

\paragraph{Results in Real-world Large-scale Scenarios.}
\label{sec:experiments:large_scale}
Extending the experimentation to settings characterized by extremely low client participation, we test both our {\ghb} with $\tau$ tuned via a grid-search and our adaptive {\fedhbm}, which exploits client participation to keep the same communication complexity of {\fedavg}.
As discussed in \cref{par:discussion_tau,par:overall_error_ghbm}, under such extreme client participation patterns {\ghb} performs better because the trade-off between heterogeneity reduction and gradient lag is explicitly tuned by the choice of the best performing $\tau$, while {\fedhbm} will likely adopt a suboptimal value. 
However, results in \cref{tab:experiments:large_scale} show a stark improvement over the state-of-art for both our algorithms, indicating that the design principle of our momentum formulation is remarkably robust and provides effective improvement even when client participation is very low (\eg $C\leq 1\%$).
\begin{table}[th]
\caption{\textbf{Test accuracy (\%) comparison of best SOTA FL algorithms on large-scale and realistic settings.} {\ghb} is the best algorithm when client participation is extremely low, while {\fedhbm} still improves the other competitors by a large margin. {\nan} means that the algorithm did not converge.}
\label{tab:experiments:large_scale}
\resizebox{\linewidth}{!}{
\begin{tabular}{lccccccccl}
\toprule
\multirow{3}{*}{\textsc{\large Method}} & \multicolumn{4}{c}{\large\textcolor{NavyBlue}{\mobilenet}} & \multicolumn{3}{c}{\large\textcolor{MaterialGreen600}{\vit}} & \\
  \cmidrule(l){2-5} \cmidrule(l){6-8}
  & {\landmarks}  &  \multicolumn{3}{c}{{\inaturalist} } &
  {\landmarks}  &  \multicolumn{2}{c}{{\inaturalist} } &  {\stackoverflow} &  \multicolumn{1}{c}{} \\ 
  \cmidrule(lr){2-2} \cmidrule(l){3-5} \cmidrule(l){6-6} \cmidrule(lr){7-8} \cmidrule(l){9-9}
 &
  {\small$C\approx0.79\%$} &
  {\small$C\approx0.1\%$} &
  {\small$C\approx0.5\%$} &
  {\small$C\approx1\%$} &
  {\small$C\approx0.79\%$} &
  {\small$C\approx0.1\%$} &
  {\small$C\approx0.5\%$} &
  {\small$C\approx0.12\%$} &
   \\ \midrule
 \cmidrule(lr){1-10}
\fedavg &
  $60.3 {\,\scriptstyle \pm 0.2}$ &
  $38.0 {\,\scriptstyle \pm 0.8}$ &
  $45.25 {\,\scriptstyle \pm 0.1}$ &
  $47.59 {\,\scriptstyle \pm 0.1}$ &
  $68.5 {\,\scriptstyle \pm 0.5}$ &
  $65.6 {\,\scriptstyle \pm 0.1}$ &
  $70.7 {\,\scriptstyle \pm 0.8}$ &
  $24.0 {\,\scriptstyle \pm 0.4}$ &
   \\
\scaffold &
  $61.0 {\,\scriptstyle \pm 0.1}$ &
  \nan &
  \nan &
  \nan &
  $67.5 {\,\scriptstyle \pm 3.3}$ &
  \nan &
  \nan &
  $24.8 {\,\scriptstyle \pm 0.4}$ &
   \\ \cmidrule(l){1-1}
\fedavgm &
  $61.5 {\,\scriptstyle \pm 0.2}$ &
  $41.3 {\,\scriptstyle \pm 0.4} $ &
  $46.0 {\,\scriptstyle \pm 0.1} $ &
  $48.4 {\,\scriptstyle \pm 0.1}$ &
  $70.0 {\,\scriptstyle \pm 0.5}$ &
  $66.0 {\,\scriptstyle \pm 0.2}$ & 
  $71.4 {\,\scriptstyle \pm 0.5}$ &
  $24.1 {\,\scriptstyle \pm 0.3}$ &
   \\
\mimemom &
  \nan &
  \nan &
  \nan &
  \nan &
  \nan &
  \nan &
  \nan &
  $\underline{24.9} {\,\scriptstyle \pm 0.6}$ &
   \\ \cmidrule(l){1-1} 
\rowcolor{Cerulean!30!white} \textbf{{\ghb} - best $\tau$ (ours)} &
  $\mathbf{65.9} {\,\scriptstyle \pm 0.1}$ &
  $\mathbf{41.8} {\,\scriptstyle \pm 0.1}$ &
  $\mathbf{48.7} {\,\scriptstyle \pm 0.1}$ &
  $\mathbf{50.5} {\,\scriptstyle \pm 0.1}$ &
  $\mathbf{74.3} {\,\scriptstyle \pm 0.6}$ &
  $\mathbf{68.8} {\,\scriptstyle \pm 0.3}$ &
  $\mathbf{73.5} {\,\scriptstyle \pm 0.4}$ &
  $\mathbf{27.0} {\,\scriptstyle \pm 0.1}$ &
   \\
\rowcolor{Cerulean!30!white} \textbf{{\fedhbm} (ours)} &
  $\underline{65.4} {\,\scriptstyle \pm 0.2}$ &
  $\underline{41.6} {\,\scriptstyle \pm 0.2}$ &
  $\underline{47.3} {\,\scriptstyle \pm 0.0}$ &
  $\underline{49.8} {\,\scriptstyle \pm 0.0}$ &
  $\underline{73.1} {\,\scriptstyle \pm 0.9}$ &
  $\underline{66.7} {\,\scriptstyle \pm 0.7}$ &
  $\underline{72.1} {\,\scriptstyle \pm 0.5}$ &
  $24.5 {\,\scriptstyle \pm 0.4}$ &
   \\ \bottomrule
\end{tabular}
\setlength{\tabcolsep}{1.4pt}
}
\vspace{-2mm}
\end{table}

\vspace{-2pt}
\paragraph{Communication Efficiency.}
\edit{Results in \cref{tab:comm_comp_cost}} reveal that our proposed algorithms lead to a dramatic reduction in both communication and computational cost, with an average saving of  respectively \textcolor{MaterialGreen800}{$+55.9\%$} and \textcolor{MaterialGreen800}{$+61.5\%$}. 
In practice, \edit{while {\fedhbm} has the same communication complexity of {\fedavgm} and {\ghb} slightly higher}, both our algorithms much show faster convergence and higher final model quality, \edit{which ultimately lead to a significant reduction of the total communication and computational cost}. In particular, in settings with extremely low client participation (\eg{} {\landmarks} and {\inaturalist}), {\ghb} is more suitable for best accuracy, while {\fedhbm} is the best at lowering the communication cost.
\begin{table}[h]
\caption{\textbf{Total communication and computational cost for reaching the final model quality of {\fedavg}}, across academic and real-world large-scale datasets (details in \cref{app:exp:cost}). The coloured arrows indicate respectively a reduction (\textcolor{MaterialGreen800}{$\boldsymbol{\downarrow}$}) and an increase (\textcolor{MaterialRed800}{$\boldsymbol{\uparrow}$}) of communication/computational cost.}
\label{tab:comm_comp_cost}
\centering
\resizebox{\linewidth}{!}{
\begin{tabular}{@{}lcllllllll@{}}
\toprule
\multirow{3}{*}{\large\textsc{Method}} & \multirow{3}{*}{\textsc{\thead{Comm. \\ Overhead}}} &
  \multicolumn{4}{c}{\textsc{\edit{Total} Communication Cost {\footnotesize(bytes exchanged)}}} &
  \multicolumn{4}{c}{\textsc{\edit{Total} Computational Cost {\footnotesize(Wall-Clock Time hh:mm)} }} \\ 
  \cmidrule(lr){3-6} \cmidrule(l){7-10} 
 & &
  \multicolumn{2}{c}{\cifar{100} ($\alpha=0$)} &
  \multicolumn{2}{c}{\landmarks} &
  \multicolumn{2}{c}{\cifar{100} ($\alpha=0$)} &
  \multicolumn{2}{c}{\landmarks} \\ 
  \cmidrule(lr){3-4} \cmidrule(lr){5-6} \cmidrule(lr){7-8} \cmidrule(lr){9-10}
 & &
  {\lenet} &
  {\resnet} &
  {\mobilenet} &
  {\vit} &
  {\lenet} &
  {\resnet} &
  {\mobilenet} &
  {\vit} \\ 
  \midrule
{\fedavg}   & 
    \textcolor{MaterialGreen800}{$\mathbf{1\times}$} 
    & 30.9 GB 
    & 10.3 GB 
    & 89.8 GB 
    & 483.7 GB 
    & 02:05 
    & 03:36 
    & 13:51    
    & 13:56  \\
{\scaffold} & 
    \textcolor{MaterialOrange800}{$\mathbf{2\times}$}
    & 40.8 GB {\scriptsize\textcolor{MaterialRed800}{$\mathbf{\uparrow 32.0\%}$}}
    & 14.2 GB {\scriptsize\textcolor{MaterialRed800}{$\mathbf{\uparrow 37.8\%}$}}
    & 51.2 GB {\scriptsize\textcolor{MaterialGreen800}{$\mathbf{\downarrow 43.0\%}$}} 
    & 967.4 GB {\scriptsize\textcolor{MaterialRed800}{$\mathbf{\uparrow 100.0\%}$}}
    & 01:23 {\scriptsize\textcolor{MaterialGreen800}{$\mathbf{\downarrow 34.0\%}$}}
    & 02:39 {\scriptsize\textcolor{MaterialGreen800}{$\mathbf{\downarrow 26.4\%}$}}
    & 08:28 {\scriptsize\textcolor{MaterialGreen800}{$\mathbf{\downarrow 38.9\%}$}}
    & 15:15 {\scriptsize\textcolor{MaterialRed800}{$\mathbf{\uparrow 9.4\%}$}}   \\ \cmidrule{1-1}
{\fedavgm}  & 
    \textcolor{MaterialGreen800}{$\mathbf{1\times}$} 
    & 21.0 GB {\scriptsize\textcolor{MaterialGreen800}{$\mathbf{\downarrow 32.0\%}$}} 
    & 9.1 GB  {\scriptsize\textcolor{MaterialGreen800}{$\mathbf{\downarrow 11.6\%}$}}
    & 73.6 GB {\scriptsize\textcolor{MaterialGreen800}{$\mathbf{\downarrow 18.0\%}$}}
    & 403.1 GB {\scriptsize\textcolor{MaterialGreen800}{$\mathbf{\downarrow 16.7\%}$}}%
    & 01:25 {\scriptsize\textcolor{MaterialGreen800}{$\mathbf{\downarrow 32.0\%}$}}
    & 03:10 {\scriptsize\textcolor{MaterialGreen800}{$\mathbf{\downarrow 12.0\%}$}}
    & 11:22 {\scriptsize\textcolor{MaterialGreen800}{$\mathbf{\downarrow 18.0\%}$}}
    & 11:37 {\scriptsize\textcolor{MaterialGreen800}{$\mathbf{\downarrow 16.7\%}$}}   \\
{\mimemom}  & 
    \textcolor{MaterialRed800}{$\mathbf{3\times}$}   
    & 21.5 GB {\scriptsize\textcolor{MaterialGreen800}{$\mathbf{\downarrow 30.4\%}$}}
    & 30.9 GB {\scriptsize\textcolor{MaterialRed800}{$\mathbf{\uparrow 200.0\%}$}}
    & 269.4 GB {\scriptsize\textcolor{MaterialRed800}{$\mathbf{\uparrow 200.0\%}$}}
    &  1.417 TB {\scriptsize\textcolor{MaterialRed800}{$\mathbf{\uparrow 200.0\%}$}}
    & 01:27 {\scriptsize\textcolor{MaterialGreen800}{$\mathbf{\downarrow 30.4\%}$}}
    & 10:42 {\scriptsize\textcolor{MaterialRed800}{$\mathbf{\uparrow 197.8\%}$}}
    & 41:07 {\scriptsize\textcolor{MaterialRed800}{$\mathbf{\uparrow 197.8\%}$}}
    & 41:30 {\scriptsize\textcolor{MaterialRed800}{$\mathbf{\uparrow 197.8\%}$}}  \\ \cmidrule{1-1}
\rowcolor{Cerulean!30!white}
\textbf{{\ghb} (ours)}      & 
    \textcolor{MaterialYellow800}{$\mathbf{1.5\times}$}
    & \underline{8.5} GB {\scriptsize\textcolor{MaterialGreen800}{$\mathbf{\downarrow 72.5\%}$}}  
    & \underline{7.0} GB  {\scriptsize\textcolor{MaterialGreen800}{$\mathbf{\downarrow 32.5\%}$}}  
    & \underline{48.5} GB {\scriptsize\textcolor{MaterialGreen800}{$\mathbf{\downarrow 46.0\%}$}}
    & \underline{314.4} GB {\scriptsize\textcolor{MaterialGreen800}{$\mathbf{\downarrow 35.0\%}$}} %
    & \underline{00:24} {\scriptsize\textcolor{MaterialGreen800}{$\mathbf{\downarrow 80.8\%}$}}
    & \underline{01:37} {\scriptsize\textcolor{MaterialGreen800}{$\mathbf{\downarrow 55.0\%}$}}
    & \textbf{05:20} {\scriptsize\textcolor{MaterialGreen800}{$\mathbf{\downarrow 61.5\%}$}}
    &  \textbf{06:30}  {\scriptsize\textcolor{MaterialGreen800}{$\mathbf{\downarrow 53.3\%}$}} \\
\rowcolor{Cerulean!30!white}
\textbf{{\fedhbm} (ours)}   & 
    \textcolor{MaterialGreen800}{$\mathbf{1\times}$} 
    & \textbf{5.2} GB {\scriptsize\textcolor{MaterialGreen800}{$\mathbf{\downarrow 83.0\%}$}}  
    & \textbf{4.2} GB  {\scriptsize\textcolor{MaterialGreen800}{$\mathbf{\downarrow 59.2\%}$}} 
    & \textbf{29.6} GB {\scriptsize\textcolor{MaterialGreen800}{$\mathbf{\downarrow 67.0\%}$}}
    &  \textbf{234.4} GB {\scriptsize\textcolor{MaterialGreen800}{$\mathbf{\downarrow 51.5\%}$}} %
    & \textbf{00:22} {\scriptsize\textcolor{MaterialGreen800}{$\mathbf{\downarrow 82.0\%}$}}
    & \textbf{01:29} {\scriptsize\textcolor{MaterialGreen800}{$\mathbf{\downarrow 59.0\%}$}}
    & \underline{06:23} {\scriptsize\textcolor{MaterialGreen800}{$\mathbf{\downarrow 54.0\%}$}}
    & \underline{07:31} {\scriptsize\textcolor{MaterialGreen800}{$\mathbf{\downarrow 46.0\%}$}}  \\ \bottomrule
\end{tabular}
}
\end{table}

\section{Conclusions}
In this work, we propose \textit{Generalized Heavy-Ball Momentum} (\ghb), a novel momentum-based optimization method for Federated Learning (FL) that effectively mitigates the joint effect of statistical heterogeneity and partial participation. 
We theoretically prove that {\ghb} converges under arbitrary heterogeneity in \textit{cyclic partial participation}, achieving the same rate classical momentum enjoys in \textit{full participation}.
Additionally, we introduce {\fedhbm}, a communication-efficient variant that retains the benefits of momentum while maintaining the same communication complexity as {\fedavg}.
Extensive experiments, conducted under standard random uniform client sampling, confirm that {\ghb} significantly outperforms state-of-the-art FL methods in both convergence speed and final model quality, demonstrating its robustness in large-scale, real-world heterogeneous FL scenarios.
\section*{Acknowledgements}
The authors would like to thank Carlo Ciliberto for fruitful initial discussions on the theoretical aspects of {\ghb} and for his valuable feedback about the presentation of the method.
\section*{Funding}
The author(s) declare that financial support was received for the research, authorship, and/or publication of this article. This study was carried out within the project FAIR - Future Artificial Intelligence Research - and received funding from the European Union Next-GenerationEU [PIANO NAZIONALE DI RIPRESA E RESILIENZA (PNRR) – MISSIONE 4 COMPONENTE 2, INVESTIMENTO 1.3 – D.D. 1555 11/10/2022, PE00000013 - CUP: E13C22001800001]. This manuscript reflects only the authors’ views and opinions, neither the European Union nor the European Commission can be considered responsible for them. A part of the computational resources for this work was provided by hpc@polito, which is a Project of Academic Computing within the Department of Control and Computer Engineering at the Politecnico di Torino (\href{http://www.hpc.polito.it}{\texttt{http://www.hpc.polito.it}}). We acknowledge the CINECA award under the ISCRA initiative for the availability of high-performance computing resources. This work was supported by CINI.

\bibliography{main}

\begin{thebibliography}{57}
\providecommand{\natexlab}[1]{#1}
\providecommand{\url}[1]{\texttt{#1}}
\expandafter\ifx\csname urlstyle\endcsname\relax
  \providecommand{\doi}[1]{doi: #1}\else
  \providecommand{\doi}{doi: \begingroup \urlstyle{rm}\Url}\fi

\bibitem[Acar et~al.(2021)Acar, Zhao, Navarro, Mattina, Whatmough, and
  Saligrama]{acar2021FedDyn}
Durmus Alp~Emre Acar, Yue Zhao, Ramon~Matas Navarro, Matthew Mattina, Paul~N
  Whatmough, and Venkatesh Saligrama.
\newblock Federated learning based on dynamic regularization.
\newblock In \emph{ICLR}, 2021.

\bibitem[Alistarh et~al.(2017)Alistarh, Grubic, Li, Tomioka, and
  Vojnovic]{alistarh2017QSGD}
Dan Alistarh, Demjan Grubic, Jerry Li, Ryota Tomioka, and Milan Vojnovic.
\newblock Qsgd: Communication-efficient sgd via gradient quantization and
  encoding.
\newblock In \emph{NeurIPS}, 2017.

\bibitem[Caldarola et~al.(2021)Caldarola, Mancini, Galasso, Ciccone, Rodola,
  and Caputo]{Caldarola_2021_CVPR}
Debora Caldarola, Massimiliano Mancini, Fabio Galasso, Marco Ciccone, Emanuele
  Rodola, and Barbara Caputo.
\newblock Cluster-driven graph federated learning over multiple domains.
\newblock In \emph{CVPR Workshop}, 2021.

\bibitem[Caldarola et~al.(2022)Caldarola, Caputo, and
  Ciccone]{caldarola2022improving}
Debora Caldarola, Barbara Caputo, and Marco Ciccone.
\newblock Improving generalization in federated learning by seeking flat
  minima.
\newblock In \emph{ECCV}, 2022.

\bibitem[Caldas et~al.(2019)Caldas, Duddu, Wu, Li, Konečný, McMahan, Smith,
  and Talwalkar]{caldas2019leaf}
Sebastian Caldas, Sai Meher~Karthik Duddu, Peter Wu, Tian Li, Jakub Konečný,
  H.~Brendan McMahan, Virginia Smith, and Ameet Talwalkar.
\newblock Leaf: A benchmark for federated settings.
\newblock \emph{arXiv preprint arXiv:1812.01097}, 2019.

\bibitem[Chen et~al.(2021)Chen, Hong, Zha, Zhang, Liu, and
  Han]{chen2020FedSVRG}
Dawei Chen, Choong~Seon Hong, Yiyong Zha, Yunfei Zhang, Xin Liu, and Zhu Han.
\newblock Fedsvrg based communication efficient scheme for federated learning
  in mec networks.
\newblock \emph{IEEE Transactions on Vehicular Technology}, 2021.

\bibitem[Chen et~al.(2022)Chen, Blum, and Sadler]{chen2022OADMM}
Yicheng Chen, Rick~S. Blum, and Brian~M. Sadler.
\newblock Communication efficient federated learning via ordered admm in a
  fully decentralized setting.
\newblock In \emph{CISS}, 2022.

\bibitem[Cheng et~al.(2024)Cheng, Huang, Wu, and Yuan]{cheng2024momentum}
Ziheng Cheng, Xinmeng Huang, Pengfei Wu, and Kun Yuan.
\newblock Momentum benefits non-iid federated learning simply and provably.
\newblock In \emph{ICLR}, 2024.

\bibitem[Cho et~al.(2023)Cho, Sharma, Joshi, Xu, Kale, and
  Zhang]{cho2023cyclic}
Yae~Jee Cho, Pranay Sharma, Gauri Joshi, Zheng Xu, Satyen Kale, and Tong Zhang.
\newblock On the convergence of federated averaging with cyclic client
  participation.
\newblock In \emph{ICML}, 2023.

\bibitem[Das et~al.(2022)Das, Acharya, Hashemi, sujay sanghavi, Dhillon, and
  ufuk topcu]{das2022faster}
Rudrajit Das, Anish Acharya, Abolfazl Hashemi, sujay sanghavi, Inderjit~S
  Dhillon, and ufuk topcu.
\newblock Faster non-convex federated learning via global and local momentum.
\newblock In \emph{The 38th Conference on Uncertainty in Artificial
  Intelligence}, 2022.

\bibitem[Defazio \& Bottou(2019)Defazio and Bottou]{defazioineffvr}
Aaron Defazio and Leon Bottou.
\newblock On the ineffectiveness of variance reduced optimization for deep
  learning.
\newblock In \emph{NeurIPS}, 2019.

\bibitem[Dosovitskiy et~al.(2021)Dosovitskiy, Beyer, Kolesnikov, Weissenborn,
  Zhai, Unterthiner, Dehghani, Minderer, Heigold, Gelly, Uszkoreit, and
  Houlsby]{dosovitskiy2021an}
Alexey Dosovitskiy, Lucas Beyer, Alexander Kolesnikov, Dirk Weissenborn,
  Xiaohua Zhai, Thomas Unterthiner, Mostafa Dehghani, Matthias Minderer, Georg
  Heigold, Sylvain Gelly, Jakob Uszkoreit, and Neil Houlsby.
\newblock An image is worth 16x16 words: Transformers for image recognition at
  scale.
\newblock In \emph{ICLR}, 2021.

\bibitem[Ferrante \& Frigo(2012)Ferrante and Frigo]{coupon1}
Marco Ferrante and Nadia Frigo.
\newblock A note on the coupon-collector's problem with multiple arrivals and
  the random sampling.
\newblock \emph{arXiv preprint arXiv:1209.2667}, 2012.

\bibitem[Ferrante \& Saltalamacchia(2014)Ferrante and Saltalamacchia]{coupon2}
Marco Ferrante and Monica Saltalamacchia.
\newblock The coupon collector's problem.
\newblock \emph{Materials matem{\`a}tics}, 2014.

\bibitem[Gong et~al.(2022)Gong, Li, and Freris]{gong2022fedadmm}
Yonghai Gong, Yichuan Li, and Nikolaos~M. Freris.
\newblock Fedadmm: A robust federated deep learning framework with adaptivity
  to system heterogeneity.
\newblock \emph{arXiv preprint arXiv:2204.03529}, 2022.

\bibitem[Gürbüzbalaban et~al.(2015)Gürbüzbalaban, Ozdaglar, and
  Parrilo]{gur2022IGD}
Mert Gürbüzbalaban, Asu Ozdaglar, and Pablo Parrilo.
\newblock Convergence rate of incremental gradient and newton methods.
\newblock \emph{SIAM Journal on Optimization}, 2015.

\bibitem[He et~al.(2015)He, Zhang, Ren, and Sun]{he2015deep}
Kaiming He, Xiangyu Zhang, Shaoqing Ren, and Jian Sun.
\newblock Deep residual learning for image recognition.
\newblock \emph{arXiv preprint arXv:1512.03385}, 2015.

\bibitem[Hsieh et~al.(2020)Hsieh, Phanishayee, Mutlu, and
  Gibbons]{pmlr-v119-hsieh20a}
Kevin Hsieh, Amar Phanishayee, Onur Mutlu, and Phillip Gibbons.
\newblock The non-{IID} data quagmire of decentralized machine learning.
\newblock In \emph{ICML}, 2020.

\bibitem[Hsu et~al.(2019)Hsu, Qi, and Brown]{hsu2019measuring}
Tzu-Ming~Harry Hsu, Hang Qi, and Matthew Brown.
\newblock Measuring the effects of non-identical data distribution for
  federated visual classification.
\newblock \emph{arXiv preprint arXiv:1909.06335}, 2019.

\bibitem[Hsu et~al.(2020)Hsu, Qi, and Brown]{hsu2020FederatedVisual}
Tzu-Ming~Harry Hsu, Hang Qi, and Matthew Brown.
\newblock Federated visual classification with real-world data distribution.
\newblock In Andrea Vedaldi, Horst Bischof, Thomas Brox, and Jan-Michael Frahm
  (eds.), \emph{ECCV}, 2020.

\bibitem[Idelbayev(2021)]{resnetCifarGithub}
Yerlan Idelbayev.
\newblock Proper {ResNet} implementation for {CIFAR10/CIFAR100} in {PyTorch},
  2021.

\bibitem[Ioffe \& Szegedy(2015)Ioffe and Szegedy]{ioffe2015batch}
Sergey Ioffe and Christian Szegedy.
\newblock Batch normalization: Accelerating deep network training by reducing
  internal covariate shift.
\newblock In \emph{ICML}, 2015.

\bibitem[Kairouz~et al.(2021)]{kairouz2021advances}
Peter Kairouz~et al.
\newblock Advances and open problems in federated learning.
\newblock \emph{Found. Trends Mach. Learn.}, 2021.

\bibitem[Karagulyan et~al.(2024)Karagulyan, Shulgin, Sadiev, and
  Richtárik]{karagulyan2024spam}
Avetik Karagulyan, Egor Shulgin, Abdurakhmon Sadiev, and Peter Richtárik.
\newblock Spam: Stochastic proximal point method with momentum variance
  reduction for non-convex cross-device federated learning.
\newblock \emph{arXiv preprint arXiv:2405.20127}, 2024.

\bibitem[Karimireddy et~al.(2020)Karimireddy, Kale, Mohri, Reddi, Stich, and
  Suresh]{karimireddy2020scaffold}
Sai~Praneeth Karimireddy, Satyen Kale, Mehryar Mohri, Sashank Reddi, Sebastian
  Stich, and Ananda~Theertha Suresh.
\newblock Scaffold: Stochastic controlled averaging for federated learning.
\newblock In \emph{ICML}, 2020.

\bibitem[Karimireddy et~al.(2021)Karimireddy, Jaggi, Kale, Mohri, Reddi, Stich,
  and Suresh]{karimireddy2021Mime}
Sai~Praneeth Karimireddy, Martin Jaggi, Satyen Kale, Mehryar Mohri, Sashank~J.
  Reddi, Sebastian~U Stich, and Ananda~Theertha Suresh.
\newblock Breaking the centralized barrier for cross-device federated learning.
\newblock In \emph{NeurIPS}, 2021.

\bibitem[Kim et~al.(2024)Kim, Kim, and Han]{kim2024communication}
Geeho Kim, Jinkyu Kim, and Bohyung Han.
\newblock Communication-efficient federated learning with accelerated client
  gradient.
\newblock In \emph{CVPR}, 2024.

\bibitem[Koloskova et~al.(2020)Koloskova, Lin, Stich, and
  Jaggi]{koloskova2020Decentralized}
Anastasia Koloskova, Tao Lin, Sebastian~U Stich, and Martin Jaggi.
\newblock Decentralized deep learning with arbitrary communication compression.
\newblock In \emph{ICLR}, 2020.

\bibitem[Kopparapu \& Lin(2020)Kopparapu and Lin]{kopparapu2020fedfmc}
Kavya Kopparapu and Eric Lin.
\newblock Fedfmc: Sequential efficient federated learning on non-iid data.
\newblock \emph{arXiv preprint arXiv:2006.10937}, 2020.

\bibitem[Li \& Wang(2019)Li and Wang]{li2019fedmd}
Daliang Li and Junpu Wang.
\newblock Fedmd: Heterogenous federated learning via model distillation.
\newblock \emph{arXiv preprint arXiv:1910.03581}, 2019.

\bibitem[Li et~al.(2019)Li, Sahu, Zaheer, Sanjabi, Talwalkar, and
  Smith]{li2019FedDANE}
Tian Li, Anit~Kumar Sahu, Manzil Zaheer, Maziar Sanjabi, Ameet Talwalkar, and
  Virginia Smith.
\newblock Feddane: A federated newton-type method.
\newblock \emph{Asilomar Conference on Signals, Systems, and Computers}, 2019.

\bibitem[Li et~al.(2020)Li, Sahu, Zaheer, Sanjabi, Talwalkar, and
  Smith]{li2020FedProx}
Tian Li, Anit~Kumar Sahu, Manzil Zaheer, Maziar Sanjabi, Ameet Talwalkar, and
  Virginia Smith.
\newblock Federated optimization in heterogeneous networks.
\newblock In \emph{MLSys}, 2020.

\bibitem[Lin et~al.(2020)Lin, Stich, Patel, and Jaggi]{Lin2020}
Tao Lin, Sebastian~U. Stich, Kumar~Kshitij Patel, and Martin Jaggi.
\newblock Don't use large mini-batches, use local sgd.
\newblock In \emph{ICLR}, 2020.

\bibitem[Liu et~al.(2020)Liu, Gao, and Yin]{liu2020SGDM}
Yanli Liu, Yuan Gao, and Wotao Yin.
\newblock An improved analysis of stochastic gradient descent with momentum.
\newblock In \emph{NeurIPS}, 2020.

\bibitem[McMahan et~al.(2017)McMahan, Moore, Ramage, Hampson, and
  y~Arcas]{mcmahan2017FedAvg}
Brendan McMahan, Eider Moore, Daniel Ramage, Seth Hampson, and Blaise~Aguera
  y~Arcas.
\newblock Communication-efficient learning of deep networks from decentralized
  data.
\newblock In \emph{Artificial intelligence and statistics}, 2017.

\bibitem[Mishchenko et~al.(2019)Mishchenko, Gorbunov, Takáč, and
  Richtárik]{mishchenko2019distributed}
Konstantin Mishchenko, Eduard Gorbunov, Martin Takáč, and Peter Richtárik.
\newblock Distributed learning with compressed gradient differences.
\newblock \emph{arXiv preprint arXiv:1901.09269}, 2019.

\bibitem[Mishchenko et~al.(2022)Mishchenko, Malinovsky, Stich, and
  Richtarik]{mishchenko22b}
Konstantin Mishchenko, Grigory Malinovsky, Sebastian Stich, and Peter
  Richtarik.
\newblock {P}rox{S}kip: Yes! {L}ocal gradient steps provably lead to
  communication acceleration! {F}inally!
\newblock In \emph{ICML}, 2022.

\bibitem[Mishchenko et~al.(2024)Mishchenko, Li, Fan, and
  Venieris]{mishchenko2024federated}
Konstantin Mishchenko, Rui Li, Hongxiang Fan, and Stylianos Venieris.
\newblock Federated learning under second-order data heterogeneity, 2024.
\newblock URL \url{https://openreview.net/forum?id=jkhVrIllKg}.

\bibitem[Ozfatura et~al.(2021)Ozfatura, Ozfatura, and
  Gündüz]{Ozfatura2021FedADC}
Emre Ozfatura, Kerem Ozfatura, and Deniz Gündüz.
\newblock Fedadc: Accelerated federated learning with drift control.
\newblock In \emph{2021 IEEE International Symposium on Information Theory
  (ISIT)}, 2021.

\bibitem[Patel et~al.(2022)Patel, Wang, Woodworth, Bullins, and
  Srebro]{patel2022celsgd}
Kumar~Kshitij Patel, Lingxiao Wang, Blake~E Woodworth, Brian Bullins, and Nati
  Srebro.
\newblock Towards optimal communication complexity in distributed non-convex
  optimization.
\newblock In \emph{NeurIPS}, 2022.

\bibitem[Polyak(1964)]{polyak1964heavyball}
Boris Polyak.
\newblock Some methods of speeding up the convergence of iteration methods.
\newblock \emph{Ussr Computational Mathematics and Mathematical Physics}, 1964.

\bibitem[Reddi et~al.(2021)Reddi, Charles, Zaheer, Garrett, Rush,
  Kone{\v{c}}n{\`y}, Kumar, and McMahan]{reddi2020FedOpt}
Sashank Reddi, Zachary Charles, Manzil Zaheer, Zachary Garrett, Keith Rush,
  Jakub Kone{\v{c}}n{\`y}, Sanjiv Kumar, and H~Brendan McMahan.
\newblock Adaptive federated optimization.
\newblock In \emph{ICLR}, 2021.

\bibitem[Reisizadeh et~al.(2020)Reisizadeh, Mokhtari, Hassani, Jadbabaie, and
  Pedarsani]{reisizadeh2020FedPAQ}
Amirhossein Reisizadeh, Aryan Mokhtari, Hamed Hassani, Ali Jadbabaie, and
  Ramtin Pedarsani.
\newblock Fedpaq: A communication-efficient federated learning method with
  periodic averaging and quantization.
\newblock In \emph{AISTATS}, 2020.

\bibitem[Sandler et~al.(2018)Sandler, Howard, Zhu, Zhmoginov, and
  Chen]{sandler2018mobilenetv2}
Mark Sandler, Andrew~G. Howard, Menglong Zhu, Andrey Zhmoginov, and
  Liang{-}Chieh Chen.
\newblock Mobilenetv2: Inverted residuals and linear bottlenecks.
\newblock In \emph{CVPR}, 2018.

\bibitem[Sattler et~al.(2020)Sattler, Wiedemann, Müller, and
  Samek]{sattler2020STC}
Felix Sattler, Simon Wiedemann, Klaus-Robert Müller, and Wojciech Samek.
\newblock Robust and communication-efficient federated learning from non-i.i.d.
  data.
\newblock \emph{IEEE Transactions on Neural Networks and Learning Systems},
  2020.

\bibitem[Stadje(1990)]{coupon0}
Wolfgang Stadje.
\newblock The collector's problem with group drawings.
\newblock \emph{Advances in Applied Probability}, 1990.

\bibitem[Steiner et~al.(2022)Steiner, Kolesnikov, Zhai, Wightman, Uszkoreit,
  and Beyer]{steiner2022how}
Andreas~Peter Steiner, Alexander Kolesnikov, Xiaohua Zhai, Ross Wightman, Jakob
  Uszkoreit, and Lucas Beyer.
\newblock How to train your vit? data, augmentation, and regularization in
  vision transformers.
\newblock \emph{TMLR}, 2022.

\bibitem[Stich(2019)]{stich2018local}
Sebastian~U. Stich.
\newblock Local {SGD} converges fast and communicates little.
\newblock In \emph{ICML}, 2019.

\bibitem[Varno et~al.(2022)Varno, Saghayi, Rafiee~Sevyeri, Gupta, Matwin, and
  Havaei]{varno2022AdaBest}
Farshid Varno, Marzie Saghayi, Laya Rafiee~Sevyeri, Sharut Gupta, Stan Matwin,
  and Mohammad Havaei.
\newblock Adabest: Minimizing client drift in federated learning via adaptive
  bias estimation.
\newblock In \emph{ECCV}, 2022.

\bibitem[Wang et~al.(2022)Wang, Marella, and Anderson]{wang2022fedadmm}
Han Wang, Siddartha Marella, and James Anderson.
\newblock Fedadmm: A federated primal-dual algorithm allowing partial
  participation.
\newblock \emph{arXiv preprint arXiv:2203.15104}, 2022.

\bibitem[Wu \& He(2018)Wu and He]{Wu2018GroupNorm}
Yuxin Wu and Kaiming He.
\newblock Group normalization.
\newblock In \emph{ECCV}, 2018.

\bibitem[Xie et~al.(2019)Xie, Koyejo, Gupta, and Lin]{xie2019AdaAlter}
Cong Xie, Oluwasanmi Koyejo, Indranil Gupta, and Haibin Lin.
\newblock Local adaalter: Communication-efficient stochastic gradient descent
  with adaptive learning rates.
\newblock \emph{CoRR}, 2019.

\bibitem[Xu et~al.(2021)Xu, Wang, Wang, and Yao]{xu2021fedcm}
Jing Xu, Sen Wang, Liwei Wang, and Andrew Chi-Chih Yao.
\newblock Fedcm: Federated learning with client-level momentum.
\newblock \emph{arXiv preprint arXiv:2106.10874}, 2021.

\bibitem[Yang et~al.(2021)Yang, Fang, and Liu]{yang2021achieving}
Haibo Yang, Minghong Fang, and Jia Liu.
\newblock Achieving linear speedup with partial worker participation in
  non-{IID} federated learning.
\newblock In \emph{ICLR}, 2021.

\bibitem[Zaccone et~al.(2022)Zaccone, Rizzardi, Caldarola, Ciccone, and
  Caputo]{zaccone2022FedSeq}
Riccardo Zaccone, Andrea Rizzardi, Debora Caldarola, Marco Ciccone, and Barbara
  Caputo.
\newblock Speeding up heterogeneous federated learning with sequentially
  trained superclients.
\newblock In \emph{ICPR}, 2022.

\bibitem[Zeng et~al.(2022)Zeng, Li, Yu, He, Xu, Niyato, and Yu]{zeng2022fedGSP}
Shenglai Zeng, Zonghang Li, Hongfang Yu, Yihong He, Zenglin Xu, Dusit Niyato,
  and Han Yu.
\newblock Heterogeneous federated learning via grouped sequential-to-parallel
  training.
\newblock In Arnab Bhattacharya, Janice Lee Mong~Li, Divyakant Agrawal,
  P.~Krishna Reddy, Mukesh Mohania, Anirban Mondal, Vikram Goyal, and Rage
  Uday~Kiran (eds.), \emph{Database Systems for Advanced Applications}, 2022.

\bibitem[Zhao et~al.(2018)Zhao, Li, Lai, Suda, Civin, and
  Chandra]{zhao2018federated}
Yue Zhao, Meng Li, Liangzhen Lai, Naveen Suda, Damon Civin, and Vikas Chandra.
\newblock Federated learning with non-iid data.
\newblock \emph{arXiv preprint arXiv:1806.00582}, 2018.

\end{thebibliography}
\bibliographystyle{tmlr}

\newpage
\appendix
\onecolumn

\section{Additional Discussion}

\subsection{Extended Related Works}
\label{app:discussion}
Recently, similarly based on variance reduction as SCAFFOLD, \citep{mishchenko22b} propose \textsc{ScaffNew} to achieve accelerated communication complexity in heterogeneous settings through control variates, guaranteeing convergence under arbitrary heterogeneity in full participation. 
The work by \cite{mishchenko2024federated}, under the assumption of second-order data heterogeneity, proposes an algorithm which can reduce client drift by estimating the global update direction as well as employing regularization.
\edit{%
The proposed algorithm can be seen as a combination of {\fedprox} with {\scaffold}/\textsc{ScaffNew}, and similarly relies on additional server control variates to correct the drift, so the underlying principle is still variance reduction. Quite differently, GHBM is based on momentum, properly modified to tackle heterogeneity and partial participation in FL.}
Similarly to the already discussed {\mime} \citep{karimireddy2021Mime}, \cite{karagulyan2024spam} propose the \textsc{SPAM} algorithm and leverage momentum as a local correction term to benefit from second-order similarity.

\paragraph{Comparison with FedACG \citep{kim2024communication}.}

\edit{
We provide a comparison with the FedACG algorithm based on: algorithmic design, theoretical guarantees and empirical results.
Algorithmically, it has two modifications \wrt{} {\fedavgm}: (i) it uses the Nesterov Accelerated Gradient (NAG) to broadcast a lookahead global model and (ii) adds a proximal local penalty similar to {\fedprox} \wrt{} this transmitted global model. The method has the same communication complexity as FedAvg, because it does not exchange additional information.
Our work proposes instead a novel formulation of momentum, explicitly designed to provide an advantage in heterogeneous FL with partial client participation. We propose both the main algorithm ({\ghb}), which has \textit{stateless} clients but has $1.5\times$ the communication complexity of FedAvg, and communication efficient versions (\eg{} {\fedhbm}), that preserve the communication complexity as FedAvg, at the cost of using local storage.
From a theoretical perspective, the convergence rate of FedACG does not prove any advantage \wrt{} heterogeneity, since it still relies on the bounded heterogeneity assumption. {\ghb} is proven to converge under arbitrary heterogeneity in cyclic partial participation, recovering the same convergence rate that \cite{cheng2024momentum} proved for {\fedcm} when in full participation. This is a significant advantage that then reflects in significantly improved performance.
From an empirical perspective, simulation results are presented in \cref{fig:full_conv_curves_main}. While it is faster than FedAvgM, it still falls short behind our algorithms in heterogeneous scenarios. This is a consequence of the same issue we showed in \cref{sec:method_ghb} for classical momentum.
}

\begin{figure}[hb]
    \centering
    \includegraphics[width=0.9\linewidth]{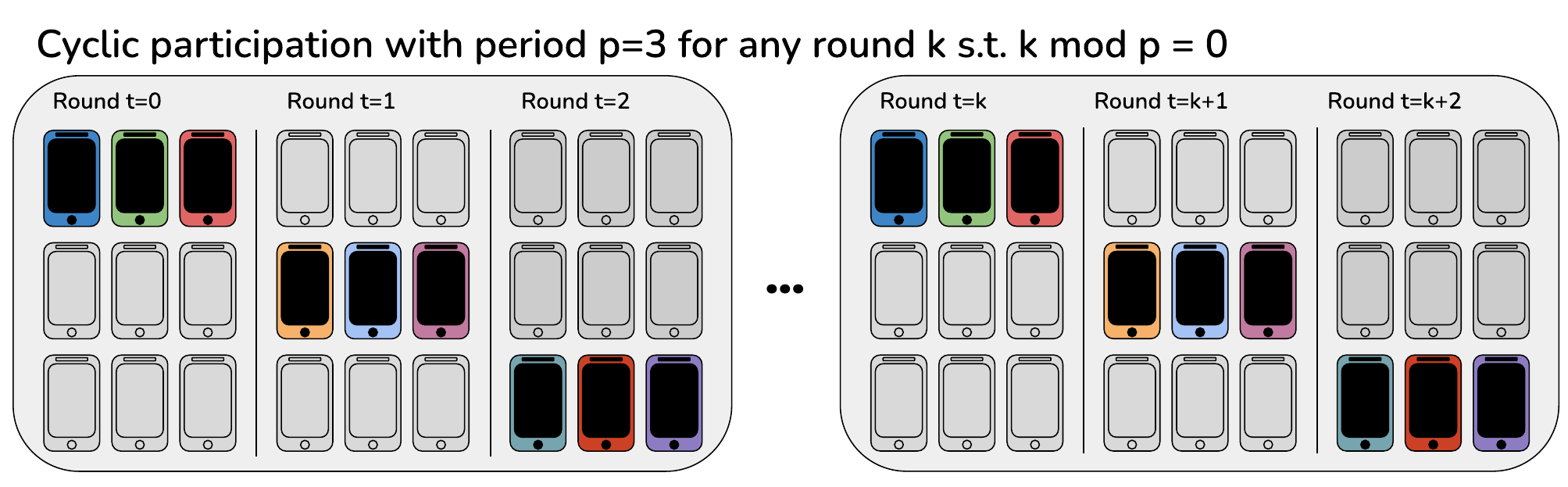}
    \caption{\edit{\textbf{Illustration of cyclic client participation with a total of $K=9$ clients.} 
    \cref{thm:GHBM} holds under the assumption of cyclic participation, which simply states that there is \edit{any} fixed order \edit{(so client shuffling methods like Shuffle-Once are compliant with the assumption)} in which clients appear across rounds in the training, \ie each client is sampled every $p=\frac{1}{C}$ rounds.
    In the above image, $K\cdot C=3$ clients are selected for training, \ie each client is selected exactly once every $p=3$ rounds.}}
    \label{fig:cyclic_illustration}
\end{figure}

\subsection{Notes on Failure Cases of SOTA Algorithms}
\label{app:exp:failure_cases}
In this paper, we evaluated our approach using the large-scale FL datasets proposed by \citep{hsu2020FederatedVisual}. Notably, several recent state-of-the-art FL algorithms failed to converge on these datasets. For {\scaffold} this result aligns with prior works \citep{reddi2020FedOpt, karimireddy2021Mime}, since it is unsuitable for cross-device FL with thousands of devices. Indeed, the client control variates can become stale, and may consequently degrade the performance. For {\mimemom} \citep{karimireddy2021Mime}, despite extensive hyperparameter tuning using the authors' original code, we were unable to achieve convergence. This finding is surprising since the approach has been proposed to tackle cross-device FL. To our knowledge, this is the first work to report these failure cases, likely due to the lack of prior evaluations on such challenging datasets. We believe these findings underscore the need for further investigation into the factors contributing to algorithm performance in large-scale, heterogeneous FL settings.

\section{Proofs}
\label{appendix:theory}

\subsubsection*{Algorithms}
To handle the proof, we analyze a simpler version of our algorithm, in which we use the update rule in \cref{eq:proto_mom_update_real} instead of the one described in  \cref{eq:ghb_exp}.
The resulting \cref{algo:ghbm_theory_app} we analyze is reported along the plain {\ghb} (\cref{algo:ghbm_practical_app}) we used in the experiments. Both algorithms enjoy the same underlying idea: use the gradients of a larger portion of the clients to estimate the momentum term.

\begin{algorithm}[H]
\algsetup{linenosize=\tiny}
\footnotesize
\caption{\textsc{GHBM (practical version)}}
\label{algo:ghbm_practical_app}

\begin{algorithmic}[1]
\REQUIRE initial model $\theta^0$, $K$ clients, $C$ participation ratio, $T$ number of total round, $\eta$ and $\eta_l$ learning rates, $\tau \in \mathbb{N}^+$.
\FOR{$t=1$ to $T$} 
    \STATE $\mathcal{S}^t \leftarrow $ subset of clients  $\sim\mathcal{U}(\mathcal{S}, \max(1, K\cdot C))$
    \FOR{$i \in \mathcal{S}^t $ \textbf{in parallel}}
        \STATE $\theta_{i}^{t,0} \leftarrow \theta^{t-1}$
        \FOR{$j=1$ to $J$}
            \STATE sample a mini-batch $d_{i,j}$ from $\mathcal{D}_i$
            \STATE $u_i^{t,j} \leftarrow \nabla f_i(\theta_i^{t, j-1}, d_{i,j}) + \beta \tilde{m}^t_{\tau}$
            \STATE $\theta_i^{t,j} \leftarrow \theta_i^{t,j-1} - \eta_l u_i^{t,j}$
        \ENDFOR
    \ENDFOR
    \STATE $u^t \leftarrow \frac{1}{|\S{t}|} \sum_{i \in \S{t}} \left(\theta^{t-1} - \theta_i^{t, J} \right)$
    \STATE $\theta^t \leftarrow \theta^{t-1} -\eta u^t$
    \STATE $\tilde{m}^{t+1}_\tau \leftarrow \frac{1}{\tau J}\left(\theta^{t-\tau} - \theta^t\right)$
\ENDFOR
\end{algorithmic}
\end{algorithm}

\begin{algorithm}[H]
\algsetup{linenosize=\tiny}
\footnotesize
\caption{\textsc{GHBM (theory version)}}
\label{algo:ghbm_theory_app}

\begin{algorithmic}[1]
\REQUIRE initial model $\theta^0$, $K$ clients, $C$ participation ratio, $T$ number of total round, $\eta$ and $\eta_l$ learning rates, $\tau \in \mathbb{N}^+$. 
\FOR{$t=1$ to $T$} 
    \STATE $\mathcal{S}^t \leftarrow $ subset of clients  $\sim\mathcal{U}(\mathcal{S}, \max(1, K\cdot C))$
    \FOR{$i \in \mathcal{S}^t $ \textbf{in parallel}}
        \STATE $\theta_{i}^{t,0} \leftarrow \theta^{t-1}$
        \FOR{$j=1$ to $J$}
            \STATE sample a mini-batch $d_{i,j}$ from $\mathcal{D}_i$
            \STATE $u_i^{t,j} \leftarrow \beta \nabla f_i(\theta_i^{t, j-1}, d_{i,j}) + (1-\beta) \tilde{m}^t_{\tau}$
            \STATE $\theta_i^{t,j} \leftarrow \theta_i^{t,j-1} - \eta_l u_i^{t,j}$
        \ENDFOR
    \ENDFOR
    \STATE $u^t \leftarrow \frac{1}{\eta_l|\S{t}|J} \sum_{i \in \S{t}} \left(\theta^{t-1} - \theta_i^{t, J} \right)$
    \STATE $\bar{\theta}^t \leftarrow \theta^{t-1} - u^t + (1-\beta)\tilde{m}^t_\tau$
    \STATE $\tilde{m}^{t+1}_\tau \leftarrow (1-\beta)\tilde{m}^t_\tau + \frac{1}{\tau}\left(\bar{\theta}^{t-\tau} - \bar{\theta}^t\right)$
    \STATE $\theta^t \leftarrow \theta^{t-1} - \eta \tilde{m}^{t+1}_\tau$
\ENDFOR
\end{algorithmic}
\end{algorithm}

In the following, we list the differences between the two:
\begin{enumerate}
    \item Explicit use of $\tau$-averaged gradients when updating the momentum term (line 13). This can be implemented by keeping server-side an auxiliary sequence of models $\bar{\theta}^t$, in which the momentum added client side is subtracted server-side (line 12), such that taking the difference of two models gives the sum of pseudo-grads.
    \item Use of convex sum in local updates (line 7). This is done to align with the formulation of momentum methods in \cite{cheng2024momentum}, and more in general with the formulation of momentum commonly analyzed in literature. There is no theoretical difference between the two versions, as they only differ by a constant scaling \citep{liu2020SGDM}.
    \item Use of gradients averaged over local steps (line 11). This is done to align with the analysis of \cite{cheng2024momentum, xu2021fedcm}, and it is equivalent to coupling server and client learning rates (\ie setting $\eta=\gamma J\eta_l$ in \cref{algo:ghbm_theory_app}, where $\gamma$ is the server learning rate we would use in \cref{algo:ghbm_practical_app}).
\end{enumerate}
The two algorithms have similar performances, which are reported in \cref{fig:compare_ghbm_theory}

\begin{figure}[!h]
    \centering
    \includegraphics[width=0.45\linewidth]{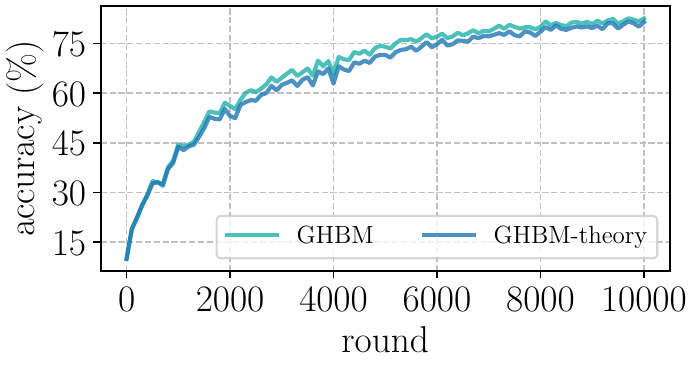}
    \includegraphics[width=0.45\linewidth]{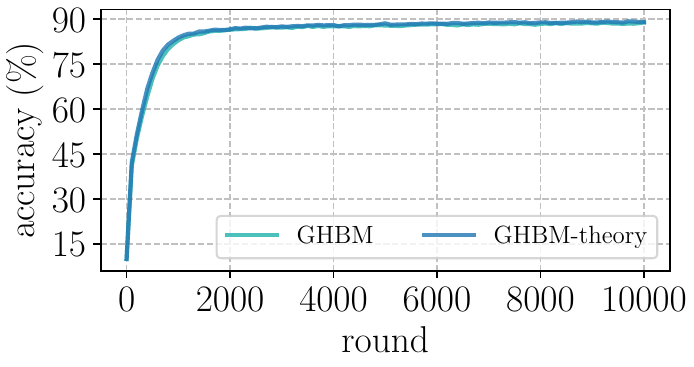}
    \includegraphics[width=0.45\linewidth]{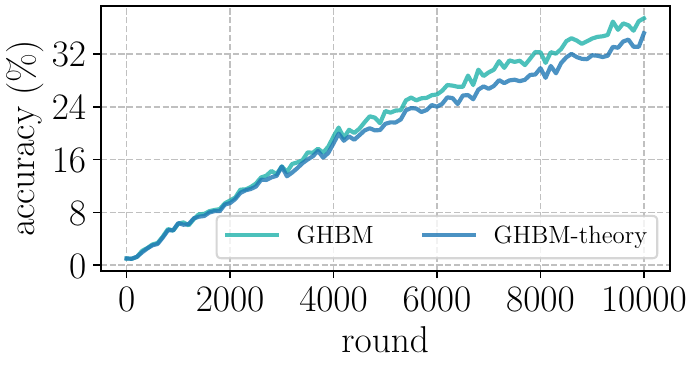}
    \includegraphics[width=0.45\linewidth]{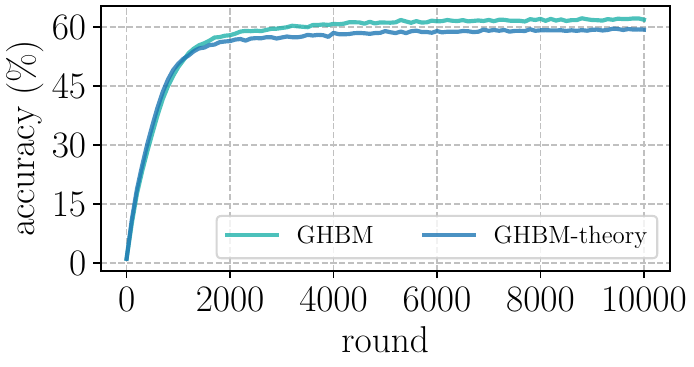}
    \caption{Comparing the {\ghb} implementation analyzed in theory (\cref{algo:ghbm_theory_app}) with the one proposed in the main paper (\cref{algo:ghbm_practical_app}). The plots show the convergence rate on \cifar{10} (top) and \cifar{100} (bottom), in \textsc{non-iid} (left) and \textsc{iid} (right) scenarios with {\resnet} architecture.}
    \label{fig:compare_ghbm_theory}
\end{figure}

\subsection*{Preliminaries}
Our convergence proof for {\ghb} is based on the recent work of \cite{cheng2024momentum}, which offers new proof techniques for momentum-based FL algorithms.
Throughout the proofs we use the following auxiliary variables to facilitate the presentation:

\begin{align}
    \drift{t} &:= \frac{1}{|\S{}|J} \sum_{j=1}^{J} \sum_{i=1}^{|\S{}|} \expect{\normsq{\theta_{i}^{t,j} - \theta^{t-1}}} \label{def:cdrift} \\
    \svrvar{t} &:= \expect{\normsq{\nabla f(\theta^{t-1}) - \tilde{m}^{t+1}_\tau}} \label{eq:svrvar} \\
    \zeta_{i}^{t,j} &:= \expect{\theta_i^{t, j+1} - \theta_{i}^{t, j}} \label{def:fs_drift} \\
    \fsdrift{t} &:= \frac{1}{|\S{}|}\sum_{i=1}^{|\S{}|} \expect{\normsq{\zeta_i^{t, 0}}} \nonumber \\
    \Lambda_t &:=\expect{\normsq{\left(\frac{1}{\tau} \sum_{k=t-\tau+1}^{t} \frac{1}{|\S{k}|J} \sum_{i=1}^{|\S{k}|} \sum_{j=1}^{J} \tilde{g}_i^{k,j}({\theta_i^{k,j-1}}) \right) - g^{t_\tau} }} \label{def:delayed_mgrad} \\
    \gamma_t &:=\expect{\normsq{g^{t_{\tau}} - \nabla f(\theta^{t-1})}}
\end{align}

Additionally, here we report the \textit{bounded gradient heterogeneity} assumption. It is used to quantify the heterogeneity reduction effect of {\ghb} varying its $\tau$ hyperparameter. Notice that our main claim does not depend on this assumption, as for the optimal value of $\tau=\nicefrac{1}{C}$ the assumption is not needed (see \cref{lemma:tau_ghb}).

\subsection{Momentum Expressions}
In this section we report the derivation of the momentum expressions in \cref{eq:mom_hb,eq:ghb} from the main paper.

\begin{lemma}[Heavy-Ball Formulation of Classical Momentum]
\label{lemma:expr_hb}
Let us consider the following classical formulation of momentum:
\begin{align}
    \tilde{m}^t &= \beta \tilde{m}^{t-1} + \tilde{g}^t(\theta^{t-1}) \label{eq:mom_exp_m} \\
    \theta^{t} &= \theta^{t-1} - \eta \tilde{m}^t \label{eq:mom_exp_theta}
\end{align}
The same update rule can be equivalently expressed with the following, known as \textit{heavy-ball} formulation:
\begin{equation}
\label{eq:mom_hb_proof}
    \theta^{t} = \theta^{t-1} + \beta(\theta^{t-1}-\theta^{t-2})  - \eta \tilde{g}(\theta^{t-1})
\end{equation}
\end{lemma}
\begin{proof}
    First derive the expression of $\tilde{m}^t$ from \cref{eq:mom_exp_theta}, both for time $t$ and $t-1$:
    \begin{align*}
        \tilde{m}^t &= \frac{\left(\theta^{t-1}-\theta^t\right)}{\eta} \\
        \tilde{m}^{t-1} &= \frac{\left(\theta^{t-2}-\theta^{t-1}\right)}{\eta} \\
    \end{align*}
    Now plug these expressions into \cref{eq:mom_exp_m} to obtain (\ref{eq:mom_hb_proof}):
    \begin{align*}
        \frac{\left(\theta^{t-1}-\theta^t\right)}{\eta} &= \beta \frac{\left(\theta^{t-2}-\theta^{t-1}\right)}{\eta} + \tilde{g}^t(\theta^{t-1}) \\
        \left(\theta^t - \theta^{t-1}\right) &= \beta \left(\theta^{t-1}-\theta^{t-2}\right) - \eta\tilde{g}^t(\theta^{t-1}) \\
        \theta^t &= \theta^{t-1}  + \beta \left(\theta^{t-1}-\theta^{t-2} \right) - \eta\tilde{g}^t(\theta^{t-1})
    \end{align*}
\end{proof}

\begin{lemma}[Heavy-Ball formulation of generalized momentum]
\label{lemma:expr_ghb}
    Let us consider the following generalized formulation of momentum:
    \begin{align}
        \tilde{m}_\tau^t &= \frac{1}{\tau}\sum_{k=1}^{\tau} \beta \tilde{m}_\tau^{t-k} + \tilde{g}^t(\theta^{t-1}) \label{eq:ghb_exp_m}\\
        \theta^t &= \theta^{t-1} - \eta \tilde{m}_\tau^{t} \label{eq:ghb_exp_theta}
    \end{align}
    The same update rule can be equivalently expressed in an heavy ball form, which we call as \textit{Generalized Heavy-Ball} momentum ({\ghb}):
    \begin{equation}
    \label{eq:ghb_proof}
        \theta^{t} = \theta^{t-1} + \frac{\beta}{\tau}(\theta^{t-1}-\theta^{t-\tau-1}) - \eta \tilde{g}(\theta^{t-1})
    \end{equation}
\end{lemma}
\begin{proof}
    First derive the expression of $\tilde{m}_\tau^t$ from \cref{eq:ghb_exp_theta}, both for time $t$ and $t-1$:
    \begin{align*}
        \tilde{m}_\tau^t &= \frac{\left(\theta^{t-1}-\theta^t\right)}{\eta} \\
        \tilde{m}_\tau^{t-1} &= \frac{\left(\theta^{t-2}-\theta^{t-1}\right)}{\eta} \\
    \end{align*}
    Now plug these expressions into \cref{eq:ghb_exp_m}:
    \begin{align*}
        \frac{\left(\theta^{t-1}-\theta^t\right)}{\eta} &= \frac{\beta}{\tau}\sum_{k=1}^{\tau} \frac{\left( \theta^{t-k-1} - \theta^{t-k}\right)}{\eta} + \tilde{g}^t(\theta^{t-1}) \\
        \left(\theta^{t}-\theta^{t-1}\right) &= \frac{\beta}{\tau}\sum_{k=1}^{\tau} \left( \theta^{t-k} - \theta^{t-k-1}\right) - \eta\tilde{g}^t(\theta^{t-1}) \\
        \theta^{t} &= \theta^{t-1} + \frac{\beta}{\tau}\sum_{k=1}^{\tau} \left( \theta^{t-k} - \theta^{t-k-1}\right) - \eta\tilde{g}^t(\theta^{t-1}) \\
        \theta^{t} &= \theta^{t-1}  + \frac{\beta}{\tau}(\theta^{t-1}-\theta^{t-\tau-1}) - \eta \tilde{g}^t(\theta^{t-1})
    \end{align*}
    Where the last equality (\ref{eq:ghb_proof}) comes from telescoping the summation on the rhs.
\end{proof}

\subsection{Technical Lemmas}
Now we cover some technical lemmas which are useful for computations later on. These are known results that are reported here for the convenience of the reader.

\begin{lemma}[relaxed triangle inequality] 
\label{lemma:relaxed_triangle}
Let $\{\boldsymbol{v}_1, \ldots, \boldsymbol{v}_n\}$ be $n$ vectors in $\mathbb{R}^d$. Then, the following is true:
\begin{equation*}
    \left\|  \sum_{i=1}^n \boldsymbol{v}_i  \right\|^2
    \leq n \sum_{i=1}^n \left\|\boldsymbol{v}_i\right\|^2
\end{equation*}
\end{lemma}

\begin{proof}
    By Jensen's inequality, given a convex function $\phi$, a series of $n$ vectors $\{\boldsymbol{v}_1, \ldots, \boldsymbol{v}_n\}$ and a series of non-negative coefficients $\lambda_i$ with $\sum_{i=1}^n \lambda_i =1$, it results that 
    \begin{equation*}
        \phi\left( \sum_{i=1}^n \lambda_i \boldsymbol{v}_i\right) 
        \leq
        \sum_{i=1}^n \lambda_i \phi\left( \boldsymbol{v}_i\right) 
    \end{equation*}
    Since the function $\boldsymbol{v}\rightarrow \|\boldsymbol{v}\|^2$ is convex, we can use this inequality with coefficients $\lambda_1 = \ldots = \lambda_n = 1/n$, with $\sum_{i=1}^n \lambda_i =1$, and obtain that 
    \begin{equation*}
    \left\| \frac{1}{n}\sum_{i=1}^n \boldsymbol{v}_i\right\|^2 
    =
    \frac{1}{n^2}  \left\| \sum_{i=1}^n \boldsymbol{v}_i\right\|^2 
    \leq 
    \frac{1}{n}\sum_{i=1}^n \|\boldsymbol{v}_i\|^2
    \end{equation*}
\end{proof}

\subsection{Proofs of Main Lemmas}
In this section we provide the proofs of the main theoretical results presented in the main paper.

\paragraph{Proof of Lemma~\ref{lemma:tau_ghb}}(Deviation of $\tau$-averaged gradient from true gradient)\\

Let define $\S{}_d:= \S{} - \S{t}_\tau$ and $\S{}_i:= \S{} \cap \S{t}_\tau$.
Let us note that when all clients participate, \textit{i.e.} $\S{}_d = \emptyset$, the claim is trivially true. For $\S{}_d \ne \emptyset$, we can expand the terms at the left-hand side using their definitions as follows:

 \begin{align}
    \gamma_t &= \expect{\normsq{\frac{1}{|\S{t}_{\tau}|} \sum_{i=1}^{|\S{t}_{\tau}|} g_i^t - \frac{1}{|\S{}|}\sum_{i=1}^{|\S{}|} g_i^t}} \\
    &=\expect{\normsq{ \sum_{i \in \S{}_i}  \left(\frac{1}{|\S{t}_{\tau}|} - \frac{1}{|\S{}|} \right)g_i^t -   \sum_{k \in \S{}_d}\frac{1}{|\S{}|} g_k^t }} \\
    \overset{\text{lemma \ref{lemma:relaxed_triangle}}}&{\leq} 2\left(
    \underbrace{\expect{\normsq{ \sum_{i \in \S{}_i}  \left(\frac{1}{|\S{t}_{\tau}|}  - \frac{1}{|\S{}|} \right)g_i^t }}}_{\mathcal{T}_3} +   
    \underbrace{\expect{\normsq{\sum_{k \in \S{}_d}\frac{1}{|\S{}|} g_k^t }}}_{\mathcal{T}_4} \right)  \label{eq:proof_T3T4}
 \end{align}

 Let us consider first $\mathcal{T}_3$. We have:

 \begin{align}
    \mathcal{T}_3 &= \expect{\normsq{ \sum_{i \in \S{}_i}  \left(\frac{1}{|\S{t}_{\tau}|}  - \frac{1}{|\S{}|} \right)g_i^t }} 
    = \expect{\left(\frac{1}{|\S{t}_{\tau}|}  - \frac{1}{|\S{}|} \right)^2 \normsq{ \sum_{i \in \S{}_i}  g_i^t }} \\
    \overset{\text{lemma~\ref{lemma:relaxed_triangle}}}&{\leq}
    \expect{ \left(\frac{1}{|\S{t}_{\tau}|}  - \frac{1}{|\S{}|} \right)^2 |\S{}_i|  \sum_{i \in \S{}_i}  \normsq{g_i^t }} \\
    &= \expect{ \left(\frac{1}{|\S{t}_{\tau}|}  - \frac{1}{|\S{}|} \right)^2 |\S{}_i|  \sum_{i \in \S{}_i}  \normsq{g_i^t - \nabla f(\theta^{t-1}) + \nabla f(\theta^{t-1})}}\\
    \overset{\text{lemma \ref{lemma:relaxed_triangle}}}&{\leq} 2\expect{ \left(\frac{1}{|\S{t}_{\tau}|}  - \frac{1}{|\S{}|} \right)^2 |\S{}_i|  \sum_{i \in \S{}_i} \left( \normsq{g_i^t - \nabla f(\theta^{t-1})} + \normsq{\nabla f(\theta^{t-1})} \right)}\\
    \overset{\text{assumption \ref{assum:bounded_gd}}}&{\leq} 2\expect{ \left(\frac{1}{|\S{t}_{\tau}|}  - \frac{1}{|\S{}|} \right)^2 |\S{}_i|  \left( |\S{}_i|G^2 + \sum_{i \in \S{}_i} \normsq{\nabla f(\theta^{t-1})} \right)}
\end{align}
Since the term $\nabla f(\theta^{t-1})$ does not depend on the index $i$, we get
\begin{align}
    & 2\expect{ \left(\frac{1}{|\S{t}_{\tau}|}  - \frac{1}{|\S{}|} \right)^2 |\S{}_i|  \left( |\S{}_i|G^2 + \sum_{i \in \S{}_i} \normsq{\nabla f(\theta^{t-1})} \right)} \\
    &= 2\expect{ \left(\frac{1}{|\S{t}_{\tau}|}  - \frac{1}{|\S{}|} \right)^2 |\S{}_i|  \left( |\S{}_i|G^2 + |\S{}_i|\normsq{\nabla f(\theta^{t-1})} \right)} \\
    &= 2\expect{ \left(\frac{1}{|\S{t}_{\tau}|}  - \frac{1}{|\S{}|} \right)^2 |\S{}_i|^2}  \left( G^2 + \normsq{\nabla f(\theta^{t-1})} \right)
 \end{align}
 Now, note that $\S{t}_\tau \subseteq \S{} \implies |\S{}_i| = |\S{t}_\tau| $. Therefore, 
 \begin{align}
    \label{eq:proof_T3}
     \mathcal{T}_3 &\leq 2\expect{ \left(\frac{1}{|\S{t}_{\tau}|}  - \frac{1}{|\S{}|} \right)^2 |\S{}_i|^2}  \left( G^2 + \normsq{\nabla f(\theta^{t-1})} \right) \\
     & = 2 \expect{ \left(\frac{|\S{}| - |\S{t}_{\tau}|}{|\S{}|}\right)^2}  \left( G^2 + \normsq{\nabla f(\theta^{t-1})} \right)
 \end{align}

 Moving now to $\mathcal{T}_4$, we have:
 \begin{align}
    \mathcal{T}_4 &= \expect{\normsq{\sum_{k \in \S{}_d}\frac{1}{|\S{}|} g_k^t }} 
    \leq \expect{ \left( \frac{1}{|\S{}|} \right)^2 \normsq{\sum_{k \in \S{}_d} g_k^t }}\\
    \overset{\text{lemma \ref{lemma:relaxed_triangle}}}&{\leq}  \expect{ \left( \frac{1}{|\S{}|} \right)^2 |\S{}_d| \sum_{k \in \S{}_d} \normsq{g_k^t }}\\
    &= \expect{ \left( \frac{1}{|\S{}|} \right)^2 |\S{}_d| \sum_{k \in \S{}_d} \normsq{g_k^t - \nabla f(\theta^{t-1}) + \nabla f(\theta^{t-1})}} \\
    \overset{\text{lemma \ref{lemma:relaxed_triangle}}}&{\leq} 2\expect{ \left( \frac{1}{|\S{}|} \right)^2 |\S{}_d| \sum_{k \in \S{}_d}  \left( \normsq{g_k^t - \nabla f(\theta^{t-1})} + \normsq{\nabla f(\theta^{t-1})} \right)}\\
     \overset{\text{assumption \ref{assum:bounded_gd}}}&{\leq} 2\expect{ \left( \frac{1}{|\S{}|} \right)^2 |\S{}_d| \left( |\S{}_d| G^2 + \sum_{k \in \S{}_d}  \normsq{\nabla f(\theta^{t-1})} \right)} \\
     &{=} 2\expect{\left(\frac{1}{|\S{}|} \right)^2  |\S{}_d|\left( |\S{}_d| G^2 + |\S{}_d|\normsq{\nabla f(\theta^{t-1})} \right)} \\
     &{=}  2\expect{\left(\frac{|\S{}_d|}{|\S{}|} \right)^2}  \left(G^2 + \normsq{\nabla f(\theta^{t-1})} \right)\\
 \end{align}
Observing that $|\S{}_d| = |\S{}| - |\S{t}_\tau|$ we obtain:
\begin{equation}
    \label{eq:proof_T4}
    \mathcal{T}_4 \leq 2\expect{\left(\frac{|\S{}_d|}{|\S{}|} \right)^2}  \left(G^2 + \normsq{\nabla f(\theta^{t-1})} \right) = \expect{ \left(\frac{|\S{}| - |\S{t}_{\tau}|}{|\S{}|}\right)^2}  \left( G^2 + \normsq{\nabla f(\theta^{t-1})} \right)
\end{equation}

Finally, by plugging (\ref{eq:proof_T3}) and (\ref{eq:proof_T4}) in (\ref{eq:proof_T3T4}) we obtain
\begin{equation*}
    \expectcs{\normsq{g^{(t)_{\tau}}(\theta) - \nabla f(\theta)}}{t} \leq 8 \expectcs{ \left(\frac{|\S{}| - |\S{t}_{\tau}|}{|\S{}|}\right)^2}{t}  \left( G^2 + \normsq{\nabla f(\theta)} \right)
\end{equation*}
which concludes the proof.
\begin{flushright}\qedsymbol\end{flushright}

\paragraph{Proof of Corollary~\ref{corollary:tau_ghb_cyclic}}
This corollary follows from Lemma~\ref{lemma:tau_ghb}, which states that 
\begin{equation*}
    \expectcs{\normsq{g^{(t)_{\tau}}(\theta) - \nabla f(\theta)}}{t} \leq 8 \expectcs{ \left(\frac{|\S{}| - |\S{t}_{\tau}|}{|\S{}|}\right)^2}{t}  \left( G^2 + \normsq{\nabla f(\theta)} \right)
\end{equation*}

To prove the results, we use (i) \cref{assum:cyclic_part}, (ii) the fact that  $|\S{t}| = |\S{}| C\; \forall t$ and (iii) $\S{t}_\tau$ is union of $\tau$ disjoint $\mathcal{S}^t$ sets. Using points (i)-(iii), and assuming $\tau \in [0, \frac{1}{C}]$, it follows that:
\begin{equation*}
      \normsq{g^{(t)_{\tau}}(\theta) - \nabla f(\theta)} \leq 8  \left(1-\tau C\right)^2  \left( G^2 + \normsq{\nabla f(\theta)} \right)
\end{equation*}
\begin{flushright}\qedsymbol\end{flushright}

\newpage
\paragraph{Proof of Lemma~\ref{lemma:delayed_mgrad}}(Bounded error of delayed gradients)\\

    Note that, by \cref{assum:cyclic_part}, $|\S{t}|=|\S{}|C\, \forall t$, and that $|\S{}|C\tau = |\S{t}_\tau| $:
    \begin{align}
        \Lambda_t &= \expect{\normsq{\frac{1}{\tau} \sum_{k=t-\tau+1}^{t} \frac{1}{|\S{k}|J} \sum_{i=1}^{|\S{k}|} \sum_{j=1}^{J} \tilde{g}_i^{k,j}({\theta_i^{k,j-1}})  - g^{t_\tau} }} \\
        &= \expect{\normsq{\frac{1}{\tau} \sum_{k=t-\tau+1}^{t} \frac{1}{|\S{k}|J} \sum_{i=1}^{|\S{k}|} \sum_{j=1}^{J} \left(\tilde{g}_i^{k,j}({\theta_i^{k,j-1}})  - g_i({\theta^{t-1}}) \right)}} \\
        &= \expect{\normsq{\frac{1}{\tau} \sum_{k=t-\tau+1}^{t} \frac{1}{|\S{k}|J} \sum_{i=1}^{|\S{k}|} \sum_{j=1}^{J} \left(
        \tilde{g}_i^{k,j}({\theta_i^{k,j-1}})  - g_i({\theta_i^{k,j-1}}) 
        +g_i({\theta_i^{k,j-1}}) - g_i({\theta^{k-1}}) 
        +g_i({\theta^{k-1}}) - g_i({\theta^{t-1}})
        \right)}}\\
        &\leq 3 \left(\mathcal{T}_1 + \mathcal{T}_2 + \mathcal{T}_3 \right)
    \end{align}
    \begin{align}
        \mathcal{T}_1 &= \expect{\normsq{\frac{1}{\tau} \sum_{k=t-\tau+1}^{t} \frac{1}{|\S{k}|J} \sum_{i=1}^{|\S{k}|} \sum_{j=1}^{J} \left(
        \tilde{g}_i^{k,j}({\theta_i^{k,j-1}})  - g_i({\theta_i^{k,j-1}})\right)}}\\
        &\leq \frac{1}{\tau}\frac{\sigma^2}{|\S{t}|J} = \frac{\sigma^2}{|\S{t}_\tau|J} \\
        \mathcal{T}_2 &= \expect{\normsq{\frac{1}{\tau} \sum_{k=t-\tau+1}^{t} \frac{1}{|\S{k}|J} \sum_{i=1}^{|\S{k}|} \sum_{j=1}^{J} \left(
        g_i({\theta_i^{k,j-1}})  - g_i({\theta^{k-1}})\right)}}\\
        &\leq \frac{L^2}{|\S{}|J\tau}\sum_{k=t-\tau+1}^{t}\sum_{i=1}^{|\S{}|} \sum_{j=1}^{J}\expect{\normsq{\theta^{k,j-1} - \theta^{k-1}}} \\
        &= \frac{L^2}{\tau} \sum_{k=t-\tau+1}^{t}\drift{k} \\
        \mathcal{T}_3 &= \expect{\normsq{\frac{1}{\tau} \sum_{k=t-\tau+1}^{t} \frac{1}{|\S{k}|J} \sum_{i=1}^{|\S{k}|} \sum_{j=1}^{J} \left(
        g_i({\theta^{k-1}})  - g_i({\theta^{t-1}})\right)}}\\
        &\leq \frac{L^2}{|\S{}|\tau}\sum_{k=t-\tau+1}^{t}\sum_{i=1}^{|\S{}|}\expect{\normsq{\theta^{k-1} - \theta^{t-1}}} \\
        &\leq \frac{L^2}{\tau}\sum_{k=t-\tau+1}^{t}\expect{\normsq{\theta^{k-1} - \theta^{t-1}}} \\
        &= \frac{L^2}{\tau} \sum_{k=t-\tau+1}^{t}\left(t-k\right)\expect{\normsq{\theta^{k} - \theta^{k-1}}} \\
        &\leq 2L^2\eta^2 \sum_{k=t-\tau+1}^{t-1}\left(\expect{\normsq{\nabla f(\theta^{k-1}}} + \svrvar{k} \right)
    \end{align}
    So, combining with lemma \cref{lemma:client_drift,lemma:fs_drift} we have:
\begin{align}
    \sum_{t=1}^{T} \Lambda_t &\leq 3\left(\frac{T\sigma^2}{|\S{t}_\tau|J} + L^2 \sum^{T}_{t=1}\drift{t} + 2L^2\eta^2(\tau-1) \sum_{t=1}^{T-1}\left(\expect{\normsq{\nabla f(\theta^{t-1})}} + \svrvar{t} \right) \right) \\
    \overset{\text{lemma \ref{lemma:client_drift}}}&{=}
    3\bigg(\frac{T\sigma^2}{|\S{t}_\tau|J} + 2L^2\eta^2(\tau-1) \sum_{t=1}^{T-1}\left(\expect{\normsq{\nabla f(\theta^{t-1})}} + \svrvar{t} \right) \\
    &+ \underbrace{L^2TJ\eta_l^2\beta^2\sigma^2\left( 1+ 2J^3\eta_l^2\beta^2L^2\right)}_{\mathcal{T}_4} + 2J^2L^2e^2\sum^{T}_{t=1} \Xi_t)
    \bigg) \nonumber \\
    \overset{\text{lemma \ref{lemma:fs_drift}}}&{=}
    3\bigg(\frac{T\sigma^2}{|\S{t}_\tau|J} + 2L^2\eta^2(\tau-1) \sum_{t=1}^{T-1}\left(\expect{\normsq{\nabla f(\theta^{t-1})}} + \svrvar{t} \right) \\
    &+ \mathcal{T}_4 + 
    \underbrace{2J^2L^2e^2\left( 4\eta_l^2\left( (1-\beta)^2 +e(\beta\eta LT)^2\right)\right)}_{\alpha_1} \sum_{t=0}^{T-1}\left( \svrvar{t} + \expect{\normsq{\nabla f(\theta^{t-1})}}\right)
     \nonumber \\
    &+ \underbrace{2e^2J^2L^2(2e\eta_l^2\beta\tau T G_{\tau})}_{\mathcal{T}_5} \bigg)\nonumber \\
    &= 3\bigg(\frac{T\sigma^2}{|\S{t}_\tau|J} + \mathcal{T}_4 +\underbrace{\left( \alpha_1 +2L^2\eta_l^2 (\tau-1) \right)}_{\alpha_2}\sum_{t=1}^{T-1}\left(\expect{\normsq{\nabla f(\theta^{t-1})}} + \svrvar{t} \right) + \mathcal{T}_5\bigg)
\end{align}
\begin{flushright}\qedsymbol\end{flushright}

\subsection{Convergence Proof}

\begin{lemma}[Bounded variance of server updates]
Under \cref{assum:unbias_bvar,assum:smoothness}, it holds that:
\label{lemma:svrvar}
\begin{align}
    \sum_{t=1}^{T}\svrvar{t} &\leq \frac{8}{5\beta} \svrvar{0} + \frac{3}{5} \sum_{t=0}^{T-1} \expect{\normsq{\nabla f(\theta^{t-1})}} + 21\beta \frac{\sigma^2}{|\S{t}_\tau|J}T + \\
    &+\frac{448}{5}(\eta_lJL)^2(e^3\tau T) G_{\tau} + 6\beta \sum_{t=1}^{T}\gamma_t\nonumber
\end{align}
\end{lemma}

\begin{proof}
    \begin{align}
        \svrvar{t} &:= \expect{\normsq{\nabla f(\theta^{t-1}) - \tilde{m}^{t+1}_\tau}} \\
        &= \expect{\normsq{(1-\beta)(\nabla f(\theta^{t-1}) - \tilde{m}^{t}_\tau) + \beta(\nabla f(\theta^{t-1}) - \tilde{g}^{t_\tau})}} \\
        &= \expect{\normsq{(1-\beta)(\nabla f(\theta^{t-1}) - \tilde{m}^{t}_\tau)}} + \beta^2 \expect{\normsq{(\nabla f(\theta^{t-1}) - \tilde{g}^{t_\tau})}} \\
        &+ 2\beta \expect{\left\langle (1-\beta)(\nabla f(\theta^{t-1}) - \tilde{m}^{t}_\tau),  \nabla f(\theta^{t-1}) - \frac{1}{\tau} \sum_{k=t-\tau+1}^{t} \frac{1}{|\S{k}|J} \sum_{i=1}^{|\S{k}|} \sum_{j=1}^{J} {g}_i({\theta_i^{k,j-1}}) \right\rangle} \nonumber\\
    \end{align}
    Using the AM-GM inequality and \cref{lemma:relaxed_triangle}: 
    \begin{align}
        &\leq
        \left( 1+ \frac{\beta}{2} \right) \expect{\normsq{(1-\beta)(\nabla f(\theta^{t-1}) - \tilde{m}^{t}_\tau)}} +2\beta^2 \left( \gamma_t + \Lambda_t\right)
        + \nonumber \\
        &+ 4 \beta \gamma_t + 8\beta
        \left(\frac{L^2}{\tau}\sum_{k=t-\tau+1}^{t} \drift{k} + 2L^2\eta^2 \sum^{t-1}_{k=t-\tau+1}\left( \expect{\normsq{\nabla f(\theta^{k-1})}} + \svrvar{k}\right)\right)
         \\
        \overset{\text{lemma \ref{lemma:delayed_mgrad}}}&{\leq}
        \left( 1+ \frac{\beta}{2} \right) \expect{\normsq{(1-\beta)(\nabla f(\theta^{t-1}) - \tilde{m}^{t}_\tau)}} 
        + \left(2 \beta^2 + 4 \beta\right)\gamma_t + 6\beta^2\frac{\sigma^2}{|\S{t}_\tau|J} + \\
        &+\left(6\beta^2 + 8 \beta \right) 
        \underbrace{\left(
        \frac{L^2}{\tau}\sum_{k=t-\tau+1}^{t} \drift{k} + 2L^2\eta^2 \sum^{t-1}_{k=t-\tau+1}\left( \expect{\normsq{\nabla f(\theta^{k-1})}} + \svrvar{k}\right) \right)}_{\mathcal{T}_1}
        \nonumber \\
        &\leq (1-\beta)^2\left( 1+ \frac{\beta}{2} \right) \expect{\normsq{\nabla f(\theta^{t-2}) - \tilde{m}^{t}_\tau + \nabla f(\theta^{t-1}) - \nabla f(\theta^{t-2})}} + \\
        &+ 6 \beta^2\frac{\sigma^2}{|\S{t}_\tau|J} + 6\beta \gamma_t + 14\beta\mathcal{T}_1 \nonumber
    \end{align}
    Applying the AM-GM inequality again:
    \begin{align}
        &\leq (1-\beta)^2\left(1+\frac{\beta}{2} \right) \biggl[ \left( 1+ \frac{\beta}{4}\right)\expect{\normsq{\nabla f(\theta^{t-2}) -\tilde{m}^t_{\tau}}} + \\ 
        &+ \left( 1+ \frac{1}{\beta}\right)\expect{\normsq{\nabla f(\theta^{t-1}) - \nabla f(\theta^{t-2})}}  \bigg] + 6 \beta^2\frac{\sigma^2}{|\S{t}_\tau|J} + 6\beta \gamma_t + 14\beta\mathcal{T}_1 \nonumber \\
        \overset{\text{assumption \ref{assum:smoothness}}}&{\leq}
        (1-\beta)^2\left(1+\frac{\beta}{2} \right) \biggl[ \left( 1+ \frac{\beta}{4}\right)\svrvar{t-1} + \\ 
        &+ \left( 1+ \frac{1}{\beta}\right)L^2\expect{\normsq{\theta^{t-1} - \theta^{t-2}}}  \bigg] + 6 \beta^2\frac{\sigma^2}{|\S{t}_\tau|J} + 6\beta \gamma_t + 14\beta\mathcal{T}_1 \nonumber \\
        &\leq (1-\beta)^2\left(1+\frac{\beta}{2} \right) \biggl[ \left( 1+ \frac{\beta}{4}\right)\svrvar{t-1} + \\ 
        &+ 2\left( 1+ \frac{1}{\beta}\right)L^2\eta^2\left(\expect{\normsq{\nabla f(\theta^{t-2})}} + \svrvar{t-1} \right) \bigg] + 6 \beta^2\frac{\sigma^2}{|\S{t}_\tau|J} + 6\beta \gamma_t + 14\beta\mathcal{T}_1 \nonumber
    \end{align}
    Where in the last inequality we used the fact that: 
    $$\normsq{\theta^{t-1}-\theta^{t-2}} \leq 2 \eta^2\left(\normsq{\nabla f(\theta^{t-2})} + \normsq{\nabla f(\theta^{t-2}) - \tilde{m}^t_\tau} \right).$$
    
    Now notice that $(1-\beta)^2 \left(1+\frac{\beta}{2}\right)\left( 1 + \frac{\beta}{4}\right)\leq (1-\beta)$ and that $2(1-\beta)^2\left( 1+\frac{\beta}{2}\right)\left(1+\frac{1}{\beta}\right)\leq \frac{2}{\beta}$:
    \begin{align}
        \svrvar{t} &\leq (1-\beta)\svrvar{t-1} + \frac{2}{\beta}L^2\eta^2\left( \expect{\normsq{\nabla f(\theta^{t-2})}} + \svrvar{t-1}\right) + 6 \beta^2\frac{\sigma^2}{|\S{t}_\tau|J} + 6\beta \gamma_t + 14\beta\mathcal{T}_1 \\
        &= \left(1-\beta + \frac{2}{\beta}L^2\eta^2\right)\svrvar{t-1} + \frac{2}{\beta}L^2\eta^2\expect{\normsq{\nabla f(\theta^{t-2})}} + 6 \beta^2\frac{\sigma^2}{|\S{t}_\tau|J} + 6\beta \gamma_t + 14\beta\mathcal{T}_1
    \end{align}
    Define:
    \begin{itemize}
        \item $\mathcal{T}_2 := L^2TJ\eta_l^2\beta^2\sigma^2\left( 1+ 2J^3\eta_l^2\beta^2L^2\right)$
        \item $\mathcal{T}_3 := 2e^2J^2L^2(2e\eta_l^2\beta\tau T G_{\tau})$
        \item $\alpha_1:= 2J^2L^2e^2\left( 4\eta_l^2\left( (1-\beta)^2 +e(\beta\eta LT)^2\right)\right) +2L^2\eta_l^2 (\tau-1)$
    \end{itemize}
    Summing up over $T$ and substituting into $\mathcal{T}_1$ the expression for $\drift{t}$:
    \begin{align}
        \sum_{t=1}^{T}\svrvar{t} &\leq \underbrace{\left( 1-\beta + \frac{2}{\beta}L^2\eta^2 + 14\beta\alpha_1 \right)}_{\alpha_2} \sum_{t=0}^{T-1}\svrvar{t} + \\
        &+\underbrace{\left( \frac{2}{\beta} L^2\eta^2 + 14\beta\alpha_1 \right)}_{\alpha_3} \sum_{t=0}^{T-1} \expect{\normsq{\nabla f(\theta^{t-1})}} + \nonumber \\
        &+14\beta \left(\mathcal{T}_2 + \mathcal{T}_3 \right)T + 6\beta^2\frac{\sigma^2}{|\S{t}_\tau|J}T + 6\beta \sum_{t=1}^{T}\gamma_t \nonumber
    \end{align}
    We now have that:
    \begin{align}
        \alpha_2 &:= \left(1-\beta + \frac{2}{\beta} L^2\eta^2 + 14\beta\left[ 2J^2L^2e^2\left( 4\eta_l^2\left( (1-\beta)^2 +e(\beta\eta LT)^2\right)\right) +2L^2\eta_l^2 (\tau-1) \right] \right) \\
        &= \left(1-\beta + \frac{2}{\beta} L^2\eta^2 + 14\beta\left[ 8J^2L^2e^2\eta_l^2\left( (1-\beta)^2 +e(\beta\eta LT)^2\right) +2L^2\eta_l^2 (\tau-1) \right] \right) \\
        &\leq \left(1-\beta + \frac{2}{\beta} L^2\eta^2 + 112\beta e^2(\eta_lJL)^2\left[ (1-\beta)^2 +(\beta\eta LT)^2 + (\tau-1) \right] \right) \\
    \end{align}
    Now impose $(\eta_l J L) \leq \left(37\sqrt{\tau} \beta \eta LTe\right)^{-1}$ and $\eta \leq \frac{\beta}{\sqrt{8}L}$. We have that:
    \begin{align}
        \alpha_2 &\leq \left(1-\beta + \frac{2\beta}{8} + \frac{\beta}{8}\right) = \left(1-\frac{5\beta}{8}\right) \\
        \alpha_3 &\leq \frac{3\beta}{8} \\
        14\beta \mathcal{T}_2 &= 14\beta L^2TJ\eta_l^2\beta^2\sigma^2\left( 1+ 2J^3\eta_l^2\beta^2L^2\right) \\
        &= 14\beta^3(\eta_l J L)^2\left( \frac{1}{J} + 2(\eta_l J L \beta)^2\right)\sigma^2T \\
        &\leq 7\beta^2 \frac{\sigma^2}{|\S{t}_\tau|J}T
    \end{align}
    Where in the last inequality we apply:
    \[ 
    2\beta(\eta_l J L)^2\left( \frac{1}{J} + 2(\eta_l J L \beta)^2\right) \leq \frac{1}{|\S{t}_\tau|J}
    \]
    Plugging all the terms together we have:
    \begin{align}
        \sum_{t=1}^{T}\svrvar{t} &\leq \left( 1- \frac{5}{8\beta}\right) \sum_{t=0}^{T-1}\svrvar{t} + \frac{3\beta}{8} \sum_{t=0}^{T-1} \expect{\normsq{\nabla f(\theta^{t-1})}} + 13\beta^2 \frac{\sigma^2}{|\S{t}_\tau|J}T + \\
        &+56\beta(\eta_lJL)^2(e^3\tau T) G_{\tau} + 6\beta \sum_{t=1}^{T}\gamma_t\nonumber
    \end{align}
    Rearranging the terms completes the proof.
\end{proof}

\begin{lemma}
\label{lemma:client_drift}
Under \cref{assum:unbias_bvar,assum:smoothness}, for \cref{def:cdrift} it holds that:
\begin{align}
    \drift{t} &\leq 2J^2e^2 \fsdrift{t} + J\eta_l^2 \beta^2 \sigma^2 (1+2J^3 \eta_l^2 L^2 \beta^2) \\
    \sum_{t=1}^{T} \drift{t} &\leq TJ\eta_l^2 \beta^2 \sigma^2 (1+2J^3 \eta_l^2\beta^2 L^2) + 2J^2 e^2 \sum_{t=1}^{T} \fsdrift{t}
\end{align}
\end{lemma}

\begin{proof}
    \begin{align}
        \expect{\normsq{\theta_{i}^{t, j} - \theta^{t-1}}} &\leq
        2 \expect{\normsq{\sum_{k=0}^{j-1}\zeta_{i}^{t,k} }} + 2j\eta_l^2 \beta^2 \sigma^2 \\
        \overset{\text{lemma \ref{lemma:relaxed_triangle}}}&{\leq}
        2j \sum_{k=0}^{j-1} \expect{\normsq{\zeta_{i}^{t,k}}} + 2j\eta_l^2\beta^2\sigma^2
    \end{align}
    For any $ 1 \leq k \leq j-1 \leq J-2$, using $\eta L \leq \frac{1}{\beta J} \leq \frac{1}{\beta(j+1)}$, we have:
    \begin{align}
        \expect{\normsq{\zeta_{i}^{t,k}}} &\leq
        \left(1 + \frac{1}{j}\right)\expect{\normsq{\zeta_{i}^{t,k-1}}} + (1+j)\expect{\normsq{\zeta_{i}^{t,k} - \zeta_{i}^{t, k-1}}} \\
        &\leq \left(1 + \frac{1}{j}\right)\expect{\normsq{\zeta_{i}^{t,k-1}}} + (1+j)\eta_l^2 \beta^2 L^2 \left(\eta_l^2 \beta^2 \sigma^2 + \expect{\normsq{\zeta_{i}^{t, k-1}}} \right) \\
        &\leq
        \left(1 + \frac{1}{j}\right)\expect{\normsq{\zeta_{i}^{t,k-1}}} + (1+j)\eta_l^4 \beta^4 L^2 \sigma^2 + \frac{1}{1+j}\expect{\normsq{\zeta_{i}^{t,k} - \zeta_{i}^{t, k-1}}} \\
        &\leq \left(1 + \frac{2}{j}\right)\expect{\normsq{\zeta_{i}^{t,k-1}}} + (1+j)\eta_l^4 \beta^4 L^2 \sigma^2 \\
        \overset{\left(1+\frac{2}{j}\right)^j\leq e^2}&{\leq}
        e^2 \expect{\normsq{\zeta_{i}^{t,0}}} + 4j^2 \eta_l^4 \beta^4 L^2 \sigma^2
    \end{align}
    So it holds that:
    \begin{align}
        \expect{\normsq{\theta_{i}^{t, j} - \theta^{t-1}}} &\leq 2j^2 \left(e^2 \expect{\normsq{\zeta_i^{t,0}}} + 4j^2\eta_l^4 L^2 \sigma^2 \right) + 2j\eta_l^2\sigma^2 \\
        &= 2e^2j^2\expect{\normsq{\zeta_i^{t,0}}} + 2j\eta_l^2\sigma^2 \beta^2 (1+4j^3\eta_l^2 L^2 \beta^2)
    \end{align}
    So, summing up over $i$ and $j$:
    \begin{align}
        \drift{t} &\leq \frac{1}{|\S{}|J}\sum_{i=1}^{|\S{}|}\sum_{j=1}^{J} 2e^2j^2\expect{\normsq{\zeta_i^{t,0}}} + 
        2j\eta_l^2\sigma^2 \beta^2 (1+4j^3\eta_l^2 L^2 \beta^2) \\
        &\leq 2J^2e^2 \fsdrift{t} + J\eta_l^2 \beta^2 \sigma^2 (1+2J^3 \eta_l^2 L^2 \beta^2)
    \end{align}
    Finally, summing up over $T$:
    \begin{align}
        \sum_{t=1}^{T} \drift{t} &\leq 
        \underbrace{TJ\eta_l^2 \beta^2 \sigma^2 (1+2J^3 \eta_l^2\beta^2 L^2)}_{\mathcal{T}_1} + 2J^2 e^2 \sum_{t=1}^{T} \fsdrift{t} \\
        &\leq \mathcal{T}_1 + 2J^2 e^2 \left( 4\eta^2 \left((1-\beta)^2 + e(\beta\eta LT)^2 \right) \sum_{t=1}^{T-1} \left( \svrvar{t} + \expect{\normsq{\nabla f(\theta^{t-1})}} \right) + \underbrace{2e\eta^2\beta^2\tau T G_\tau}_{\mathcal{T}_2} \right) \\
        &\leq \mathcal{T}_1 + \alpha_1 \sum_{t=1}^{T-1}\left( \svrvar{t} + \expect{\normsq{\nabla f(\theta^{t-1})}}\right) + \alpha_2 \mathcal{T}_2
    \end{align}
\end{proof}

\begin{lemma}
\label{lemma:fs_drift}
Under \cref{assum:unbias_bvar,assum:smoothness,assum:cyclic_part}, 
if $224e(\eta_l JL)^2 \left((1-\beta)^2 + e(\beta\eta L T)^2 \right) \leq 1$,
for \cref{def:fs_drift} it holds for $t \ge 0$ that:
\begin{equation}
    \fsdrift{t} \leq \frac{1}{56eJ^2L^2}\sum_{t=0}^{T-1}\left( \svrvar{t} + \expect{\normsq{\nabla f(\theta^{t-1})}}\right) + 2e\eta_l^2\beta^2 \tau TG_\tau
\end{equation}
\end{lemma}
\begin{proof}
    Note that $\zeta_i^{t,0} = -\eta_l\left( (1-\beta) \tilde{m}^{t}_\tau + \beta g_i(\theta^{t-1})\right)$,
    \begin{equation}
        \frac{1}{|\S{}|}\sum_{i=1}^{|\S{}|}\normsq{\zeta_i^{t,0}} \leq 2\eta_l^2 \left( (1-\beta)^2 \normsq{\tilde{m}^t_\tau} + \frac{\beta^2}{|\S{}|}\sum_{i=1}^{|\S{}|}\normsq{g_i(\theta^{t-1})}\right)
    \end{equation}
    For any $a>0$, considering each client participates to the train every $\tau=\frac{1}{C}$ rounds:
    \begin{align}
        \expect{\normsq{g_i(\theta^{t-1})}} &= \expect{\normsq{g_i(\theta^{t-1}) - g_i(\theta^{t-\tau-1}) + g_i(\theta^{t-\tau-1})}} \\
        \overset{\text{lemma \ref{lemma:relaxed_triangle}}}&{\leq}
        (1+a)\expect{\normsq{g_i(\theta^{t-\tau-1})}} + \\ &+\left( 1+ \frac{1}{a} \right) \expect{\normsq{g_i(\theta^{t-1}) - g_i(\theta^{t-\tau-1})}} \nonumber \\
        &\leq (1+a)\expect{\normsq{g_i(\theta^{t-\tau-1})}} + \\ &+\left( 1+ \frac{1}{a} \right) L^2\expect{\normsq{\theta^{t-1} - \theta^{t-\tau-1}}} \\
        &\leq (1+a)\expect{\normsq{g_i(\theta^{t-\tau-1})}} + \\ &+2\left( 1+ \frac{1}{a} \right) L^2\eta^2\tau \sum_{k=1}^{\tau}\left(\svrvar{t-k}+\expect{\normsq{\nabla f(\theta^{t-k-1})}}\right) \\
        &\leq (1+a)^{\frac{t}{\tau}} \expect{\normsq{g_i(\theta^{t_i-1})}} +\\ &+ 2\left( 1+ \frac{1}{a} \right)L^2\eta^2\tau\sum_{s=1}^{\frac{t}{\tau}}\sum_{k=1}^{\tau}\left( \svrvar{s\tau-k} + \expect{\normsq{\nabla f(\theta^{s\tau-k})}}\right)(1+a)^{\frac{t}{\tau}-s} \nonumber \\
        &\leq
        (1+a)^{\frac{t}{\tau}} \expect{\normsq{g_i(\theta^{t_i-1})}} +\\
        &+ 2\left( 1+ \frac{1}{a} \right)L^2\eta^2\tau \sum_{k=1}^{t-1} \left( \svrvar{k} + \expect{\normsq{\nabla f(\theta^{k-1})}}\right)(1+a)^{\frac{t}{\tau}} \nonumber
    \end{align}
    Where $t_i := \min_{t \in [T]}( t \,s.t. \,\,i \in \S{t}$). Now take $a=\frac{\tau}{t}$:
    \begin{align}
        \expect{\normsq{g_i(\theta^{t-1})}} &\leq e\expect{\normsq{g_i(\theta^{t_i-1})}} + \\ &+ 2e\eta^2L^2\tau\left( \frac{t}{\tau} + 1\right)\sum_{k=1}^{t-1} \left( \svrvar{k} + \expect{\normsq{\nabla f(\theta^{k-1})}}\right) \nonumber
    \end{align}
    So:
    \begin{align}
        \sum_{t=1}^{T} \fsdrift{t} &\leq
        \sum_{t=1}^{T} 2\eta_l^2 \left( 2(1-\beta)^2\left(\svrvar{t-1} + \expect{\normsq{\nabla f(\theta^{t-2}}} \right) + \frac{\beta^2}{|\S{}|} \sum_{i=1}^{|\S{}|}\expect{\normsq{g_i(\theta^{t-1})}}\right) \\
        &\leq \sum_{t=1}^{T} 4 \eta_l^2 (1-\beta)^2\left( \svrvar{t-1} + \expect{\normsq{\nabla f(\theta^{t-2})}}\right) + \\
        &+2\eta_l^2 \beta^2 \sum_{t=1}^{T}\left(\frac{e}{|\S{}|}\sum_{i=1}^{|\S{}|}\expect{\normsq{g_i(\theta^{t_i-1})}} + 2e\eta_l^2 L^2\tau\left( \frac{t}{\tau}+1\right) \sum_{k=1}^{t-1}\left(\svrvar{k} + \expect{\normsq{\nabla f(\theta^{t-1}}} \right)\right) \nonumber \\
        &\leq  4 \eta_l^2 (1-\beta)^2 \sum_{t=1}^{T}\left( \svrvar{t-1} + \expect{\normsq{\nabla f(\theta^{t-2})}}\right) + \\
        &+2\eta_l^2 \beta^2\left(eT\sum_{t=1}^{\tau}G_t + 2e(\eta LT)^2\sum_{t=1}^{T-1}\left( \svrvar{t} + \expect{\normsq{\nabla f(\theta^{t-1})}}\right) \right) \nonumber
    \end{align}
    Let us define $G_\tau := \max_{t \in [1, \tau]}G_t$, with $G_t:=\frac{1}{|\S{t}|}\sum_{i=1}^{|\S{t}|}\expect{\normsq{g_i(\theta^{t-1})}}$. We have that:
    \begin{align}
        \sum_{t=1}^{T} \fsdrift{t} &\leq
        4 \eta_l^2 \left((1-\beta)^2 + e(\beta\eta L T)^2 \right)\sum_{t=0}^{T-1}\left( \svrvar{t} + \expect{\normsq{\nabla f(\theta^{t-1})}}\right) + 2e\eta_l^2\beta^2 \tau TG_\tau
    \end{align}
    Applying the upper bound of $\eta_l$ completes the proof.
\end{proof}

\begin{lemma}[\cite{cheng2024momentum}]
\label{lemma:conv_lemma}
Under \cref{assum:smoothness}, if $\eta L \leq \frac{1}{24}$, the following holds for all $t \ge 0$:
\begin{equation}
    \expect{f(\theta^{t})} \leq \expect{f(\theta^{t-1})} - \frac{11\eta}{24} \expect{\normsq{\nabla f(\theta^{t-1})}} + \frac{13\eta}{24} \svrvar{t}
\end{equation}
\end{lemma}

\begin{proof}
    Since f is $L$-smooth, we have:
    \begin{align}
        f(\theta^{t}) &\leq f(\theta^{t-1}) + \left\langle \nabla f(\theta^{t-1}), \theta^t - \theta^{t-1} \right\rangle + \frac{L}{2} \normsq{\theta^t - \theta^{t-1}} \\
        &= f(\theta^{t-1}) - \eta \normsq{\nabla f(\theta^{t-1}} + \eta \left\langle \nabla f(\theta^{t-1}), \nabla f(\theta^{t-1}) - \tilde{m}^{t+1}_\tau \right\rangle + \frac{L\eta^2}{2} \normsq{\tilde{m}^{t+1}_\tau}
    \end{align}
    Since $\theta^t = \theta^{t-1} - \eta \tilde{m}^{t+1}_\tau$, using Young's inequality and imposing $\eta L \leq \frac{1}{24}$, we further have:
    \begin{align}
        f(\theta^t) &\leq f(\theta^{t-1}) - \frac{\eta}{2}\normsq{\nabla f(\theta^{t-1})} + \frac{\eta}{2}\normsq{\nabla f(\theta^{t-1}) - \tilde{m}^{t+1}_\tau} +\\
        &+L\eta^2\left( \normsq{\nabla f(\theta^{t-1})} + \normsq{\nabla f(\theta^{t-1}) - \tilde{m}^{t+1}_\tau}\right) \nonumber \\
        &\leq f(\theta^{t-1}) - \frac{11\eta}{24} \normsq{\nabla f(\theta^{t-1})} + \frac{13\eta}{24} \normsq{\nabla f(\theta^{t-1}) - \tilde{m}^{t+1}_\tau}
    \end{align}
\end{proof}

\paragraph{Proof of Theorem~\ref{thm:GHBM}}(Convergence rate of GHBM for non-convex functions)\\

\textit{Under \cref{assum:unbias_bvar,assum:smoothness,assum:cyclic_part}, if we take:}
\begin{align}
\label{eq:conv_thm_constraints}
\tilde{m}^{0}_\tau&=0, \qquad
\beta =\min\left\{1, \sqrt{\frac{|\S{}|JL\Delta}{\sigma^2 T}} \right\}, \qquad
\eta =\min\left\{\frac{1}{24L}, \frac{\beta}{\sqrt{8}L} \right\}\, \\
\eta_lJL  &\lesssim \min\left\{1,  \frac{1}{\beta \eta L\sqrt{\tau}T}, \sqrt{\frac{L\Delta}{\beta^3 \tau G_\tau T}}, \frac{1}{\sqrt{\beta |\S{}|}}, \left(\frac{1}{\beta^3|\S{}|J}\right)^{\frac{1}{4}} \right\}\, \nonumber
\end{align}
\textit{then GHBM with optimal $\tau=\frac{1}{C}$ converges as:}
   \begin{equation}
       \frac{1}{T} \sum_{t=1}^{T}\expect{\normsq{\nabla f(\theta^{t-1})}} \lesssim \frac{L\Delta}{T} + \sqrt{\frac{L\Delta \sigma^2}{|\S{}|JT}}
   \end{equation} 

\begin{proof}
    Combining the results of \cref{lemma:svrvar,lemma:conv_lemma}, we have that:
    \begin{align}
        \sum_{t=1}^{T}\left(\expect{f(\theta^t} - \expect{f(\theta^{t-1}} \right) &\leq  - \frac{11\eta}{24} \sum_{t=1}^{T}\expect{\normsq{\nabla f(\theta^{t-1}}} + \frac{13\eta}{24} \sum_{t=1}^{T}\svrvar{t} \\
        \frac{1}{\eta}\expect{f(\theta^{t-1} - f(\theta^0)} &\leq \frac{26}{30\beta}\svrvar{0} - \frac{1}{15} \sum_{t=1}^{T}\expect{\normsq{\nabla f(\theta^{t-1}}} + 32\beta \frac{\sigma^2}{|\S{t}_\tau|J}T + \\
    &+\frac{448}{5}(\eta_lJL)^2(e^3\tau T) G_{\tau} + 6\beta \sum_{t=1}^{T}\gamma_t
    \end{align}
    Imposing $\tau=\frac{1}{C}$, by \cref{corollary:tau_ghb_cyclic} we have that $\gamma_t=0\,\,$ and $\S{t}_\tau=\S{}\,\,\forall t$. Also, noticing that $\tilde{m}^0_\tau=0$ implies $\svrvar{0} \leq 2L\left(f(\theta^0) - f^{*} \right) = 2L\Delta$, we have that:
    \begin{align}
        \frac{1}{T} \sum_{t=1}^{T}\expect{\normsq{\nabla f(\theta^{t-1})}} &\lesssim \frac{L\Delta}{\eta L T} + \frac{\svrvar{0}}{\beta T} + (\eta_l JL\beta)^2\tau G_\tau + \beta \frac{\sigma^2}{|\S{}|J} \\
        &\lesssim \frac{L\Delta}{T} + \frac{2L\Delta}{\beta T} + (\eta_l JL\beta)^2\tau G_\tau + \beta \frac{\sigma^2}{|\S{}|J} \label{eq:thm:final_sec_step}\\
        &\edit{\lesssim \frac{L\Delta}{T} + \frac{2L\Delta}{\beta T} + \beta^2 \left(\frac{L \Delta}{\beta^3 \tau G_\tau T}\right) \tau G_\tau + \beta \frac{\sigma^2}{|\S{}|J}}\\
        &\lesssim \frac{L\Delta}{T} + \frac{L\Delta}{\beta T} + \beta \frac{\sigma^2}{|\S{}|J}\\
        &\lesssim \frac{L\Delta}{T} + \sqrt{\frac{L\Delta \sigma^2}{|\S{}|JT}}
    \end{align}
    \edit{where the fourth inequality follows from applying the upper bound $ \eta_l J L \leq \sqrt{\frac{L \Delta}{\beta^3 \tau G_\tau T}}$ on the third term of \cref{eq:thm:final_sec_step}.}
\end{proof}

\section{Experimental Setting}
\label{app:exp}
\subsection{Datasets and Models}
\label{app:exp:datasets}
\paragraph{\cifar{10/100}.}
We consider \cifar{10} and \cifar{100} to experiment with image classification tasks, each one respectively having 10 and 100 classes. For all methods, training images are preprocessed by applying random crops, followed by random horizontal flips. Both training and test images are finally normalized according to their mean and standard deviation.
As the main model for experimentation, we used a model similar to \textsc{LeNet-5} as proposed in \citep{hsu2020FederatedVisual}. To further validate our findings, we also employed a {\resnet} as described in \citep{he2015deep}, following the implementation provided in \citep{resnetCifarGithub}. Since batch normalization~\cite{ioffe2015batch} layers have been shown to hamper performance in learning from decentralized data with skewed label distribution \citep{pmlr-v119-hsieh20a}, we replaced them with group normalization \citep{Wu2018GroupNorm}, using two groups in each layer. For a fair comparison, we used the same modified network also in centralized training. We report the result of centralized training for reference in Table~\ref{tab:acc_centralized}: as per the hyperparameters, we use $64$ for the batch size, $0.01$ and $0.1$ for the learning rate respectively for the \textsc{LeNet} and the {\resnet} and $0.9$ for momentum. We trained both models on both datasets for $150$ epochs using a cosine annealing learning rate scheduler.

\begin{wraptable}{O}{0.5\textwidth}
\vspace{-4mm}

 \caption{\textbf{Test accuracy (\%) of centralized training over datasets and models used.} Results are reported in term of mean top-1 accuracy over the last 10 epochs, averaged over 5 independent runs.}
 \label{tab:acc_centralized}
 \resizebox{0.5\textwidth}{!}{
  \begin{tabular}{lr}
  \toprule
  \textsc{Dataset} & \textsc{Acc. Centralized ($\%$)}\\
  \midrule
  \cifar{10 w/ LeNet} & $86.48 {\,\scriptstyle\pm 0.22}$\\
  \cifar{10 w/ ResNet-20} & $89.05 {\,\scriptstyle\pm 0.44}$\\
  \cifar{100 w/ LeNet} & $57.00 {\,\scriptstyle\pm 0.09}$\\
  \cifar{100 w/ ResNet-20} & $62.21 {\,\scriptstyle\pm 0.85}$\\
  \shakespeare & $52.00 {\,\scriptstyle\pm 0.16}$\\
  \stackoverflow & $28.50 {\,\scriptstyle \pm 0.25}$\\
  \landmarks & $74.03 {\,\scriptstyle \pm 0.15}$ \\
  \bottomrule
  \end{tabular}
  }

\vspace{-5mm}
\end{wraptable}

\paragraph{Shakespeare.}
The Shakespeare language modeling dataset is created by collating the collective works of William Shakespeare and originally comprises $715$ clients, with each client denoting a speaking role. However, for this study, a different approach was used, adopting the LEAF \citep{caldas2019leaf} framework to split the dataset among $100$ devices and restrict the number of data points per device to $2000$. The non-IID dataset is formed by assigning each device to a specific role, and the local dataset for each device contains the sentences from that role. Conversely, the IID dataset is created by randomly distributing sentences from all roles across the devices.

For this task, we have employed a two-layer Long Short-Term Memory (LSTM) classifier, consisting of 100 hidden units and an 8-dimensional embedding layer. Our objective is to predict the next character in a sequence, where there are a total of 80 possible character classes. The model takes in a sequence of 80 characters as input, and for each character, it learns an 8-dimensional representation. The final output of the model is a single character prediction for each training example, achieved through the use of 2 LSTM layers and a densely-connected layer followed by a softmax. This model architecture is the same used by \citep{li2020FedProx, acar2021FedDyn}.

We report the result of centralized training for reference in Table~\ref{tab:acc_centralized}: we train for $75$ epochs with constant learning rate, using as hyperparameters $100$ for the batch size, $1$ for the learning rate, $0.0001$ for the weight decay and no momentum.

\paragraph{StackOverflow.}

The Stack Overflow dataset is a language modeling corpus that comprises questions and answers from the popular Q\&A website, StackOverflow. 
Initially, the dataset consists of $342477$ unique users but for, practical reasons, we limit our analysis to a subset of $40k$ users. Our goal is to perform the next-word prediction on these text sequences. To achieve this, we utilize a Recurrent Neural Network (RNN) that first learns a $96$-dimensional representation for each word in a sentence and then processes them through a single LSTM layer with a hidden dimension of $670$. Finally, the model generates predictions using a densely connected softmax output layer. The model and the preprocessing steps are the same as in \citep{reddi2020FedOpt}.
We report the result of centralized training for reference in Table~\ref{tab:acc_centralized}: as per the hyperparameters, we use $16$ for the batch size, $10^{-\nicefrac{1}{2}}$ for the learning rate and no momentum or weight decay. We train for $50$ epochs with a constant learning rate.
Given the size of the test dataset, testing is conducted on a subset of them made by $10000$ randomly chosen test examples, selected at the beginning of training. 

\paragraph{Large-scale Real-world Datasets.}
As large-scale real-world datasets for our experimentation, we follow \cite{hsu2020FederatedVisual}. {\landmarks} is composed of $\approx164k$ images belonging to $\approx2000$ classes, realistically split among $1262$ clients. {\inaturalist} is composed of $\approx120k$ images belonging to $\approx1200$ classes, split among $9275$ clients.
These datasets are challenging to train not only because of their inherent complexity (size of images, number of classes) but also because usually at each round a very small portion of clients is selected. In particular, for {\landmarks} we sample $10$ clients per round, while for {\inaturalist} we experiment with different participation rates, sampling $10$, $50$, or $100$ clients per round. In the main paper, we choose to report the participation rate instead of the number of sampled clients to better highlight that the tested scenarios are closer to a cross-device setting, which is the most challenging for algorithms based on client participation, like {\scaffold} and ours.
As per the model, for both datasets, we use a MobileNetV2 pretrained on ImageNet.

\paragraph{Details on the Experiment in \cref{fig:compare_ghbm_theory}.}
\edit{
In the main text (see \cref{sec:theory:conv_rate}) we provide an experiment to illustrate the convergence rate of {\ghb} (see \cref{fig:compare_ghbm_theory}). The learning problem consists in a linear regression of the coefficients $(a,b,c) \in \mathbb{R}$ of a quadratic function $f(x)=ax^2+bx+c$. The synthetic dataset is made of $6400$ observations of the above function (with $a=10, b=5, c=-1$) in the range $x \in [-10, 10]$. The dataset is split among $K=50$ clients each one having $128$ samples, and non-iidness is simulated by splitting the domain into equally big disjoint subsets, and having each client the observation of that domain.
}

\begin{table}[h]
\centering
    \caption{Details about datasets' split used for our experiments}
    \label{tab:datasets}
    \resizebox{\linewidth}{!}{
    \begin{tabular}{lcccccc}
      \toprule 
        &
      \multicolumn{1}{c}{\cifar{10}} &
      \multicolumn{1}{c}{\cifar{100}} &
      \multicolumn{1}{c}{\shakespeare} &
      \multicolumn{1}{c}{\stackoverflow} &
      \multicolumn{1}{c}{\landmarks} &
      \multicolumn{1}{c}{\inaturalist}  \\
      \midrule
      Clients & 100 & 100 & 100 & 40.000 & 1262 & 9275 \\
      Number of clients per round & 10 & 10 & 10 & 50 & 10 & $\{10,50,100\}$\\
      Number of classes & 10 & 100 & 80 & 10004 & 2028 & 1203 \\
      Avg. examples per client & 500 & 500 & 2000 & 428 & 130 & 13 \\
      Number of local steps & 8 & 8 & 20 & 27 & 13 & 2 \\
      Average participation (round no.)  & 1k & 1k & 25 & 1.5 & 40 & $\{5,27,54\}$\\
      \bottomrule
    \end{tabular}
    }
\end{table}

\subsection{Simulating Heterogeneity}
\label{app:exp:non_iid}
For \cifar{10/100} we simulate arbitrary heterogeneity by splitting the total datasets according to a Dirichlet distribution with concentration parameter $\alpha$, following \cite{hsu2020FederatedVisual}. 
In practice, we draw a multinomial $q_i \sim \mathbf{Dir}(\alpha p)$ from a Dirichlet distribution, where $p$ describes a prior class distribution over $N$ classes, and $\alpha$ controls the heterogeneity among all clients: the greater $\alpha$ the more homogeneous the clients' data distributions will be. After drawing the class distributions $q_i$, for every client $i$, we sample training examples for each class according to $q_i$ without replacement.
\begin{table*}[t]
    \begin{center}
    \caption{Hyper-parameter search grid for each combination of method and dataset (for $\alpha=0$). The best values are indicated in \textbf{bold}.
    }
    \label{tab:best_hyper}
    \resizebox{0.97\linewidth}{!}{
    \begin{tabular}{lccccc}
      \toprule 
       \textsc{Method} &
       \textsc{HParam} &
      \multicolumn{2}{c}{\cifar{10/100}} &
      \multicolumn{1}{c}{\textsc{Shakespeare}} &
      \multicolumn{1}{c}{\textsc{StackOverflow}} \\
      \cmidrule{3-4}
      & & \multicolumn{1}{c}{\small \textsc{LeNet}}
        & \multicolumn{1}{c}{\small \textsc{ResNet-20}}\\
      \midrule
      \multirow{2}{*}{\textsc{All FL}} 
      & wd & [\textbf{0.001}, 0.0008, 0.0004] & [0.0001, \textbf{0.00001}] & [0, \textbf{0.0001}, 0.00001] & [\textbf{0}, 0.0001, 0.00001]\\
      & $B$ & 64 & 64 & 100 & 16\\
      \midrule
      \multirow{2}{*}{\fedavg} 
      & $\eta$ & [\edit{2}, \textbf{\edit{1.5}}, 1, 0.5, 0.1] & [\edit{1.5}, \textbf{1}, 0.1] & [\edit{1.5}, \textbf{1}, 0.5, 0.1] & [\edit{1.5}, \textbf{1}, 0.5, 0.1] \\
      & $\eta_l$ & [0.1, 0.05, \textbf{0.01}, \edit{0.005}] & [\edit{1}, \textbf{\edit{0.5}}, 0.1, 0.01] & [\edit{1.5}, \textbf{1}, 0.5, 0.1] & [1, 0.5, \textbf{0.3}, 0.1]\\
      \midrule
      \multirow{3}{*}{\fedprox} 
      & $\eta$ & [\edit{2}, \textbf{\edit{1.5}}, 1, 0.5, 0.1] & [\edit{1.5}, \textbf{1}, 0.1] & [\edit{1.5}, \textbf{1}, 0.5, 0.1] & [\edit{1.5}, \textbf{1}, 0.5, 0.1]\\
      & $\eta_l$ & [0.1, 0.05, \textbf{0.01}, \edit{0.005}] & [\edit{1}, \textbf{\edit{0.5}}, 0.1, 0.01] & [\edit{1.5}, \textbf{1}, 0.5, 0.1] & [1, 0.5, \textbf{0.3}, 0.1]\\
      & $\mu$ & [\edit{1}, 0.1, \textbf{0.01}, 0.001] & [\edit{1}, \textbf{0.1}, 0.01, 0.001] & [0.1, 0.01, 0.001, \textbf{0.0001}, \edit{0.00001}] & [0.1, \textbf{0.01}, 0.001, 0.0001]\\
      \midrule
      \multirow{2}{*}{\scaffold} 
      & $\eta$ & [\edit{1.5}, \textbf{1}, 0.5, 0.1] & [\edit{1.5}, \textbf{1}, 0.1] & [\edit{1.5}, \textbf{1}, 0.5, 0.1] & [\edit{1.5}, \textbf{1}, 0.5, 0.1]\\
      & $\eta_l$ & [0.1, 0.05, \textbf{0.01}, \edit{0.005}] & [\edit{0.5}, \textbf{0.1}, 0.01] & [\edit{1.5}, \textbf{1}, 0.5, 0.1] & [1, 0.5, \textbf{0.3}, 0.1]\\
      \midrule
      \multirow{3}{*}{\feddyn} 
      & $\eta$ & [\edit{1.5}, \textbf{1}, 0.5, 0.1] & [\edit{1.5}, \textbf{1}, 0.1] & [\edit{1.5}, \textbf{1}, 0.5, 0.1] & [\edit{1.5}, \textbf{1}, 0.5, 0.1]\\
      & $\eta_l$ & [0.1, 0.05, \textbf{0.01}, \edit{0.005}] & [0.1, \textbf{0.01}, \edit{0.005}] & [\edit{1.5}, \textbf{1}, 0.5, 0.1] & [1, 0.5, \textbf{0.3}, 0.1]\\
      & $\alpha$ & [0.1, 0.01, \textbf{0.001}, \edit{0.0001}] & [0.1, 0.01, \textbf{0.001}, \edit{0.0001}] & [0.1, \textbf{0.009}, 0.001] & [\textbf{0.1}, 0.009, 0.001]\\
      \midrule
      \multirow{3}{*}{\adabest} 
      & $\eta$ & [\edit{1.5}, \textbf{1}, 0.5, 0.1] & [\edit{1.5}, \textbf{1}, 0.5, 0.1] & [\edit{1.5}, \textbf{1}, 0.5, 0.1] & [\edit{1.5}, \textbf{1}, 0.5, 0.1]\\
      & $\eta_l$ & [0.1, 0.05, \textbf{0.01}, \edit{0.005}] & [0.1, 0.05, \textbf{0.01}, \edit{0.005}] & [\edit{1.5}, \textbf{1}, 0.5, 0.1] & [1, 0.5, \textbf{0.3}, 0.1]\\
      & $\alpha$ & [0.1, 0.01, \textbf{0.001}, \edit{0.0001}] & [0.1, 0.01, \textbf{0.001}, \edit{0.0001}] & [0.1, \textbf{0.009}, 0.001] & [\textbf{0.1}, 0.009, 0.001]\\
      \midrule
      \multirow{2}{*}{\mime} 
      & $\eta$ & [\edit{2}, \textbf{\edit{1.5}}, 1, 0.5, 0.1] & [\edit{2}, \textbf{\edit{1.5}}, 1, 0.1] & [\edit{1.5}, \textbf{1}, 0.5, 0.1] & [\edit{1.5}, \textbf{1}, 0.5, 0.1]\\
      & $\eta_l$ & [0.1, 0.05, \textbf{0.01}, \edit{0.005}] & [\edit{0.5}, \textbf{0.1}, 0.01] & [\edit{1.5}, \textbf{1}, 0.5, 0.1] & [1, 0.5, \textbf{0.3}, 0.1]\\
      \midrule
      \multirow{3}{*}{\fedavgm} 
      & $\eta$ & [1, 0.5, 0.1, \textbf{\edit{0.05}}, \edit{0.01}] & [1, \textbf{0.1}, \edit{0.05}] & [1, \textbf{0.5}, 0.1] & [\edit{1.5}, \textbf{1}, 0.5, 0.1]\\
      & $\eta_l$ & [\edit{0.5}, \textbf{0.1}, 0.05, 0.01, \edit{0.005}] & [\edit{1}, \textbf{\edit{0.5}}, 0.1, 0.01] & [\edit{1.5}, \textbf{1}, 0.5, 0.1] & [1, 0.5, \textbf{0.3}, 0.1]\\
      & $\beta$ & [0.99, 0.9, \textbf{\edit{0.85}}, \edit{0.8}] & [0.99, 0.9, \textbf{\edit{0.85}}, \edit{0.8}] & [0.99, \textbf{0.9}, \edit{0.85}] & [0.99, \textbf{0.9}, \edit{0.85}] \\
      \midrule
      \multirow{4}{*}{\edit{\textsc{FedACG}}}
      & $\eta$ & [1, 0.5, 0.1, \textbf{0.05}, 0.01] & [1, \textbf{0.1}, 0.05] & [0.5, \textbf{0.1}, 0.05] & [\edit{1.5}, \textbf{1}, 0.5, 0.1]\\
      & $\eta_l$ & [0.5, \textbf{0.1}, 0.05, 0.01, 0.005] & [0.5, \textbf{0.1}, 0.01] & [\edit{1.5}, \textbf{1}, 0.5, 0.1] & [1, 0.5, \textbf{0.3}, 0.1]\\
      & $\lambda$ & [0.99, \textbf{0.9}, 0.85] & [0.99, \textbf{0.9}, 0.85] & [0.99, \textbf{0.9}, 0.85] & [0.99, \textbf{0.9}, 0.85] \\
      & $\beta$ & [0.1, \textbf{0.01}, 0.001] & [0.1, \textbf{0.01}, 0.001] & [0.1, 0.01, 0.001, \textbf{0.0001}, \edit{0.00001}] & [0.1, \textbf{0.01}, 0.001, 0.0001]\\
      \midrule
      \multirow{3}{*}{\mimemom} 
      & $\eta$ & [1, 0.5, \textbf{0.1}, \edit{0.05}] & [\edit{1.5}, \textbf{1}, 0.5, 0.3, 0.1, \edit{0.05}] & [1, 0.5, \textbf{0.1}, \edit{0.05}] & [\edit{1.5}, \textbf{1}, 0.5, 0.1]\\
      & $\eta_l$ & [0.1, 0.05, \textbf{0.01}, \edit{0.005}] & [\edit{0.5}, 0.1, 0.05, 0.03, \textbf{0.01}, \edit{0.005}] & [\edit{1.5}, \textbf{1}, 0.5, 0.1] & [1, 0.5, 0.3, \textbf{0.1}, \edit{0.05}]\\
      & $\beta$ & [0.99, 0.95, \textbf{0.9}, \edit{0.85}, \edit{0.8}] & [0.99, 0.95, 0.9, \textbf{\edit{0.85}}, \edit{0.8}] & [0.99, \textbf{0.9}, \edit{0.85}] & [0.99, \textbf{0.9}, \edit{0.85}] \\
      \midrule
      \multirow{3}{*}{\mimelite} 
      & $\eta$ & [1, 0.5, \textbf{0.1}, \edit{0.05}] & [\edit{1.5}, \textbf{1}, 0.5, 0.3, 0.1] & [1, 0.5, \textbf{0.1}, \edit{0.05}] & [\edit{1.5}, \textbf{1}, 0.5, 0.1]\\
      & $\eta_l$ & [0.1, 0.05, \textbf{0.01}, \edit{0.005}] &  [0.1, 0.05, 0.03, \textbf{0.01}, \edit{0.005}] & [\edit{1.5}, \textbf{1}, 0.5, 0.1] & [1, 0.5, 0.3, \textbf{0.1}, \edit{0.05}]\\
      & $\beta$ & [0.99, \textbf{0.9}, \edit{0.85}, \edit{0.8}] & [0.99, 0.95, 0.9, \textbf{\edit{0.85}}, \edit{0.8}] & [0.99, \textbf{0.9}, \edit{0.85}] & [0.99, \textbf{0.9}, \edit{0.85}] \\
      \midrule
      \multirow{3}{*}{\textsc{FedCM}} 
      & $\eta$ & [1, 0.5, \textbf{0.1}, \edit{0.05}] & [\edit{1.5}, \textbf{1}, 0.5, 0.1] & [1, 0.5, \textbf{0.1}, \edit{0.05}] & -\\
      & $\eta_l$ & [1, 0.5, \textbf{0.1}, \edit{0.05}] & [1, 0.5, \textbf{0.1}, \edit{0.5}] & [\edit{1.5}, \textbf{1}, 0.5, 0.1] & -\\
      & $\alpha$ & [0.05, \textbf{0.1}, 0.5] & [0.05, \textbf{0.1}, 0.5] & [0.05, \textbf{0.1}, 0.5] & - \\
      \midrule
      \multirow{3}{*}{\textbf{{\ghb} (ours)}} 
      & $\eta$ & [\textbf{1}, 0.5, 0.1] & [\textbf{1}, 0.1] & [\textbf{1}, 0.5, 0.1] & [\textbf{1}, 0.5, 0.1] \\
      & $\eta_l$ & [0.1, 0.05, \textbf{0.01}] & [0.1, \textbf{0.01}] & [\textbf{1}, 0.5, 0.1] & [1, 0.5, \textbf{0.3}, 0.1] \\
      & $\beta$ & [\textbf{0.9}] & [\textbf{0.9}] & [\textbf{0.9}] & [\textbf{0.9}] \\
      & $\tau$ & [5, \textbf{10}, 20, 40] & [5, \textbf{10}, 20, 40] & [5, \textbf{10}, 20, 40] & [5, 10, \textbf{20}, 40] \\
      \midrule
      \multirow{3}{*}{\textbf{\fedhbm (ours)}} 
      & $\eta$ & [\textbf{1}, 0.5, 0.1] & [\textbf{1}, 0.1] & [\textbf{1}, 0.5, 0.1] & [\textbf{1}, 0.5, 0.1] \\
      & $\eta_l$ & [0.1, 0.05, \textbf{0.01}] & [0.1, \textbf{0.01}] & [\textbf{1}, 0.5, 0.1] & [1, 0.5, \textbf{0.3}, 0.1]\\
      & $\beta$ & [\textbf{1}, 0.99, 0.9] & [\textbf{1}, 0.99, 0.9] & [\textbf{1}, 0.99, 0.9] & [\textbf{1}, 0.99, 0.9] \\
      \bottomrule
    \end{tabular}
    }
    \end{center}
\end{table*} %

\subsection{Evaluating Communication and Computational Cost}
\label{app:exp:cost}
In the main paper we showed a comparison in communication and computational cost of state-of-art FL algorithms compared to our solutions {\ghb} and {\fedhbm}: in this section we detail how those results in table \cref{tab:comm_comp_cost} have been obtained. We follow a three-step procedure:
\begin{enumerate}
    \item For each algorithm $a$, we calculate the minimum number of rounds $r_a$ to reach the performance of {\fedavg}, the total amount of bytes exchanged $b_a$ in the whole training budget (number of rounds, as described in \cref{app:exp:impl_practicality}) and the measure the corresponding total training time $t_a$. In this way, the different requirements in communication and computation of each algorithm are taken into account for the next steps.
    \item We calculate the actual communication and computational requirements as $(tb_a=b_a\cdot s_a, tt_a = t_a \cdot s_a)$, where $s_a=\frac{r_a}{T}$ is the speedup of the algorithm {\wrt} {\fedavg}. For those competitor algorithms that did not reach the target performance ({\eg} {\mimemom}) in the training budget $T$, we conservatively consider $r_a=T$.
    In this way, the convergence speed of each algorithm is taken into account for determining the actual amount of computation needed.
    \item We complement the above information with with a reduction/increase factor {\wrt} {\fedavg}, calculated as 
    $rtb_a = \left(1-\frac{tb_a}{tb_{\fedavg}}\right)$ and $rtt_a = \left(1-\frac{tt_a}{tt_{\fedavg}}\right)$ and expressed as a percentage. A cost reduction (\ie $rtb_a>0$ or $rtt_a>0$) is indicated with \textcolor{MaterialGreen800}{$\boldsymbol{\downarrow}$}, while a cost increase (\ie $rtb_a<0$ or $rtt_a<0$) is indicated with \textcolor{MaterialRed800}{$\boldsymbol{\uparrow}$}.
    This gives a practical indication of how much communication/computation have been saved in choosing the algorithm at hand as an alternative for {\fedavg}.
\end{enumerate}

\subsection{Hyperparameters}

\label{app:exp:hyperparameters}
For ease of consultation, we report the hyper-parameters grids as well as the chosen values in Table~\ref{tab:best_hyper}. For {\landmarks} and {\inaturalist} we only test the best SOTA algorithms: {\fedavg} and {\fedavgm} as baselines, {\scaffold} and {\mimemom}.
\paragraph{\mobilenet.} For all algorithms we perform $E=5$ local epochs, and searched $\eta \in \{0.1, 1\}$ and $\eta_l \in \{0.01, 0.1\}$, and found $\eta=0.1, \eta_l=0.1$ works best for {\fedavgm}, while $\eta=1, \eta_l=0.1$ works best for the others. For {\inaturalist}, we had to enlarge the grid for {\scaffold} and {\mimemom}: for both we searched $\eta \in \{10^{-3/2}, 10^{-1}, 10^{-1/2}, 1\}$ and $\eta_l \in \{10^{-2}, 10^{-3/2},  10^{-1}, 10^{-1/2}\}$.
\paragraph{\vit.} For all algorithms we perform $E=5$ local epochs, and searched $\eta \in \{0.1, 1\}$ and $\eta_l \in \{0.03, 0.01\}$ following \citep{steiner2022how}, and found $\eta=0.1, \eta_l=0.03$ works best for {\fedavgm}, while $\eta=1, \eta_l=0.03$ works best for the others.

\subsection{Implementation Details}
We implemented all the tested algorithms and training procedures in a single codebase, using \textsc{PyTorch 1.10} framework, compiled with \textsc{cuda 10.2}. 
The federated learning setup is simulated by using a single node equipped with $11$ Intel(R) Core(TM) i7-6850K \textsc{CPUs} and $4$ NVIDIA GeForce GTX 1070 \textsc{GPUs}. For the large-scale experiments we used the computing capabilities offered by \textsc{LEONARDO} cluster of \textsc{CINECA-HPC}, employing nodes equipped with $1$ CPU Intel(R) Xeon 8358 32 core, 2,6 GHz \textsc{CPUs} and $4$ NVIDIA A100 SXM6 64GB (VRAM) \textsc{GPUs}.
The simulation always runs in a sequential manner (on a single \textsc{GPU}) the parallel client training and the following aggregation by the central server.

\paragraph{Practicality of Experiments.}
\label{app:exp:impl_practicality}
Under the above conditions, a single {\fedavg} experiment on \cifar{100} takes $\approx$ 02:05 hours ({\lenet}, with $T=20.000$) and $\approx$ 03:36 hours ({\resnet}, with $T=10.000$). For {\scaffold} we always use the \texttt{"option II"} of their algorithm \citep{karimireddy2020scaffold} to calculate the client controls, incurring almost no overhead in our simulations. We found that using \texttt{"option I"} usually degrades both final model quality and requires almost double the training time, due to the additional forward+backward passes. Conversely, all \textsc{Mime}'s methods incur a significant overhead due to the additional round needed to calculate the full-batch gradients, taking $\approx$ 10:40 hours for \cifar{100} with {\resnet}. 
On {\shakespeare} and {\stackoverflow}, {\fedavg} takes $\approx22$ minutes and $\approx 3.5$ hours to run respectively $T=250$ and $T=1500$ rounds.

\subsection{Additional Experiments}

\begin{wraptable}{OH}{0.55\textwidth}
\vspace{-5mm}

    \caption{\textbf{Test accuracy (\%) comparison of SOTA FL algorithms in a controlled setting.} Best result is in \textbf{bold}, second best is \underline{underlined}.}
    \label{tab:main_controlled_cifar10}
    \centering
\resizebox{\linewidth}{!}{
    \begin{tabular}{lcccccc}
      \toprule 
      \multirow{2}{*}{\textsc{Method}} &
      \multicolumn{2}{c}{\cifar{10} {\scriptsize(\resnet)}} &
      \multicolumn{2}{c}{\cifar{10} {\scriptsize(\lenet)}} \\
    \cmidrule(lr){2-3} \cmidrule(lr){4-5} 
       &
      \multicolumn{1}{c}{\small \textsc{NON-IID}} &
      \multicolumn{1}{c}{\small \textsc{IID}} &
      \multicolumn{1}{c}{\small \textsc{NON-IID}} &
      \multicolumn{1}{c}{\small \textsc{IID}} \\
        \midrule
         \fedavg & $61.0 {\scriptstyle\,\pm 1.0}$ & $86.4 {\scriptstyle\,\pm 0.2}$ & $66.1 {\scriptstyle\,\pm 0.3}$ & $83.1 {\scriptstyle\,\pm 0.3}$ \\
         \fedprox & $61.0 {\scriptstyle\,\pm 1.8}$ & $86.7 {\scriptstyle\,\pm 0.2}$ & $66.1 {\scriptstyle\,\pm 0.3}$ & $83.1 {\scriptstyle\,\pm 0.3}$ \\
         \scaffold & $71.8 {\scriptstyle\,\pm 1.7}$ & $86.8 {\scriptstyle\,\pm 0.3}$ & $74.8 {\scriptstyle\,\pm 0.2}$ & $82.9 {\scriptstyle\,\pm 0.2}$ \\
         \feddyn &  $60.2 {\scriptstyle\,\pm 3.0}$ & $87.0 {\scriptstyle\,\pm 0.3}$ & $70.9 {\scriptstyle\,\pm 0.2}$ & ${83.5} {\scriptstyle\,\pm 0.1}$ \\
         \adabest & $73.6 {\scriptstyle\,\pm 3.0}$ & $86.7 {\scriptstyle\,\pm 0.5}$ & $66.1 {\scriptstyle\,\pm 0.3}$ & $83.1 {\scriptstyle\,\pm 0.4}$ \\
         \mime & $53.7 {\scriptstyle\,\pm 2.9}$ & $86.7 {\scriptstyle\,\pm 0.1}$ & $75.1 {\scriptstyle\,\pm 0.5}$ & $83.1 {\scriptstyle\,\pm 0.2}$ \\
        \cmidrule{1-1}
         \fedavgm & $66.0 {\scriptstyle\,\pm 2.2}$ & $87.7 {\scriptstyle\,\pm 0.3}$ & $67.6 {\scriptstyle\,\pm 0.3}$ & $83.6 {\scriptstyle\,\pm 0.3}$ \\
         \fedcm ${\scriptstyle (\ghb\,\tau=1)}$ & $65.2 {\scriptstyle\,\pm 3.2}$ & $87.1 {\scriptstyle\,\pm 0.3}$ & $69.0 {\scriptstyle\,\pm 0.3}$ & $83.4 {\scriptstyle\,\pm 0.3}$ \\
         \fedadc ${\scriptstyle (\ghb\,\tau=1)}$ & $65.7 {\scriptstyle\,\pm 3.0}$ & $87.1 {\scriptstyle\,\pm 0.2}$ & $66.1 {\scriptstyle\,\pm 0.3}$ & $83.4 {\scriptstyle\,\pm 0.3}$ \\
         \mimemom & $69.2 {\scriptstyle\,\pm 3.6}$ & $88.0 {\scriptstyle\,\pm 0.1}$ & ${80.9} {\scriptstyle\,\pm 0.4}$ & $83.1 {\scriptstyle\,\pm 0.2}$ \\
         \mimelite & $57.0 {\scriptstyle\,\pm 0.9}$ & $88.0 {\scriptstyle\,\pm 0.4}$ & $78.8 {\scriptstyle\,\pm 0.4}$ & $83.2 {\scriptstyle\,\pm 0.3}$ \\
         \cmidrule{1-1}
          \rowcolor{Cerulean!30!white}\textbf{{\localghb} (ours)} & $\underline{80.6} {\scriptstyle\,\pm 0.3}$ & $\underline{88.8} {\scriptstyle\,\pm 0.1}$ & $\underline{81.1} {\scriptstyle\,\pm 0.3}$ & $\underline{83.7} {\scriptstyle\,\pm 0.1}$ \\
         \rowcolor{Cerulean!30!white}\textbf{{\fedhbm} (ours)} & $\mathbf{83.4} {\scriptstyle\,\pm 0.3}$ & $\mathbf{89.2} {\scriptstyle\,\pm 0.1}$ & $\mathbf{81.7} {\scriptstyle\,\pm 0.1}$ & $\mathbf{83.8} {\scriptstyle\,\pm 0.1}$ \\
        \bottomrule
    \end{tabular}
}
    \setlength{\tabcolsep}{1.4pt}

\vspace{-10mm}
\end{wraptable}

\paragraph{Experiments on \cifar{10}}
Table \ref{tab:main_controlled_cifar10} reports the results of experiments analogous to the ones presented in Tab. \ref{tab:main_controlled}. For the main paper, we report experiments on \cifar{100}, as it is a more complex dataset and often a more reliable testing ground for FL algorithms. Indeed, sometimes algorithms perform well on \cifar{10} but worse on \cifar{100} (as for the already discussed case of {\feddyn}). Results in Tab. \ref{tab:main_controlled_cifar10} confirm the findings of the main paper: under extreme heterogeneity, some algorithms behave inconsistently across {\lenet} and {\resnet} (notice that {\feddyn} and {\mimelite} only with {\lenet} improve {\fedavg}. Conversely, {\localghb} and {\fedhbm} both consistently improve the state-of-art by a large margin.

\end{document}